\documentclass[11pt]{article}
\usepackage[english]{babel}
\usepackage[utf8x]{inputenc}
\usepackage{fullpage}
\usepackage{url}
\usepackage{smile}
\usepackage{pdfpages}
\usepackage{kpfonts}
\usepackage{amsmath}
\usepackage{multirow}

\usepackage{graphicx}

\usepackage{enumitem}

\usepackage[colorlinks=true,
 linkcolor=blue,
 urlcolor=blue,
 citecolor=blue,
 plainpages=false]{hyperref}

\makeatother
\usepackage{natbib}
\numberwithin{equation}{section}
\numberwithin{theorem}{section}

\makeatletter
\newcommand*{\rom}[1]{\expandafter\@slowromancap\romannumeral #1@}
\makeatother

\begin{document}
\title{\bf A Diffusion Approximation Theory of Momentum SGD in Nonconvex Optimization\footnote{Working in progress.}}
\author{Tianyi Liu,~~Zhehui Chen,~~Enlu Zhou,~~Tuo Zhao\thanks{T. Liu, Z. Chen, E. Zhou, and T. Zhao are affiliated with School of Industrial and Systems Engineering at Georgia Tech; Tuo Zhao is the corresponding author; Email: tourzhao@gatech.edu.}}
\date{}
\maketitle
\begin{abstract}
Momentum Stochastic Gradient Descent (MSGD) algorithm has been widely applied to many nonconvex optimization problems in machine learning, e.g., training deep neural networks, variational Bayesian inference, and etc. Despite its empirical success, there is still a lack of theoretical understanding of convergence properties of MSGD. To fill this gap, we propose to analyze the algorithmic behavior of MSGD by diffusion approximations for  nonconvex optimization problems with strict saddle points and isolated local optima.  Our study shows that the momentum \emph{helps escape from saddle points}, but \emph{hurts the convergence within the neighborhood of optima} (if without the step size annealing or momentum annealing). Our theoretical discovery partially corroborates the empirical success of MSGD in training deep neural networks. 
\end{abstract}

\section{Introduction} \label{Intro}
Nonconvex stochastic optimization naturally arises in many machine learning problems. Taking training deep neural networks as an example, given $n$ samples denoted by $\{(x_i, y_i)\}_{i=1}^{n}$, where $x_i$ is the $i$-th input feature and $y_i$ is the response, we solve the following optimization problem,
\begin{align}\label{obj-general}
\min_{\theta}\cF(\theta):=\frac{1}{n}\sum_{i=1}^n\ell(y_i,f(x_i,\theta)),
\end{align}
where $\ell$ is a loss function,  $f$ denotes the decision function based on the neural network, and $\theta$ denotes the  parameter associated with $f$.
Stochastic Gradient Descent (SGD), which has been known for a long time as stochastic approximation in the control and simulation literature \citep{robbins1951stochastic, borkar2000ode, kushner2003stochastic,borkar2009stochastic,fu2015handbook}, has been applied to solve machine learning problems such as \eqref{obj-general} \citep{newton2018recent}. Momentum Stochastic Gradient Descent (MSGD, \cite{polyak1964some}) is one of the most popular 
   variants of SGD.  
   {Specifically, at the $t$-th iteration, we uniformly sample $i$ from $(1,...,n)$. Then, we take
\begin{align}\label{SGD_momentum}
\theta^{(t+1)} =\theta^{(t)}-\eta\nabla \ell(y_i,f(x_i,\theta^{(t)}))+ \mu(\theta^{(t)}-\theta^{(t-1)}),
\end{align}
where  $\eta$ is the step size parameter and $\mu\in[0,1)$ is the parameter for controlling the momentum.  Note that when $\mu=0$, \eqref{SGD_momentum} is reduced to Vanilla Stochastic Gradient Descent (VSGD). }

Although SGD-type algorithms have demonstrated significant empirical success for training deep neural networks, their convergence properties for nonconvex optimization are still largely unknown. For VSGD, existing literature \citep{ghadimi2013stochastic} shows that it is guaranteed to converge to a first-order optimal solution (i.e., $\nabla \cF(\theta)=0$) under general smooth nonconvex optimization.

The theoretical investigation of MSGD is even more limited than that of VSGD. The momentum in \eqref{SGD_momentum}  has been observed to significantly accelerate computation in practice. To the best of our knowledge, we are only aware of \cite{ghadimi2016accelerated} in existing literature, which shows that MSGD is guaranteed to converge to a first-order optimal solution for smooth nonconvex problems. Their analysis, however, does not justify the advantage of the momentum in MSGD over VSGD.

To fill the gap between the significant empirical success and the lack of theoretical understanding of MSGD, we are interested in answering a  natural and fundamental question in this paper: 

\begin{center}\textbf{\emph{What is the role of the momentum in nonconvex stochastic optimization?}}\end{center}

The major technical bottleneck in analyzing MSGD and answering the above question comes from the nonconvex optimization landscape of these highly complicated problems, e.g., training large recommendation systems and deep neural networks. We propose to analyze MSGD for nonconvex optimization problems under the assumption of isolated local optima and strict saddle points. This allows us to make progress toward understanding MSGD and gaining new insights on more general problems. Specifically, we consider the following problem:
\begin{align*}
\min_{x\in\RR^d} \cF(x)=\EE[f(x,\xi)],	
\end{align*}
where $\xi$ is a random variable representing the noise, and $f(x,\xi)$ is nonconvex in $x$ given any realization of $\xi$. We assume that the nonconvex landscape has the following structures: (1) every local optimum has positive curvatures along all directions; (2) there always exist negative curvatures around saddle points (strict saddle property). 

The strict saddle property is shared by several popular nonconvex optimization problems arising in machine learning and signal processing, including streaming principle component analysis (PCA), matrix regression/completion/sensing, independent component analysis, partial least square multiview learning, and phase retrieval \citep{ge2016matrix,li2016symmetry,sun2016geometric}. Moreover, since there is a significant lack of understanding the optimization landscape of general nonconvex problems, many researchers suggest that analyzing strict saddle optimization problems should be considered as the first and important step towards understanding the algorithmic behaviors in general nonconvex optimization. We also want to remark that our analysis can be extended to connected local optima cases. Doing so requires more technical machinery instead of fundamental insights. Therefore, we present the analysis for the isolated optima case for readability and simplicity, and it has already conveyed our core ideas on the effect of momentum.

By making use of the diffusion approximation of stochastic optimization, we provide global and local analysis of MSGD. Specifically, to study the global dynamics, we transfer the discrete time trajectory to a continuous time one by interpolation and prove that asymptotically this continuous time solution trajectory of MSGD  converges weakly to the solution of an appropriately constructed ODE. This ODE approximation shows that {\it the momentum helps traverse among stationary points in the non-stationary region}, where the variance of the stochastic gradient can be neglected compared with the large magnitude of the gradient. Intuitively, with the help of the momentum, the algorithm makes more progress along the descent direction. Thus, the momentum can accelerate the algorithm in this region  by a factor of $\frac{1}{1-\mu}$.

ODE approximation, however, cannot justify how momentum works in the stationary area where the variance of the stochastic gradient dominates the update. To highlight the effect of the variance, we consider the asymptotic  behavior of the normalized estimation error obtained by MSGD around the stationary points.  We show that the continuous time interpolation of the normalized error sequence converges weakly to a solution of an approximately constructed SDE. By analyzing this SDE solution, we find that the momentum can play different but important roles around saddle points and local optima.

\begin{itemize}
	
	\item \emph{The momentum helps escape from the neighborhood of saddle points}:  In this region, since the gradient diminishes, the  variance of the stochastic gradient  dominates the algorithmic behavior. Our analysis indicates that the momentum greatly increases  the variance and perturbs the algorithm more aggressively. Thus, it becomes harder for the algorithm to stay around saddle points. In addition, the momentum also encourages more aggressive  exploitation, and in each iteration, the algorithm makes more progress along the descent direction by a factor of $\frac{1}{1-\mu}$, where $\mu$ is the momentum parameter, compared with the VSGD. 
	
	
	\item  \emph{The momentum hurts the convergence within the  neighborhood of  local optima}: Similar to the neighborhood of saddle points, the gradient dies out, and the  variance of the stochastic gradient dominates. Since the momentum increases the variance, it is harder for the algorithm to enter the small neighborhood. To this respect, the momentum hurts in this region. We suggest to apply a step size annealing scheme to neutralize the large variance introduced by momentum within the neighborhood of local optima. 

\end{itemize}

{

Our ODE/SDE approximation analysis justifies the role of momentum in both stationary and non-stationary areas. However,  given the complicated nonconvex landscape,  our diffusion approximation analysis cannot establish the second order convergence guarantee and the asymptotic convergence rate of MSGD. Therefore, we further provide a simple but highly non-trival example, streaming PCA, to illustrate our characterization of the effect of momentum and also establish the asymptotic convergence results. 

Streaming PCA is a nonconvex problem with only one global optimum and $d-1$ strict saddle points up to sign change, where $d$ is the dimension. Its optimization landscape contains the following three regions:
\begin{itemize}
\item {$\boldsymbol{\cR_1}$}: The region containing the neighborhood of strict saddle points with negative curvatures;

\item $\boldsymbol{\cR_2}$: The region including the set of points whose gradient has sufficiently large magnitude;

\item $\boldsymbol{\cR_3}$: The region containing the neighborhood of all local optima  with positive curvatures along all directions. 
\end{itemize}

 By studying the corresponding mean ODE and SDE, we show that {\it with arbitrary initialization, MSGD can converge to the global optimum.} We provide asymptotic convergence rates of MSGD which precisely quantify the acceleration by momentum in $\boldsymbol{\cR_1}$ and $\boldsymbol{\cR_2}$. Meanwhile, we also show that with proper step size annealing scheme implemented,  MSGD can achieve the same convergence rate as VSGD in $\boldsymbol{\cR_3}.$
 
{Our characterization helps explain some phenomena observed when training deep neural networks.} There have been some  empirical observations and theoretical results \citep{choromanska2015loss} showing that saddle points are the major computation bottleneck, and VSGD  usually spends most of the time traveling along saddle and non-stationary regions. Since the momentum helps in both regions, we can find in practice MSGD performs better than VSGD. In addition, from our analysis, the momentum hurts convergence within the neighborhood of the optima. However, we can address this problem by decreasing the step size or the momentum parameter.  
 }
 
We further verify our theoretical findings through numerical experiments on training a residual network~\citep{he2016deep}, using both CIFAR-10 and CIFAR-100 datasets. The experimental results show that the algorithmic behavior of MSGD is consistent with our analysis. Moreover, we observe that with a proper initial step size and a proper step size annealing process, MSGD eventually achieves better generalization accuracy than that of VSGD in training neural networks.

To the best of our knowledge, our proposed theory is the first attempt towards understanding the role of momentum in nonconvex stochastic optimization beyond the convergence to stationary solutions. Taking our results as an initial start, we expect more sophisticated and stronger follow-up work for analyzing momentum SGD, e.g., extending our asymptotic theory to its nonasymptotic counterpart. Please refer to Section \ref{discussion} for more detailed discussions.

The rest of the paper is organized as follows: Section  \ref{Algorithm} introduces our nonconvex optimization problem settings and MSGD for solving the problem. Sections \ref{Section_ODE} and \ref{Section_SDE} analyze the global and local dynamics of MSGD based on diffusion approximations, respectively; Section \ref{section_pca} studies Streaming PCA and provides an asymptotic convergence rate analysis; Section \ref{Numerical} presents the numerical experiments on both streaming PCA and training deep neural networks to demonstrate our theoretical results; Section \ref{discussion} makes some further discussions on the related literature, our theoretical and experimental results and future work; The Appendix includes all technical details. 

\section{Momentum SGD}\label{Algorithm}

Recall that we study MSGD for a general nonconvex optimization problem as follows,
\begin{align}\label{general_problem}
\min_{x\in\RR^d} \cF(x)=\EE[f(x,\xi)],	
\end{align}
where $\xi$ is a random variable representing the noise and $f(x,\xi)$ is nonconvex in $x$ given any realization of $\xi.$ We assume that there is a stochastic gradient oracle taking $x'\in \RR^d$ as input and outputting $\nabla f(x',\xi'),$ where $\xi'$ is a realization of the noise $\xi,$ such that
$$\EE[\nabla f(x',\xi')] = \nabla \cF(x'), ~~~ \Cov[\nabla f(x',\xi')] = \Sigma.$$  

Given a vector $ v = (v^{(1)}, \ldots, v^{(d)})^\top\in\RR^{d}$, we define the vector norm: $\norm{v}^2_2 = \sum_j(v^{(j)})^2$.  We impose the following standard assumptions on the objective $\cF(x)$ and $f(x,\xi).$

\begin{assumption}\label{assumption_general}
~
\begin{itemize}
	\item Uniform Boundedness: There exists a constant $C$ such that 
	$\norm{\nabla f(x,\xi)}_2\leq C,~~ \forall x, \xi. $
	\item Lipschitz Continuous: There exists a constant $L$ such that 
	   $$\norm{\nabla \cF(x_1)-\nabla \cF(x_2)}_2\leq L\norm{x_1-x_2}_2, \quad\forall x_1,x_2\in \RR^d .$$
\end{itemize}
\end{assumption}

In general, the optimization landscape of \eqref{general_problem} can be very complicated with numerous local optima and saddle points. Here, we consider the case where all the saddle points satisfy the strict saddle property and every local optimum is isolated as stated in Assumption \ref{ass_isolated}. 
\begin{assumption}[Isolated  Optima and Strict Saddle Points]\label{ass_isolated}~

Denote $S=\{x\in\RR^d|\nabla\cF(x)=0\}$ as the set of all stationary points.  For  $x'\in S,$ $x'$ must be one of the following:
\begin{itemize}
\item A strict saddle point such that  $\lambda_{\min}(\nabla^2\cF(x'))<0.$
\item An isolated local optimum such that  $\lambda_{\min}(\nabla^2\cF(x'))>0.$
\end{itemize}
\end{assumption}
We want to remark that our analysis can also be extended to study the case of connected local optima. However, the proof will be much more involved. Please refer to Section \ref{discussion} for detailed discussion.

We apply SGD with Polyak's momentum \citep{polyak1964some} (MSGD for short) to solve  \eqref{general_problem}. At the $k$-th iteration, MSGD takes the following update
\begin{align}\label{alg:0}
	x_{k+1}=x_k-\eta\nabla f(x_k,\xi_k)+\mu(x_k-x_{k-1}),
\end{align}
where $\eta >0$ is the step size, and $\mu(x_k-x_{k-1})$ is the momentum with the momentum parameter $\mu \in [0,1)$. When $\mu=0$, \eqref{alg:0} is reduced to SGD. We remark that though we focus on Polyak's momentum, extending our theoretical analysis to Nesterov's momentum  \citep{nesterov1983method} is straightforward.

\section{Analyzing Global Dynamics by ODE}\label{Section_ODE}
We first analyze the global dynamics of Momentum SGD (MSGD) by taking a diffusion approximation approach. Roughly speaking, by taking the step size $\eta\rightarrow 0,$ the continuous-time interpolation of the iterations $\{x_k\}_{k=0}^\infty$, which can be treated as a stochastic process with { C\`adl\`ag paths (right continuous with left-hand limits), becomes a continuous stochastic process. For MSGD, this continuous process follows an ODE with a unique solution. This ODE helps us understand how the momentum affects the global dynamics. We remark that the momentum parameter $\mu$ is a \emph{fixed constant} in our analysis.

More precisely, we define the continuous-time interpolation  $X^\eta(\cdot)$ of the solution trajectory of the algorithm as follows: 
 for $t\geq 0$, set $X^\eta(t)=x^{\eta}_k$ on the time interval $[k\eta,k\eta+\eta).$  Throughout our analysis, similar notations are applied to other interpolations (e.g. $H^\eta(t)$, $U^\eta(t)$).  We then answer the following question:
\begin{center}
	\textbf{\emph{Does the solution trajectory sequence  $\{X^\eta(\cdot)\}_\eta$ converge weakly as $\eta$ goes to zero? \\ If so, what is the limit?}}
\end{center} 

This question has been studied for VSGD in existing literature for special nonconvex optimization problems, such as streaming PCA in  \cite{chen2017online}.  The Infinitesimal Perturbation Analysis (IPA) technique is widely used to show that under some regularity conditions,
$X^\eta(\cdot)$ converges weakly to a solution of the following ODE:
\begin{equation}\label{ODE_SGD}
\dot{X}(t) = -\nabla \cF(X(t)).
\end{equation}
This method, however, cannot be applied to analyze MSGD due to the additional momentum term. Here, we explain why this method fails. Rewrite the algorithm \eqref{alg:0} as
\begin{align*}
\delta_{k+1}=\mu \delta_k-\eta\nabla f(x_k,\xi_k),\quad x_{k+1}=x_{k}+\delta_{k+1}.
\end{align*}
One can easily check $(\delta_k, x_k)$ is Markovian.  To apply IPA, the Infinitesimal Conditional Expectation (ICE)  must converge to a constant. However, the ICE for MSGD, which can be calculated as follows:
\begin{align*}
\frac{\EE[\delta_{k+1}-\delta_k|\delta_k,x_k]}{\eta}=\frac{(\mu-1)\delta_k}{\eta}-\nabla\cF(x_k),
\end{align*}
goes to infinity (blows up). Thus, IPA is not applicable here.

To address this challenge, we provide a new technique to prove the weak convergence and find the desired ODE. In a nutshell, we first prove rigorously the weak convergence of the trajectory sequence, and then using martingale theory to find the ODE. For self-containedness, we provide a summary on the pre-requisite weak convergence theory  in Appendix \ref{summary}.

{
Under Assumption \ref{assumption_general}, we characterize the global behavior of MSGD as follows.
\begin{theorem}\label{Thm1}
	Let $D^d[0,\infty)$ be the space of $\RR^d$-valued operators  which are right continuous and have left-hand limits for each dimension.  Suppose $x_0=x_1\in\RR^d$. Then for each subsequence of $\{X^{\eta}(\cdot)\}_{\eta>0}$, there exists a further subsequence  and a process $X(\cdot)$ such that
	$X^{\eta}(\cdot)\Rightarrow X(\cdot)$ in the weak sense as $\eta\rightarrow0$ through the convergent subsequence in the space $D^d[0,\infty)$, where $X(\cdot)$ satisfies the following ODE:
\begin{align}\label{ODE0}
\dot{X}=-\frac{1}{1-\mu}\nabla\cF(X),\quad X(0)=x_0.
\end{align}
Moreover,  for any $\delta>0,$ there exists a sequence
$T^{\eta,\delta}\rightarrow \infty$ such that
$$\limsup_{\eta\rightarrow 0} \PP\left(\exists t\leq T^{\eta,\delta},~ X^{\eta}(t)\neq N_\delta(S)\right)= 0,$$
where $N_\delta(S)$ is the $\delta$-neighborhood of the stationary points.

\end{theorem}

\begin{proof}[ Proof Sketch.]
To prove this theorem, we first show  the trajectory sequence  $\{X^\eta(\cdot)\}_\eta$ converges weakly. By Prokhorov's Theorem \ref{Thm_Prohorov} (in Appendix \ref{summary}), we need to prove tightness, which means   $\{X^\eta(\cdot)\}_\eta$ is bounded in probability in space $D^d[0,\infty)$. This can be proved by Theorem \ref{Thm_Tight} (in Appendix \ref{summary}), which requires the following two conditions: (1) $x_k$ must be bounded in probability for any $k$ uniformly in step size $\eta$; (2) The maximal discontinuity (the largest difference between two iterations,  i.e., $\max_k\{x_{k+1}-x_{k}\}$) must go to zero as $\eta$ goes to $0.$ This can be shown by using the bounded gradient assumption.

We next compute the weak limit.  For simplicity, we define
\begin{align*}
\textstyle\beta_k^\eta=-\sum_{i=0}^{k-1}\mu^{k-i}[\nabla f(x_i^\eta,\xi_i)-\nabla \cF(x_i^\eta)]~~\text{and}~~\epsilon_k^\eta=-(\nabla f(x_k^\eta,\xi_k)-\nabla \cF(x_k^\eta)).
\end{align*}
We then rewrite the algorithm as follows:
\begin{align}
m_{k+1}^\eta=m_k^\eta+(1-\mu)[-m_k^\eta+\widetilde{M}(x_k^\eta)],~~x_{k+1}^\eta=x_{k}^\eta+\eta (m^\eta_{k+1}+\beta^\eta_k+\epsilon^\eta_k),\label{2Tsystem2}
\end{align}
where $\displaystyle \widetilde{M}(x)=-(1-\mu)^{-1}\nabla\cF(x).$ 
 The basic idea of the proof is to  view  \eqref{2Tsystem2} as a two-time-scale algorithm \citep{borkar1997stochastic,borkar2009stochastic}, where  $m_k$ is updated with a larger step size $(1-\mu)$  and thus under a faster time-scale, and $v_k$ is under a slower one. Then we can treat the slower time-scale iterate $v$ as static and replace the faster time-scale iterate $m$ by its stable point in term of this fixed $v$ in \eqref{2Tsystem2}. This stable point can be shown to be $\widetilde{M}(x)$.

We then show that the continuous-time interpolation of the error $[m^\eta_{k+1}-\widetilde{M}(x^\eta_k)]+\beta^\eta_k+\epsilon^\eta_k$ converges weakly to a Lipschitz continuous martingale with zero initialization. From the martingale theory, we know such kind of martingales must be a constant. Thus, the error sequence converges weakly to zero, and what is left  is actually the discretization of ODE \eqref{ODE0}. Please refer to Appendix \ref{proofODE} for the detailed proof. 
\end{proof}

Note that for any solution  $X(t)$ to \eqref{ODE_SGD}, i.e., the mean ODE of SGD, $X(\frac{1}{1-\mu}t)$ is a solution to \eqref{ODE0}. This implies that asymptotically, MSGD is $\frac{1}{1-\mu}$ faster than SGD to converge to the neighborhood of a stationary point given the same initialization. Intuitively, with the help of the momentum, the algorithm makes more progress along the descent direction, and therefore momentum can accelerate the algorithm asymptotically.

However, since the noise of the stochastic gradient diminishes as $\eta\rightarrow 0,$ such a deterministic ODE-based approach is insufficient to analyze the local behavior of MSGD around stationary points where the noise plays a dominant role over the vanishing gradient. Thus, we resort to the following SDE-based approach for a more precise characterization.

\section{Analyzing Local Dynamics by SDE}\label{Section_SDE}

To characterize the local algorithmic behavior, we need to rescale the influence of the noise. For this purpose, we consider the \emph{normalized error} $\frac{x_k-x^*}{\sqrt{\eta}}$ under the diffusion approximation framework, where $x^*\in S$ is a stationary point. Different from the previous ODE-based approach, we obtain an SDE approximation here. Intuitively, the previous ODE-based approach is analogous to the Law of Large Number for random variables, while the SDE-based approach serves the same role as Central Limit Theorem. 

Recall that under Assumption \ref{ass_isolated}, the optimization problem \eqref{general_problem} has strict saddle points and isolated local optima. We remark that the assumption on isolated local optima helps avoid the cases where the normalization error explodes when the iterate wanders along the connected local optima. Our analysis can be further extended to handle connected local optima. However, the analysis will be much more complicated. Please refer to Section \ref{discussion} for detailed discussion. For consistency, we first study the algorithmic behavior around the local optimum.
{
\begin{remark}
 The $\sqrt{\eta}$ normalization actually normalizes the error by its standard deviation. Specifically, consider the $\lfloor1/\eta\rfloor-$th iterate of SGD initialized at the stationary point $x^*,$ $$x_{\lfloor1/\eta\rfloor}= x^*-\eta\sum_{i=0}^{\lfloor1/\eta\rfloor-1}\nabla f(x_i) -\eta\sum_{i=0}^{\lfloor1/\eta\rfloor-1} \xi_i ,$$ where $f(x)$ is the objective to be maximized and $\{\xi_i\}_i$ are the noise in the stochastic gradient  i.i.d. sampled from some unknown distribution with mean zero and bounded variance. Given the continuity of the gradient, $\nabla f(x_i)$ is approximately zero and noise will dominate  around the stationary point $x^*.$ Therefore, $x_{\lfloor1/\eta\rfloor}$ can be further approximated as follows.
$$x_{\lfloor1/\eta\rfloor}\approx x^*-\eta\sum_{i=0}^{\lfloor1/\eta\rfloor-1} \xi_i .$$ 
Thus, the variance of the error $x_{\lfloor1/\eta\rfloor}-x^*$ is of order $O(\eta)$:
$$\Var\left(x_{\lfloor1/\eta\rfloor}-x^*\right ) = \eta^2 \Var\left(\sum_{i=0}^{\lfloor1/\eta\rfloor-1}\xi_i\right) = O(\eta).$$
Therefore, we actually normalize the error by its standard deviation $O(\sqrt{\eta})$, which is analogous to rescaling the sample sum by $\sqrt{N}$ in Central Limit Theorem.
 \end{remark}
}

\subsection{Local Dynamics Around Local Optima }\label{around_opt}

We first consider the algorithmic behavior of MSGD when it is around a local optimum $x^*$. Define the normalized process $u_{k}^{\eta}=(x_{k}^{\eta}-x^*)/\sqrt{\eta},$ where $\lambda_{\min}(\nabla^2 \cF(x^*))>0.$ Accordingly, $U^{\eta}(t)=(X^{\eta}(t)-x^*)/\sqrt{\eta}.$ The next theorem characterizes the limiting process of $U^{\eta}(t).$


\begin{theorem}\label{Thm_SDE1}
	As $\eta\rightarrow 0$, $\{U^{\eta}(\cdot)\}$  converges weakly to the unique stationary solution of 
\begin{align}\label{SDE1}
dU= -\frac{1}{1-\mu} \nabla^2 \cF(x^*) U dt+\frac{1}{1-\mu} dW_t,
	\end{align}
where  $\{W_t\}$ is a Wiener process with covariance matrix $\Sigma =\EE[\nabla f(x^*,\xi)\nabla f(x^*,\xi)^\top].$ 
\end{theorem}}
Note that  our analysis is very different from that in \cite{chen2017online} because of the failure of IPA due to the similar blow-up issue. We remark that our technique mainly relies on Theorem \ref{thm8_1} (in Appendix \ref{summary}) from \cite{kushner2003stochastic}. Since the proof is much more sophisticated and involved than IPA, we introduce the key technique, Fixed-State-Chain, in a high level.

\begin{proof}[Proof Sketch.]
 Note that the algorithm can be rewritten as
\begin{align*}
x_{k+1}^{\eta,i}&\textstyle=x_{k}^{\eta,i}-\eta\Big[\sum_{j=1}^{k-1}\mu^{k-j}\nabla f(x_j^\eta,\xi_j)+\nabla \cF(x_k)\Big]-\eta\left[\nabla f(x_k^\eta,\xi_k) - \nabla \cF(x_k^\eta)\right].
\end{align*}
Here, for a vector $x\in \RR^d$ and an integer $i\leq d$, $x^{(i)}$ represents the $i$-th dimension of $x$. We define
\begin{align*}
\zeta^\eta_{k}&\textstyle =-\Big[\sum_{j=1}^{k-1}\mu^{k-j}\nabla f(x_j^\eta,\xi_j)\Big],~~Z^\eta_{k} =g(\zeta^\eta_k,x^\eta_k)+\gamma_{k}^\eta,\\
\gamma_{k}^\eta & =\nabla \cF(x_k^\eta) - \nabla f(x_k^\eta,\xi_k) \quad\textrm{and}\quad g(\zeta_k^\eta,x_k^\eta) =\zeta_k^\eta-\nabla \cF(x_k^\eta) .
\end{align*}
Here, $g$ is the accelerated gradient flow, and $\gamma_k^\eta$ is the noise.   Then the algorithm becomes
$$ x_{k+1}^{\eta}=x_{k}^{\eta}+\eta Z_{k}^{\eta}=x_{k}^{\eta}+\eta g(\zeta_k^{\eta},x_k^{\eta})+\eta \gamma_{k}^{\eta},$$
and thus we have
$ u_{k+1}^{\eta}=u_{k}^{\eta}+\sqrt{\eta}[g(\zeta_k^{\eta},x_k^{\eta})+\gamma_{k}^{\eta}].$ Note that $ g(\zeta_k^{\eta},x_k^{\eta}) \in \mathcal{F}_k^\eta~\textrm{and}~\EE[\gamma_{k}^{\eta}|\mathcal{F}_k^\eta]=0$ imply that the noise $\{\gamma_{k}^{\eta}\}$ is a martingale difference sequence.

We then manipulate the algorithm to extract the Markov structure of the algorithm in an explicit form. To make it clear, given $X$, there exists a transition function $P(\cdot,\cdot|X)$ such that
$$P\{\zeta_{k+1}^{\eta}\in\cdot|\mathcal{F}_k^{\eta}\}=P(\zeta_k^{\eta},\cdot|X=x_k^{\eta}).$$
This comes from the observation $\zeta_{k+1}^{\eta}=\mu\zeta_{k}^{\eta}-\mu\nabla  f(x_k^\eta,\xi_k^{\eta}),$
where the randomness only comes from $\xi_k$ when state $x_k$ is given. Then the fixed-state-chain  refers to the Markov chain with transition function $P(\cdot,\cdot|X)$ for a fixed $X$. The state of this Markov chain is denoted by $\{\zeta_k(X)\}$.  For notational simplicity, let $\tilde{M}(x)=-\frac{1}{1-\mu}\nabla\cF(x).$ We then  decompose  $x_{k+1}^{\eta}-x_{k}^{\eta}$ as follows:
\begin{align}\label{eq11}
x_{k+1}^{\eta}-x_{k}^{\eta}
&=\eta\tilde{M}(x_k^\eta)+\eta \gamma_{k}^{\eta}+\eta[g(\zeta_k(x_k^{\eta}),x_k^{\eta})-\tilde{M}(x_k^\eta)]\notag\\
&\hspace{0.5in}+\eta [g(\zeta_k^{\eta},x_k^{\eta})-g(\zeta_k(x_k^{\eta}),x_k^{\eta})]=\eta\tilde{M}(x_k^\eta)+\eta W_{k}^{\eta}.		
\end{align}
The error term $W_{k}^{\eta}$ in \eqref{eq11} comes from three sources:  (1) difference between the fixed-state-chain and the limiting process: $g(\zeta_k(x_k^{\eta}),x_k^{\eta})-\tilde{M}(x_k^\eta)$; (2) difference between the accelerated gradient flow and the fixed-state-chain: $g(\zeta_k^{\eta},x_k^{\eta})-g(\zeta_k(x_k^{\eta}),x_k^{\eta})$; (3) the noise $\gamma_{k}^\eta$.

We handle them separately and combine the results together to get the variance of  $ W_{k}^{\eta,i}$. 
Note that $\{u_{k}^{\eta}\}$ satisfies the following update: $$u_{k+1}^{\eta}-u_{k}^{\eta}=\eta \frac{\tilde{M}(x_k^\eta)}{\sqrt{\eta}}+\sqrt{\eta} W_{k}^{\eta}.$$ Together with the fact that around the optimum $x^*$, $ \tilde{M}(x)=-\frac{1}{1-\mu}\nabla^2 \cF(x^*) (x-x^*)+o\left(\|(x-x^*)\|_2\right),$ we further obtain
\begin{align}\label{discrete_SDE}
 \frac{u_{k+1}^{\eta}-u_{k}^{\eta}}{\eta}&=-\frac{1}{1-\mu}\nabla^2 \cF(x^*) u_k^\eta+\frac{W_{k}^{\eta}}{\sqrt{\eta}}+o\left(|u_{k}^{\eta,1}|\right).
\end{align}
After calculating the variance of $W$, we see that essentially \eqref{discrete_SDE} is the discretization of SDE \eqref{SDE1}.
For the detailed proof, please refer to Appendix \ref{proof_Thm_SDE1}.
\end{proof}

Note that \eqref{SDE1} admits an explicit solution which is known as an Ornstein-Uhlenbeck (O-U) process \citep{oksendal2003stochastic} having the following expression:
\begin{align*}
U(t)= \exp\left(-\frac{1}{1-\mu}\nabla^2\cF(x^*)t\right) U(0) + \int_0^t \exp\left(\frac{1}{1-\mu}\nabla^2\cF(x^*)(t-s)\right) \frac{\Sigma^{\frac{1}{2}}}{1-\mu}dB_t.
\end{align*} 
Given $U(0),$ the above formula shows that $U(t)$ is Gaussian for all $t>0.$ Therefore, we can identify the limiting density of $U(t)$ as $t\rightarrow\infty$ by figuring out the limiting mean and covariance matrix. In fact the mean $m(t) = \EE(U(t))$ and covariance matrix $\rho_t =\EE[U(t)U(t)^\top]$ satisfy the following ODEs, respectively:
\begin{align*}
d m(t) &= -\frac{\nabla^2\cF(x^*)}{1-\mu}m(t) dt,\\
d \rho(t) &= -\frac{1}{1-\mu}\left(\nabla^2\cF(x^*)\rho + \rho \nabla^2\cF(x^*) \right)+\frac{1}{(1-\mu)^2}\Sigma.
\end{align*}
Since $\nabla^2\cF(x^*)$ is positive definite, we have $m(t)\rightarrow 0$ and 
$$\rho_\mu= \lim_{t\rightarrow\infty}\rho(t) = \int_0^\infty  \exp\left(-\frac{1}{1-\mu}\nabla^2\cF(x^*)s\right) \frac{1}{(1-\mu)^2}\Sigma \exp\left(-\frac{1}{1-\mu}\nabla^2\cF(x^*)s\right) ds<\infty.$$
Therefore, when MSGD enters the neighborhood of a local optimum, it will stay near the local optimum and behave like a Brownian motion. Moreover, by a change of variables, we can rewrite $\rho_\mu$ as follows:
\begin{align*}
\rho_\mu &=\int_0^\infty  \exp\left(-\frac{1}{1-\mu}\nabla^2\cF(x^*)s\right) \frac{1}{(1-\mu)^2}\Sigma \exp\left(-\frac{1}{1-\mu}\nabla^2\cF(x^*)s\right) ds\\
&= \frac{1}{(1-\mu)} \int_0^\infty  \exp\left(-\nabla^2\cF(x^*)s\right)\Sigma \exp\left(-\nabla^2\cF(x^*)s\right) ds\\
& = \frac{1}{(1-\mu)}\rho_0.
\end{align*}

We see clearly that the momentum essentially increases the variance of the normalized error by a factor of $\frac{1}{1-\mu}$  around the local optimum compared with VSGD. Thus, it becomes harder for the algorithm to converge. The next theorem provides a more precise characterization of such a phenomenon. 
\begin{theorem}\label{lemma_phase3}
 Let $(\lambda_i, e_i)$'s be the eigenvalue, eigenvector pairs of $\nabla^2\cF(x^*)$ such that  $\lambda_1\geq \lambda_2 \geq ...\geq \lambda_d>0.$
Given a sufficiently small $\epsilon>0$ and $\phi=\sum_{i=1}^d \Sigma_{i,i} < \infty$, we need the step size $\eta$ satisfying
\begin{align}\label{eqconverge}
\eta<(1-\mu)\lambda_d\epsilon/(4\phi)
\end{align}
such that $X^\eta(t)$ enters the $\epsilon$-neighborhood of the local optimum with probability at least $3/4$ at some time $T_3$  after restarting the counter of time, i.e., 
$\norm{X^{\eta}(T_3)-x^*}_2^2\leq\epsilon,$
where
	\begin{align*}
 T_3\asymp\frac{(1-\mu)}{2\lambda_d}\cdot\log \Big(\frac{8\lambda_d\delta^2}{\lambda_d\epsilon-4\eta\phi}\Big), 
\end{align*}
given $\norm{X^{\eta}(0)-x^*}_2^2\leq \delta^2.$
\end{theorem}
Note that when $\mu = 0,$ we can choose the step size of VSGD as $\eta_0\asymp\frac{\lambda_d\epsilon }{4\phi}$, which does not satisfy \eqref{eqconverge} for $\mu$ close to $1$. This means that \emph{when using the same step size of VSGD,  MSGD fails to converge,  since the variance increased by the momentum becomes too large.} To handle this issue, we have to decrease the step size by a factor $1-\mu$, also known as the step size annealing, i.e.,
 \begin{align}\label{eta}
 \eta\asymp(1-\mu)\epsilon\lambda_d/(4\phi)\asymp(1-\mu)\eta_0.
 \end{align}	
We also want to remark that here the probability $3/4$ can be any constant in $(0,1).$ Theorem  \ref{lemma_phase3} implies the algorithm needs asymptotically at most
\begin{align*}
N_3\asymp\frac{T_3}{\eta}\asymp\frac{\phi}{\epsilon \lambda_d^2}\cdot\log \Big(\frac{8\lambda_d\delta^2}{\lambda_d\epsilon-4\eta_0\phi}\Big)
\end{align*}
iterations to converge to an $\epsilon$-optimal solution. Note that the $N_3$ does not depend on $\mu$. Therefore, MSGD does not have an advantage over VSGD around local optima. 

\subsection{Local Dynamics Around Saddle Points }
We then study the algorithmic  behavior around strict saddle points. Define the normalized process $u_{k}^{\eta}=(x_{k}^{\eta}-\hat x)/\sqrt{\eta},$ where $\hat x\in S, \lambda_{\min}(\nabla^2 \cF(\hat x))<0.$ Accordingly, $U^{\eta}(t)=(X^{\eta}(t)-\hat x)/\sqrt{\eta}.$ By the same SDE approximation technique used in Section~\ref{around_opt}, we  obtain the following theorem.
\begin{theorem}\label{SDE_Saddle}
For any $C>0,$ there exist $\delta>0$ and $\eta'>0$ such that
\begin{align}\label{eq_thm3_dim}
\sup_{ \eta<\eta'}\PP(\sup_{\tau>0} \norm{U^{\eta}(\tau)}_2\leq C)\leq 1-\delta.
\end{align}


\end{theorem}
\begin{proof}[Proof Sketch.] 
 We  prove \eqref{eq_thm3_dim} by contradiction. Assume the conclusion does not hold, that is there exists a constant $C>0,$ such that for any  $\eta'>0$ we have $$\sup_{\eta\leq \eta'}\PP(\sup_{\tau>0} |U^{\eta}(\tau)|\leq C)=1.$$   That implies there exists a sequence $\{\eta_n\}_{n=1}^\infty$ converging to $0$ such that 
\begin{align}\label{eq_contra}
\lim_{n\rightarrow\infty}\PP(\sup_{\tau>0} |U^{\eta_n}(\tau)|\leq C)= 1.
\end{align}
We next show that  this subsequence $\{U^{\eta_n}(\cdot)\}_n$ is tight. To do so, we need to verify two conditions of Theorem \ref{Thm_Tight0} in Appendix \ref{summary}.
By \eqref{eq_contra}, we know that condition (i) in Theorem \ref{Thm_Tight0} holds. We next check  condition(ii) in Theorem \ref{Thm_Tight0}. When $\sup_{\tau>0} |U^{\eta_n,i}(\tau)|\leq C$ holds, Assumption \ref{assumption_general} yields that $\norm{u_{k+1}^{\eta_n}-u_{k}^{\eta_n}}_2\leq C'\eta_n,$ where $C'$ is some constant. Thus, for any $t,\epsilon>0,$ we have $$\|U^{\eta_n}(t)-U^{\eta_n}(t+\epsilon)\|_2\leq \epsilon/\eta C'\eta=C'\epsilon,$$ or equivalently $$\varpi'_T (U^{\eta_n},\epsilon)\leq C'\epsilon, \forall T>0,$$ where $\varpi$ is the modulus of continuous defined in Definition \ref{modulus}.  Thus, condition (ii) in Theorem \ref{Thm_Tight0} holds. Then we have  $\{U^{\eta_n}(\cdot)\}_n$ is tight and thus converges weakly. Following similar lines to Theorem \ref{Thm_SDE1}, we can verify C.5-C.8 and show that 
 $\{U^{\eta_n}(\cdot)\}_n$ converges weakly to a solution of	
 \begin{align}\label{SDE_1}
dU= -\frac{1}{1-\mu} \nabla^2 \cF(\hat x) U dt+\frac{1}{1-\mu} dW_t.
  \end{align} 
The process defined by \eqref{SDE_1} is an unstable O-U process. When initialized at $\hat x,$ it has mean $0$ and exploding variance. When not initialized at $\hat x,$ it has exploding mean and variance. Thus, for any $\delta,$ there exist a time $\tau'$, such that
$$\PP(\|U(\tau')\|_1\geq C)\geq 2\delta.$$
Since $\{U^{\eta_n}\}_n$ converges weakly to $U,$ $\{U^{\eta_n}(\tau')\}_n$ converges in distribution to $U(\tau').$ This implies that there exists $N>0$, such that for any $n>N,$ 
$$|{\PP(\|U(T)\|_2\geq C)-\PP(\|U^{\eta_n}(T)\|_2\geq C)}|\leq {\delta}.$$
Then we find a $\tau'>0$ such that 
$$\PP(\|U^{\eta_n}(\tau')\|_2\geq C)\geq {\delta}, \quad\forall n>N, $$ 
or equivalently $$\PP(\|U^{\eta_n}(\tau')\|_2\leq C)< 1-{\delta},  \quad\forall n>N.$$ 
Since $\left\{\omega\big|\sup_\tau \|U^{\eta_n}(\tau)(\omega)\|_2\leq C\right\}\subset\left\{\omega\big|\|U^{\eta_n}(\tau')(\omega)\|_2<C\right\},$ we have 
$$\PP(\sup_\tau \|U^{\eta_n}(\tau)\|_2\leq C)\leq 1-{\delta}, \quad\forall n>N, $$ 
which leads to a contradiction with \eqref{eq_contra}. Our assumption does not hold. We prove Theorem \ref{SDE_Saddle}.
\end{proof}

Theorem~\ref{SDE_Saddle} implies that with a constant probability $\delta,$ MSGD escapes from the saddle points at some time $T_1$, i.e., $\norm{X^\eta(T_1)-\hat x}_2^2$ is greater than $\delta^2$ ($\delta=\cO (\sqrt\eta)$).  Note that  from the proof  of Theorem~\ref{SDE_Saddle}, when the step size $\eta$ is small, the process defined by SDE \eqref{SDE_1}  characterizes the local behavior of $X^\eta$ around saddle points.  For any fixed $\mu, $ let $U_\mu$ be the solution to \eqref{SDE_1}. Then we can verify that
\begin{align*}
\EE(U_\mu(t)) =	\frac{1}{1-\mu}\EE(U_0(t)), ~~~ \Var(U_\mu(t))  = \frac{1}{1-\mu}\Var(U_0(t)).
\end{align*}
More precisely,  we can obtain  the following proposition on the asymptotic escaping rate of MSGD.

\begin{theorem}\label{Time_Saddle}
	Let $\nabla^2\cF(\hat x) = P\Lambda P^\top$ be the eigenvalue decomposition of $\nabla^2\cF(\hat x),$ where $\Lambda = \diag(\lambda_1,..., \lambda_d)$ and $\lambda_1\geq \lambda_2\geq...\geq\lambda_d$ and $\lambda_d<0.$ Denote $\Sigma =\EE[\nabla f(\hat x,\xi)\nabla f(\hat x,\xi)^\top].$  Given a pre-specified $\nu\in(0,1)$, $\eta\asymp\eta_0 $, and $\delta=\cO(\sqrt{\eta})$, then the following result holds: We need at most 
		\begin{align}\label{T1}
	T_1\asymp \frac{(1-\mu)}{2|\lambda_d|}\log\left(\frac{2\eta^{-1}\delta^2(1-\mu)|\lambda_d|}{\Phi^{-1}\left(\frac{9}{16}\right)^2(P^\top \Sigma P)_{d,d}} +1\right),
	\end{align}
such that $\displaystyle \norm{X_\eta(T_1)-\hat{x}}_2^2\geq \delta^2$ with probability at least $\frac{3}{4}$, where $\Phi(x)$ is the CDF of the standard normal distribution.
\end{theorem}

Theorem \ref{Time_Saddle} suggests that we need asymptotically 
\begin{align*}
N_1 \asymp \frac{(1-\mu)\phi}{|\lambda_d|^2\epsilon}\log\left(\frac{2(1-\mu)\eta^{-1}\delta^2|\lambda_d|}{\Phi^{-1}\left(\frac{1+\nu/2}{2}\right)^2(P^\top \Sigma P)_{d,d}} +1\right)
\end{align*}
iterations to escape from saddle points.
 \emph{Thus, when using the same step size, MSGD can escape from saddle points in fewer iterations than VSGD by a factor of $1-\mu$.} This is due to  the fact that  the momentum can greatly increase  the variance and perturb the algorithm more aggressively. Thus, it becomes harder to stay around saddle points. Moreover, the momentum also encourages more aggressive  exploitation, and in each iteration, the algorithm  makes more progress along the descent direction by a factor of $\frac{1}{1-\mu}$.

In summary, compared with VSGD ($\mu = 0$), momentum accelerates escaping from saddle points by a factor of $1-\mu.$ However, momentum can also hurt the final convergence around the local optimum because of the increased variance. Therefore,  we suggest to decrease the step size by a factor $1-\mu$ in the later stage,  MSGD can then achieve the similar convergence rate as VSGD. Note that we can also decrease the momentum parameter $\mu$ instead of the step size $\eta.$ We will show in Section \ref{Numerical} that momentum annealing and step size annealing can both ensure the convergence of MSGD.

\section{Example: Streaming PCA}\label{section_pca}

In this section, we apply our convergence analysis to study the algorithmic behavior of MSGD and provide explicit convergence result for the streaming PCA problem formulated as follows.
\begin{align}\label{eqn:PCA}
\max_v\; v^{\top}\EE_{X\sim\cD}[XX^{\top}]v \quad\textrm{subject to}~~ v\in \mathbb{S}=\{v\in\RR^{d}\,|\,\|v\|_2=1\}.
\end{align}

 For notational simplicity, we denote the covariance matrix as $\Sigma=\EE[XX^{\top}]$ .  Before we proceed, we impose the following assumption on $\Sigma$:
\begin{assumption}\label{Ass0}
The covariance matrix $\Sigma$ is positive definite with eigenvalues  $\lambda_1> \lambda_2\geq...\geq\lambda_d>0$ and associated normalized eigenvectors $v^1,\,v^2,\,...,\,v^d$.  Moreover, there exists an orthogonal matrix Q such that:
$\Sigma=Q\Lambda Q^{\top},$
where $\Lambda={\rm diag}(\lambda_1, \lambda_2,...,\lambda_d).$ 
\end{assumption}

Under this assumption, the optimization landscape of \eqref{eqn:PCA} has been well studied. \cite{chen2017online} have shown that the eigenvectors $\pm v^1,\,\pm v^2,\,...,\,\pm v^d$ are all the stationary points for problem \eqref{eqn:PCA} on the unit sphere $\mathbb{S}$. Moreover, the eigen-gap assumption ($\lambda_1> \lambda_2$) guarantees that the global optimum $v^1$ is identifiable up to sign change. Meanwhile, $v^2,\,...,\,v^{d-1}$  are $d-2$ strict saddle points, and  $v^d$  is the global minimum.

Given the optimization landscape of \eqref{eqn:PCA}, we have already understood well the behavior of VSGD algorithms, including Oja's rule and stochastic generalized Hebbian algorithms (SGHA) for streaming PCA \citep{chen2017online}.  We consider a variant of SGHA with Polyak's momentum \citep{polyak1964some}. Recall that we are given a streaming data set $\{X_k\}_{k=1}^\infty$ drawn independently from some zero-mean distribution $\cD$. At the $k$-th iteration, the algorithm takes
\begin{align}\label{alg0}
	v_{k+1}=v_k+\eta(I-v_kv_k^\top)\Sigma_kv_k+\mu(v_k-v_{k-1}),
\end{align}
where $\Sigma_k=X_kX_k^{\top}$ and $\mu(v_k-v_{k-1})$ is the momentum with a parameter $\mu \in [0,1)$. When $\mu=0$, \eqref{alg0} is reduced to SGHA.  A detailed derivation of \eqref{alg0} is provided in Appendix \ref{Appendix_PCA}.
 {
\begin{remark}
The constraint in problem \eqref{eqn:PCA} restricts the solution space to be a unit sphere $\SSS,$ which is a manifold. In order to match our algorithm \eqref{SGD_momentum}, we consider \eqref{eqn:PCA} to be an unconstraint optimization problem on the manifold by using the manifold gradient  $(I-xx^\top)\Sigma x$. For general manifold optimization problems,  additional projection may be required  to ensure the solution trajectory staying on the manifold. However,  for the sphere constraint as in \eqref{eqn:PCA}, when $\eta$ is small, moving along the direction of the manifold gradient, the solution trajectory can stay close to  $\SSS,$ as shown in Lemma \ref{lem_bound} in Appendix \ref{Appendix_PCA}.
\end{remark}
}
Before we proceed, we impose the  following assumption on the problem:
\begin{assumption}\label{Ass1}
The data points  $\{X_k\}_{k=1}^\infty$ are drawn independently from a distribution $\cD$ in $R^d$, such that:
$$\EE[X]=0, ~ \EE[XX^{\top}]=\Sigma,~ \|X\|\leq C_d,$$
 where $C_d$ is  a constant (possibly dependent on $d$).
\end{assumption}
This  uniformly boundedness assumption can actually be relaxed to the boundedness of the $(4+\delta)$-th-order moment ($\delta>0$) with a careful truncation argument. The proof, however, will be much more involved and beyond the scope of this paper. Thus, we use the uniformly boundedness assumption for convenience.  

Under Assumptions \ref{Ass0} and \ref{Ass1}, we first apply Theorem \ref{Thm1} and provide an ODE approximation for Algorithm \eqref{alg0} in the following corollary.
\begin{corollary}\label{ODE_solution}
 Suppose $v_0=v_1\in\mathbb{S}$. Then 
	$V^{\eta}(\cdot)\Rightarrow V(\cdot)$ in the weak sense as $\eta\rightarrow0$ in the space $D^d[0,\infty)$, where $V(\cdot)$ is the unique solution to the following ODE:
\begin{equation}\label{ODE_PCA}
\dot{V} = \frac{1}{1-\mu}(\Sigma V-V^{\top}\Sigma VV), ~~V(0)=v_0,
\end{equation}
and has the following explicit form $V(t)= QH(t),$ where
\begin{equation*}
H^{(i)}(t)=\Big(\sum_{i=1}^d [H^{(i)}(0)\exp\Big(\frac{\lambda_it}{1-\mu}\Big)]^2\Big)^{-\frac{1}{2}}H^{(i)}(0)\exp\left(\frac{\lambda_it}{1-\mu}\right),\quad  i=1,...,d,
\end{equation*}
where $H(0) = Q^\top v_0.$
 Moreover, suppose $v_0\neq \pm v^{i},~ \forall i=2,...,d,$  as $t\rightarrow\infty,$ $V(t)$ converges to $v^1,$ which is the global maximum to \eqref{eqn:PCA}.
\end{corollary}
 Please refer to Appendix \ref{proof_ode_solution} for the detailed proof. Different from the general ODE \eqref{ODE0}, ODE \eqref{ODE_PCA} has an explicit form solution which implies that whenever MSGD escapes from strict saddle points $v^i,~ i\geq 2$, it will directly converge to the global optimum $v^1.$ Therefore, we can provide a more precise characterization of the algorithmic behavior in the non-stationary area for streaming PCA than general nonconvex problems. Moreover, since streaming PCA has one isolated global optimum and strict saddle points, our SDE analysis for the stationary area can be directly applied. We have the following corollary to characterize the asymptotic convergence rate of MSGD. 
\begin{corollary}\label{corollary_rate}
Let $\eta$ be the step size of MSGD and $\eta_0\asymp\frac{(\lambda_1-\lambda_2)\epsilon }{\phi}$ be the step size of VSGD as chosen in \cite{chen2017online}.

\noindent $\bullet$ {\bf Phase I: Escape from Saddle Points.} Suppose $ v^\eta_0 = v^2,$ the strict saddle point corresponding to $\lambda_2.$ Given  $\eta\asymp\eta_0 $, and $\delta=\cO(\sqrt{\eta})$, we need asymptotically at most 
		\begin{align}\label{N1}
	N_1 \asymp \frac{(1-\mu)\phi}{(\lambda_1-\lambda_2)^2\epsilon}\log\left(\frac{2(1-\mu)\eta^{-1}\delta^2(\lambda_1-\lambda_2)}{\Phi^{-1}\left(\frac{9}{16}\right)^2\alpha^2_{12}} +1\right),
	\end{align}
iterations such that $\displaystyle \norm{v^\eta_{N_1}-v_2}_2^2\geq \delta^2$ with probability at least $3/4$, where $\Phi(x)$ is the CDF of the standard normal distribution.

\noindent $\bullet$ {\bf Phase II: Traverse from Saddle Points to the Global Optimum.} 
Suppose $\norm{v^\eta_{0}-v^i}_2^2\geq \delta^2, ~~\forall i\geq 2.$ For sufficiently small $\eta$, $\delta=\cO(\sqrt{\eta})$, we need
\begin{align}\label{N2}
 N_2\asymp\frac{(1-\mu)\phi}{2\epsilon(\lambda_1-\lambda_2)^2}\log\left(\frac{2-\delta^2}{\delta^2}\right)
 \end{align}
  iterations such that $\left\|v^{\eta}_{N_2} - v^1\right\|_2^2\leq \delta^2$ with probability at least ${3}/{4}.$

\noindent $\bullet$ {\bf Phase III: Converge to the Global Optimum.} 
For a sufficiently small $\epsilon>0$ and $\eta\asymp(1-\mu)\eta_0$, there exists some constant $\delta=\cO(\sqrt{\eta})$, such that  $\left\|v^{\eta}_{0} - v^1\right\|^2\leq \delta^2$,
we need
\begin{align}\label{N3}
N_3\asymp\frac{\phi}{\epsilon (\lambda_1-\lambda_2)^2}\cdot\log \Big(\frac{8(\lambda_1-\lambda_2)\delta^2}{(\lambda_1-\lambda_2)\epsilon-4\eta_0\phi}\Big)
\end{align}
 iterations to ensure $\left\|v^{\eta}_{N_3} - v^1\right\|_2^2\leq\epsilon$ with probability at least $3/4$ .
 \end{corollary}
 
 Please refer to Appendix \ref{proof_rate} for the detailed proof.  From Corollary \ref{corollary_rate}, we can see clearly that momentum accelerates escaping from saddle points and traversal to the global optimum by a factor of $1-\mu.$ If we further decrease the step size in Phase III, MSGD can achieve the same convergence rate as VSGD.

\section{Numerical Experiments}\label{Numerical}
We present numerical experiments for both streaming PCA and training deep neural networks. The experiments on streaming PCA verify our theory in Section 5, and the experiments on training deep neural networks support our theoretical results for the general problem and also verify some of our discussions in Section 7 later. 
\subsection{Streaming PCA}

We first provide a numerical experiment to verify our theory for streaming PCA.  We set $d=4$ and the covariance matrix 
$\Lambda=\rm{diag}\{4,3,2,1\}.$ The optimum is $(1,0,0,0).$ 
Figure~\ref{fig} compares the performance of VSGD, MSGD (with and without the step size annealing, and momentum annealing in Phase III). The initial solution is the saddle point $(0,1,0,0)$. We choose
$\mu=0.9$ and $\eta=5\times 10^{-4}$, decrease the step size of MSGD by a factor $1-\mu$ after $2\times 10^{4}$ iterations in Figure~\ref{fig}.b, and decrease the momentum by a factor $1/10$ after $2\times 10^{4}$ iterations in Figure~\ref{fig}.c. Figure  \ref{fig} plot the results of 100 simulations, and the vertical axis corresponds to $||H_k^{(1)}|-1|$. We can clearly differentiate the three phases of VSGD in Figure \ref{fig}.a.  For MSGD in Figures \ref{fig}.b,  \ref{fig}.c and \ref{fig}.d, we hardly recognize Phases I and II, since they last for a much shorter time. This is because the momentum significantly helps escape from saddle points and evolve toward the global optimum. Moreover, we also observe  in Figure \ref{fig}.b that MSGD without the step size annealing and the momentum annealing does not converge well, but the step size annealing or the momentum annealing resolves this issue. All these observations are consistent with our analysis. Figure \ref{fig}.e plots the optimization errors of these three algorithms averaged over all 100 simulations, and we observe similar results.

\begin{figure}[htb!]
	\centering
		\subfigure[Three phases in SGD.]{
	\includegraphics[width=0.48\textwidth]{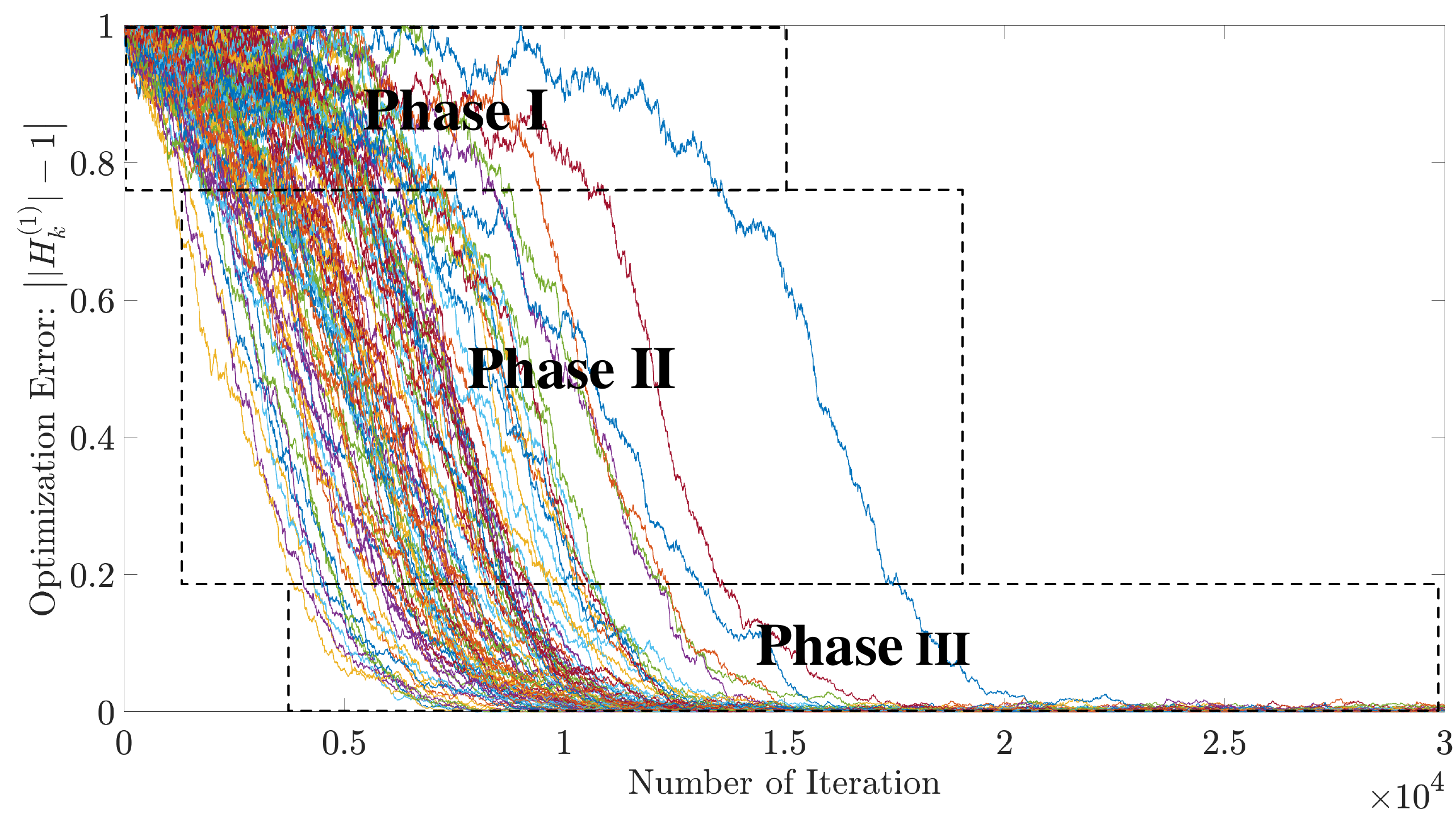}\label{fig:sgd}}
	\subfigure[MSGD does not converge.]{
	\includegraphics[width=0.48\textwidth]{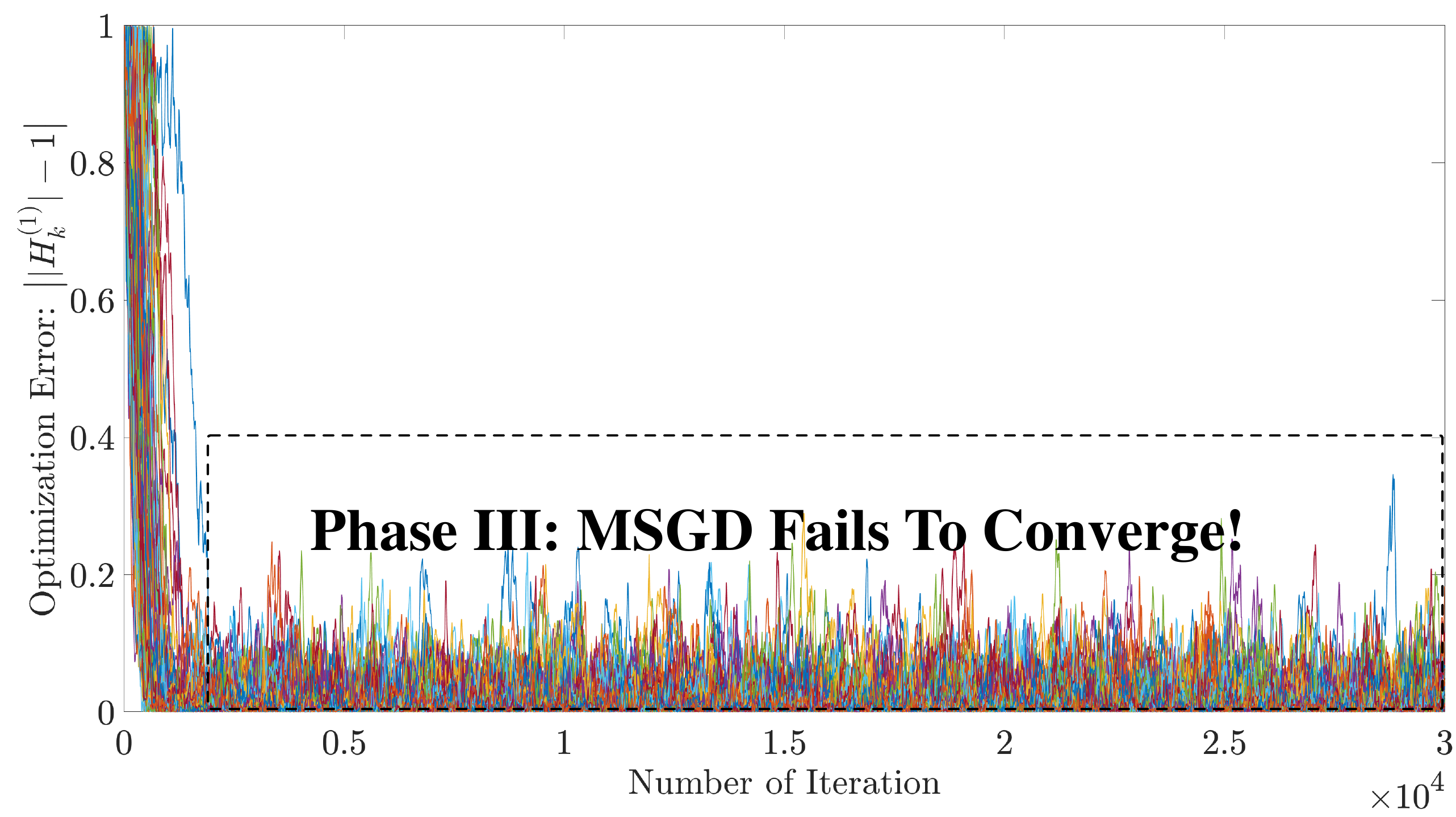}\label{fig:msgd}}\\
	\subfigure[MSGD with SSA converges.]{\includegraphics[width=0.48\textwidth]{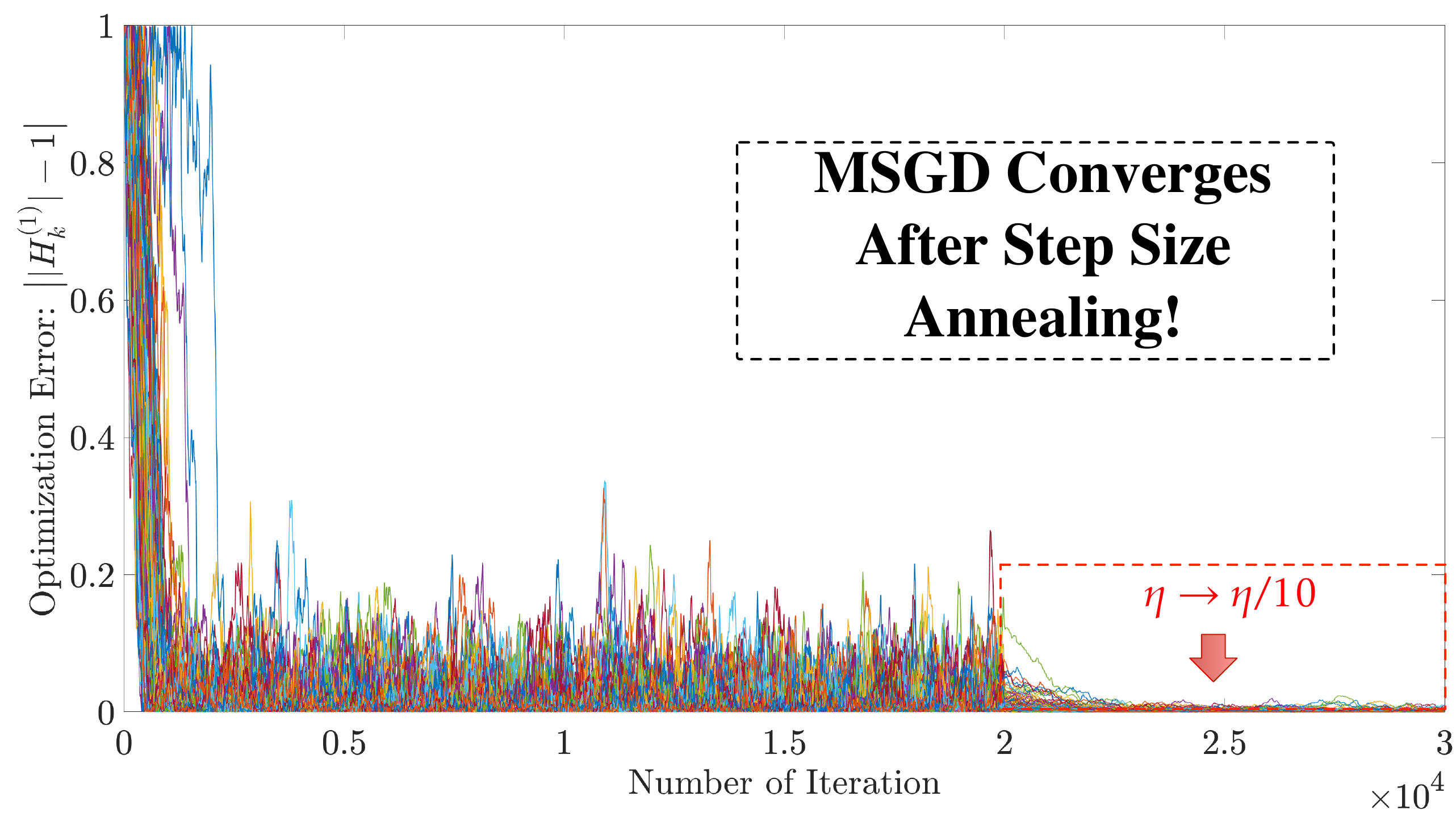}\label{fig:msgd_ssa}}
	\subfigure[MSGD with MA converges.]{ \includegraphics[width=0.48\textwidth]{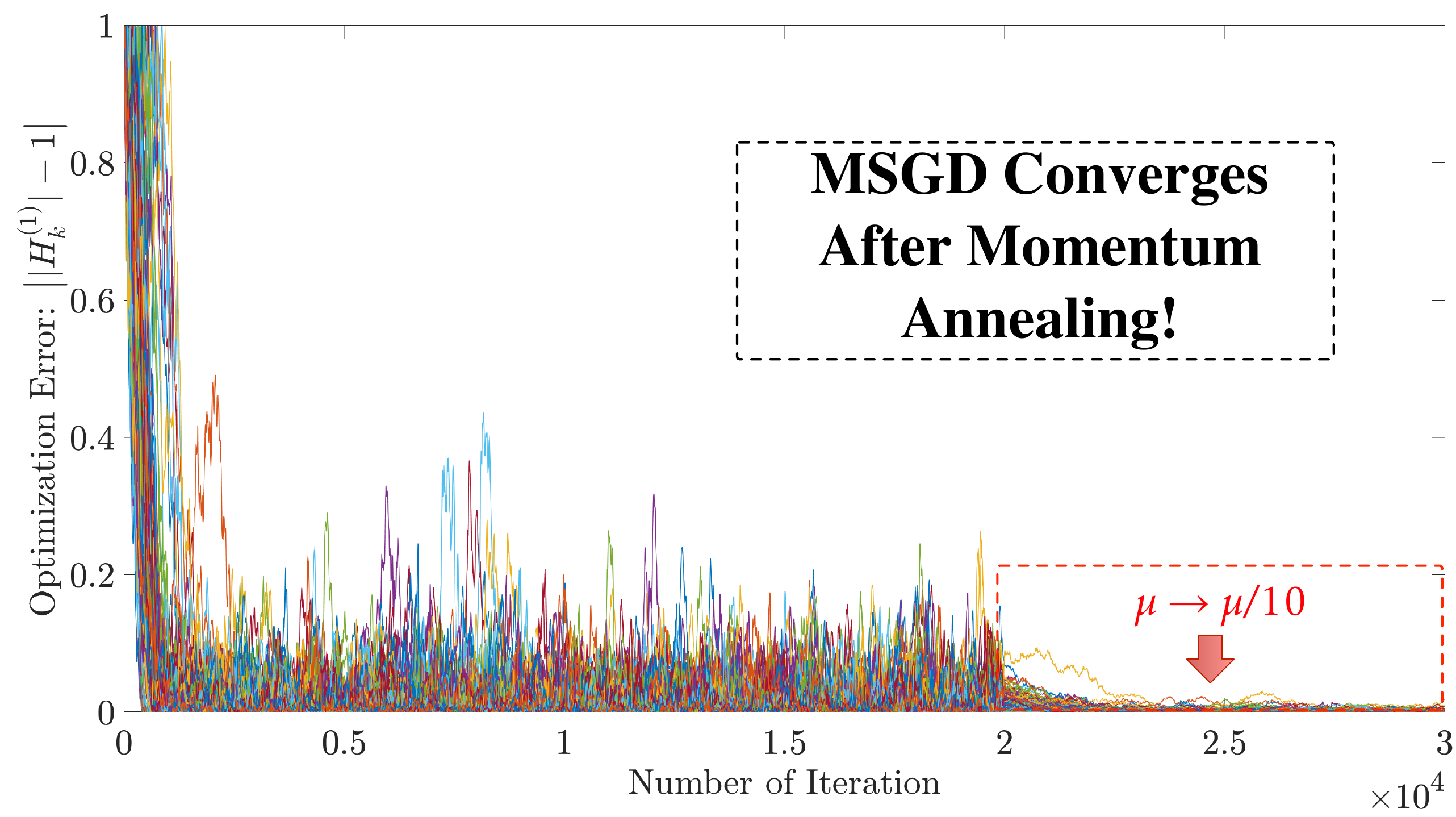}\label{fig:msgd_ma}}\\ 
		\subfigure[Comparison among SGD and different MSGDs.]{\includegraphics[width=0.48\textwidth]{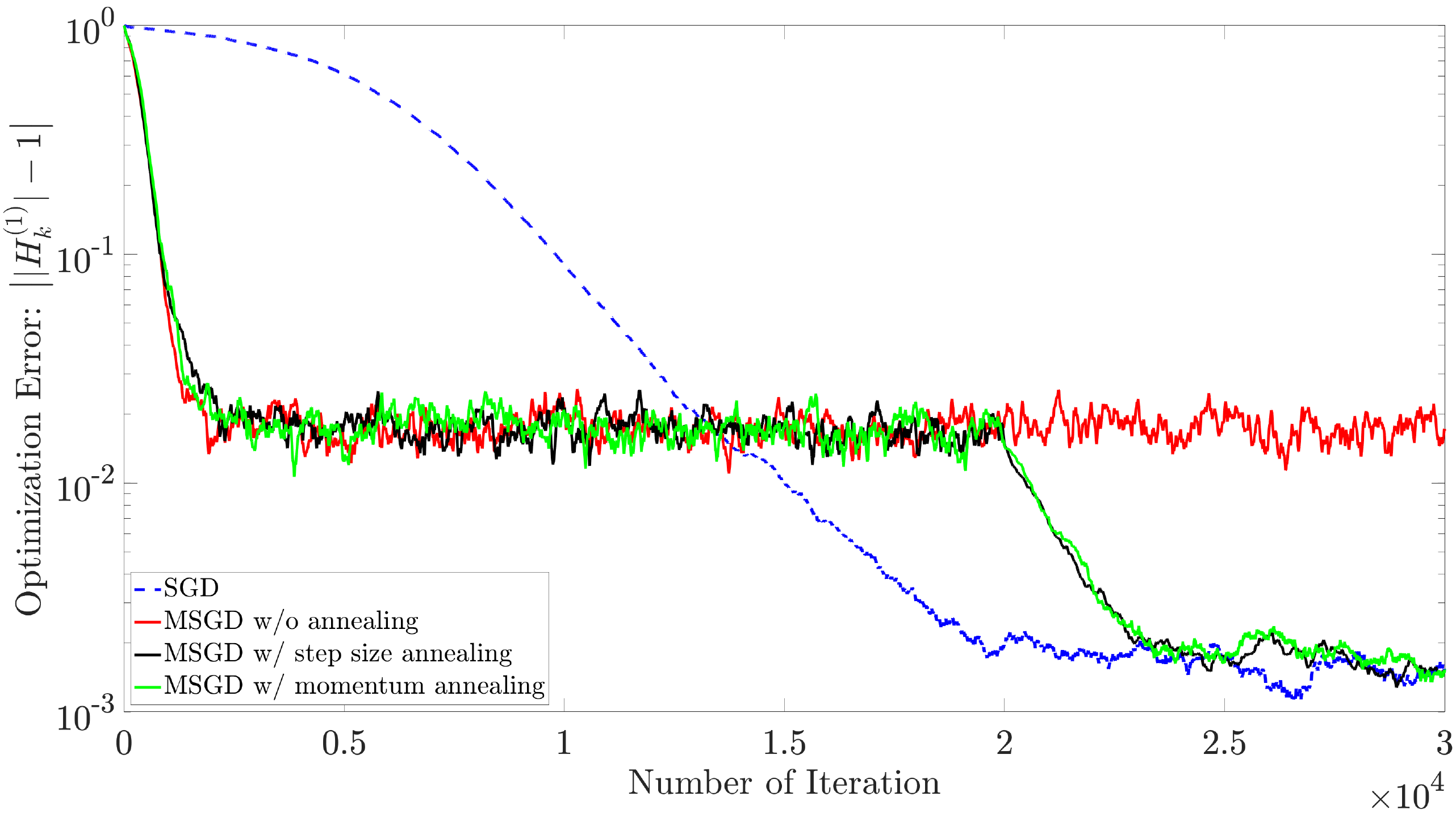}\label{fig:compare}}\\
		\caption{Comparison between SGD and MSGD (with and without the Step Size Annealing (SSA) and Momentum Annealing (MA) in Phase III).}	\label{fig}
\end{figure}

\subsection{Deep Neural Networks}

{Momentum SGD and its variants has been widely applied in training deep neural networks \citep{sutskever2013importance,kingma2014adam,goodfellow2016deep, he2016deep} and has been implemented in popular deep learning libraries, such as Tensorflow \citep{abadi2016tensorflow} and PyTorch \citep{paszke2019pytorch}.
} In this section, we present several experiments to compare MSGD with VSGD in training a 9-layer Residual Net (ResNet-9, \citet{david2018resnet}) over CIFAR-$10$ and CIFAR-$100$ datasets for $10$ and $100$-class image classification tasks, respectively. Both datasets contain $60$k images, in which $50$k images are used for training, and the rest $10$k are used for testing. The network architecture of ResNet-9 is shown in Figure~\ref{fig:archi} and summarized in Table~\ref{tab:archi}. All experiments are done in PyTorch with one NVIDIA RTX 2080-Ti GPU. For each experiment, we repeat for $20$ times with different random seeds and report the average and standard deviation. 

\begin{figure*}[t]
\begin{center}
\subfigure[An illustrative visualization of 9 layers in ResNet-9 (ResBlock contains 2 layers)]
	{\includegraphics[width = 0.92\textwidth]{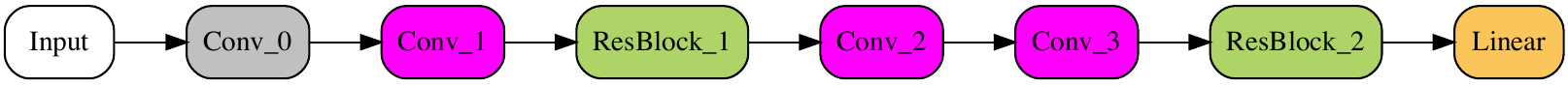}}
\subfigure[Grey convolutional layer]
{\includegraphics[width = 0.4\textwidth]{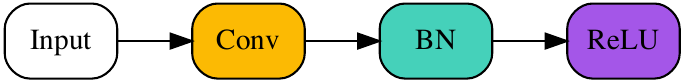}}
\subfigure[Pink convolutional layer]
{\includegraphics[width = 0.56\textwidth]{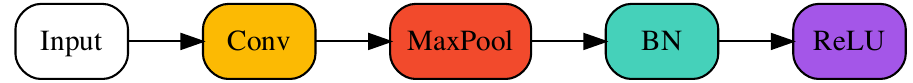}}\\
\subfigure[Residual block: containing two convolutional layers]
{\includegraphics[width = 0.85\textwidth]{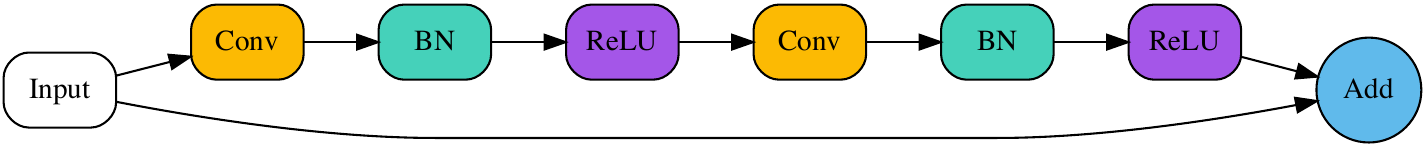}}
\end{center}
\caption{The Network Architecture of ResNet-9 and Its Detailed Components.}\label{fig:archi}
\end{figure*}

\begin{table*}[t]
\centering
\caption{Network Architecture of ResNet-9.}\label{tab:archi}
\begin{tabular}{c|c|c}
\hline
\hline
Layer & Output size & Filter, activation and pooling\\
\hline
Conv & $32\times 32$ &  $[3\times 3, 64]\times 1$, stride 1\\
\hline
Conv & $16\times 16$ &  $[3\times 3,128]  \times 1$, stride 1, Max pooling (2)\\
\hline
Residual Block& $16\times 16$ & $[3\times 3,128]$, stride 1\\
\hline
Conv & $8\times 8$ &   $[3\times 3, 216]\times 1$, stride 1, Max pooling (2)\\
\hline
Conv & $4\times 4$ &   $[3\times 3, 512]\times 1$, stride 1, Max pooling (2)\\
\hline
Residual Block& $4\times 4$ & $[3\times 3,512]$, stride 1\\
\hline
Linear & Number of classes &  Max pooling (4), fully connected\\
 \hline
 \hline
\end{tabular}
\end{table*}


We adopt the training configure from \cite{david2018resnet}, which uses the label smooth loss function~\citep{szegedy2016rethinking}. Specifically, for a $K$ classification problem, given a training sample $x$ with class $y$, we denote its predicted probability for class $i$ as $p_i(x)$, and then the loss function is 
\begin{align*}
	f(x,y;\epsilon) = (1-\epsilon)\sum_{i=1}^K\delta_{i}(y) \log p_i(x) +  \frac{\epsilon}{K}\sum_{i=1}^K\log p_i(x),
\end{align*}
where $\epsilon$ denotes the smoothing parameter, and $\delta_i(y) = \mathbf{1}_{\{i=y\}}$ is the indicator function. In our experiments, we set $\epsilon$ as $0.2$. In addition, for each experiment, we train the network for $100$ epochs and use the batch size as $512$. Moreover, we use the state-of-the-art step size setting with warmup as follows: \begin{equation*}
\eta_i=
    \begin{cases}
      \frac{i}{20}\eta, & 1\leq i\leq 20,\\
      \left(1-\frac{i-20}{80}\right) \eta, & 21\leq i\leq 100,
    \end{cases}       
\end{equation*}
where $\eta_i$ is the step size used in the $i$-th epoch for $1\leq i\leq 100.$ The warmup is effective to obtain a good parameter in training deep neural network. 

 
 For MSGD, we set the momentum parameter $\mu$ as $0.9$, and choose the step size $\eta_{\textrm{M}}$ as $\{0.04 \ell: \ell\in\NN, 4\leq \ell\leq 15\}.$ Thus, for VSGD, we use the equivalent step size of MSGD ($\eta_{\mathrm{V}}= \frac{\eta_{\mathrm{M}}}{1-\mu}$) chosen from $\{0.4 \ell: \ell\in\NN, 4\leq \ell\leq 15\}.$  Figures~\ref{cifar10:best} and~\ref{cifar100:best} show that the comparisons of loss values between the MSGD and the VSGD with their best settings over CIFAR-10. As can be seen, the validate loss of MSGD decreases faster than that of the VSGD and eventually achieves a smaller value. For more comparison results, please see Appendix~\ref{NNsetting}.  In addition, Table~\ref{tab:result} presents the validate accuracy of both MSGD and VSGD over CIFAR datasets. As can be seen, in average, the MSGD is better than the VSGD with the equivalent step size over CIFAR-10 and CIFAR-100 tasks. We further test the significance of the pairwise comparison between the best MSGD and the best VSGD. For CIFAR-10  ($\eta_{\mathrm{M}}=0.36$ and $\eta_{\mathrm{V}}=2$) and CIFAR-100 ($\eta_{\mathrm{M}}=0.56$ and $\eta_{\mathrm{V}}=2.4$), the corresponding $p$-values are $\mathbf{ 0.0108}$ and $\mathbf{1.023\times 10^{-5}}$, respectively. This shows  that the best MSGD {\bf significantly outperforms} the best VSGD.
  	\begin{figure}[t]
		\centering
		\label{cifar10:best}
			\subfigure[Best setting: $\eta_{\textrm{V}} = 2$ and $\eta_{\textrm{M}}=0.36$]{
		\includegraphics[width=0.45\textwidth]{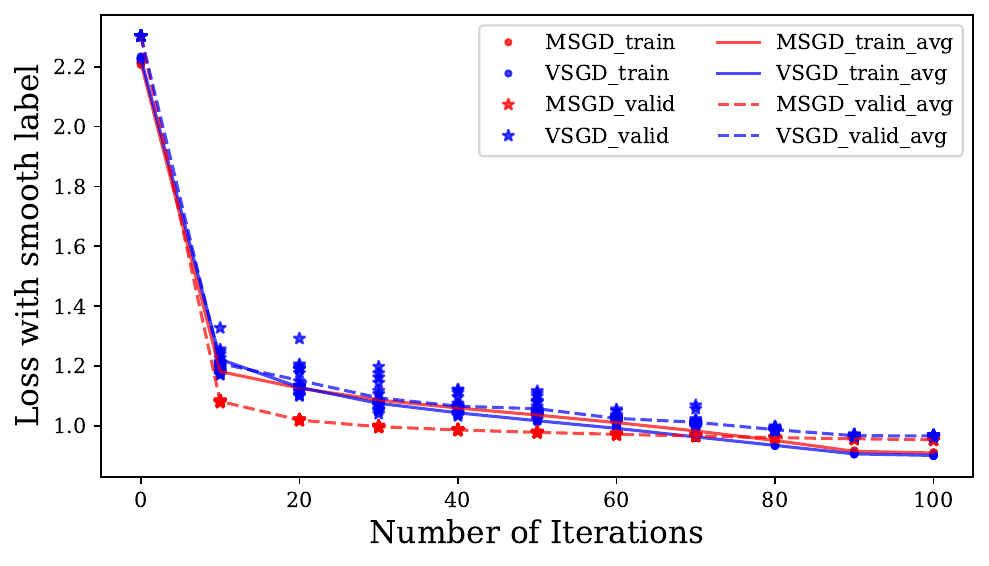}\label{fig:10_3.6}
	}
			\subfigure[Zoom-in for the best setting]{
		\includegraphics[width=0.46\textwidth]{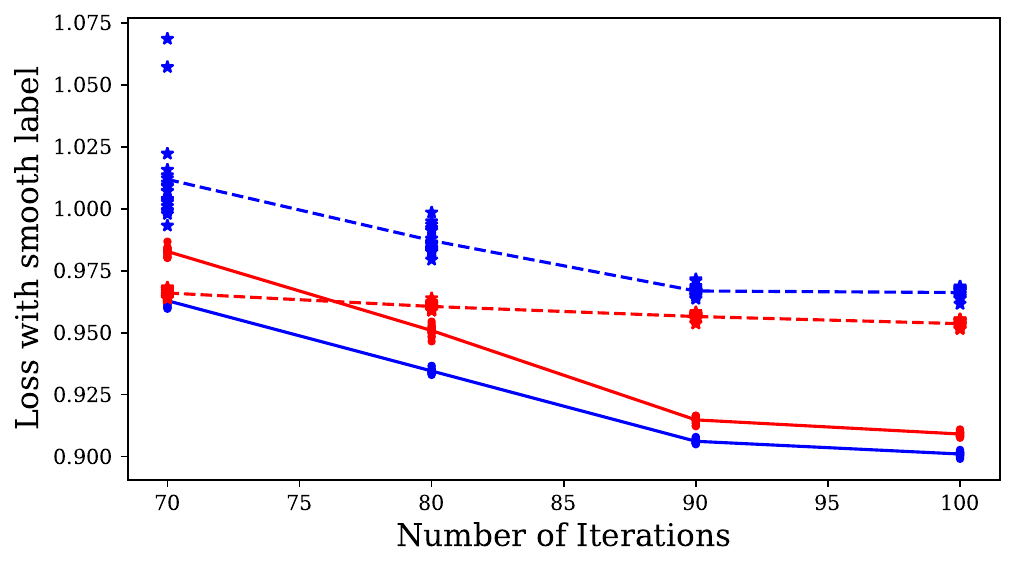}\label{fig:10_3.6}
	}
	\caption{Experimental Results of ResNet-9 on CIFAR-10 under the Best Settings:  $\boldsymbol{\eta_{\textrm{V}} = 2}$, $\boldsymbol{\eta_{\textrm{M}}=0.36}$.}
	\end{figure}

		\begin{figure}[t]
		\centering
		\label{cifar100:best}

		\subfigure[Best setting: $\eta_{\textrm{V}}=2.4$ and $\eta_{\textrm{M}}=0.56$]{
		\includegraphics[width=0.45\textwidth]{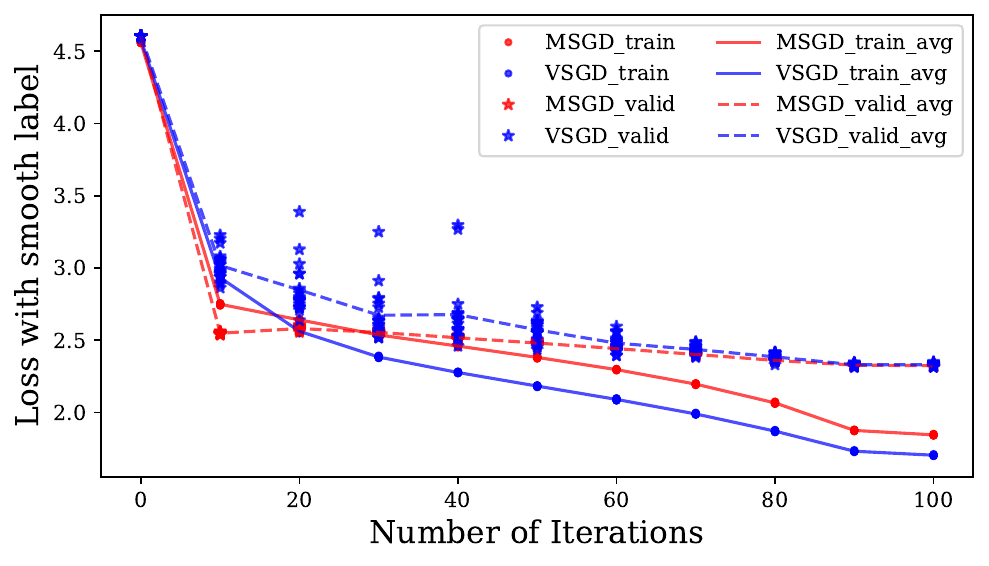}\label{fig:100_3.6}
	}
		\subfigure[Zoom-in for the best setting]{
		\includegraphics[width=0.45\textwidth]{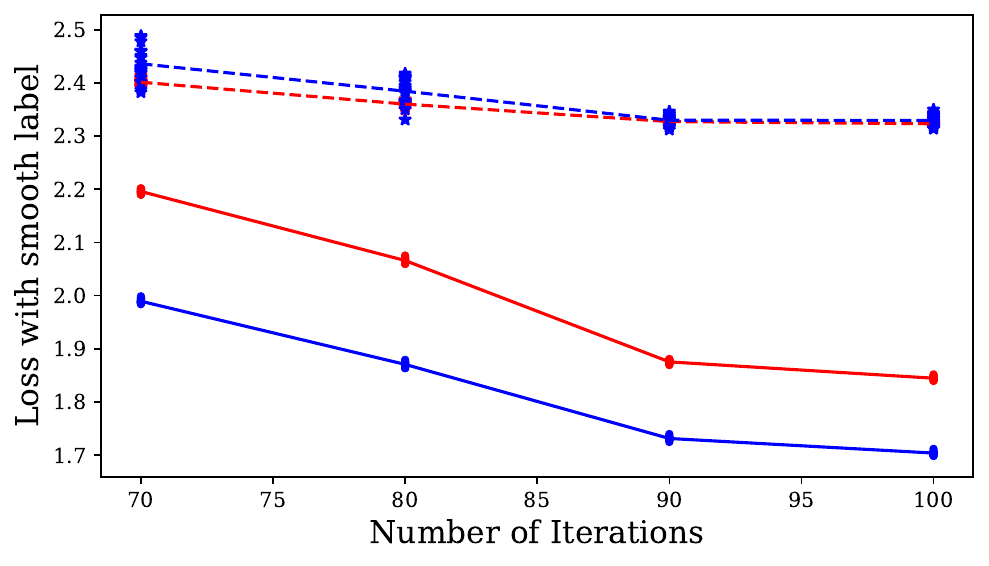}\label{fig:100_3.6}
	}
		\caption{Experimental Results of ResNet-9 on CIFAR-100 under the Best Ssettings: $\boldsymbol{\eta_{\textrm{V}}=2.4}$, $\boldsymbol{\eta_{\textrm{M}}=0.56}$.}
					\end{figure}

\begin{table*}[htb!]
	\centering
	\caption{Results of Validation Accuracy and the Corresponding Standard Deviations (in the Bracket) for the Last Epoch under the ResNet-9 over CIFAR-10 and CIFAR-100.  }\label{tab:result}
	\small
	\begin{tabular}{c|c c c c  c c c c c c c c c}
	\hline
	\hline
				$\frac{\eta}{1-\mu}$   & $1.6$ & $2$ & $2.4$& $2.8$ & $3.2$& $3.6$ & $4$ & $4.4$&$4.8$ &$5.2$ &$5.6$& $6$\\
		
		\hline
		\multicolumn{12}{c}{CIFAR-10} \\
				\hline

			    VSGD  & $95.31$&
			    $\mathbf{95.32}$&
			    $95.19$&
			    $95.23$&
			    $95.07$&
			    $95.06$&
			    $94.91$&
			    $94.80$&
			    $94.70$&
			    $94.45$&
			    $94.38$&
			    $94.06$
			    \\
	              & $(0.14)$&
	              $(0.14)$&
	              $(0.23)$&
	              $(0.19)$&
	              $(0.22)$&
	              $(0.20)$&
	              $(0.25)$&
	              $(0.34)$&
	              $(0.20)$&
	              $(0.30)$&
	              $(0.26)$&
	              $(0.51)$
	              \\
	    MSGD  & $95.65$&
		$95.71$&
	    $95.78$&
	    $95.81$&
	    $95.83$&
	    $\mathbf{95.87}$&
	    $95.82$&
	    $95.84$&
	    $95.80$&
	    $95.78$&
	    $95.77$&
	    $95.75$
	     \\
	     & $(0.13)$&
	     $(0.13)$&
	     $(0.11)$&
	     $(0.11)$&
	     $(0.11)$&
	     $(0.14)$&
	     $(0.08)$&
	     $(0.11)$&
	     $(0.14)$&
	     $(0.15)$&
	     $(0.12)$&
	     $(0.10)$\\ 
		\hline
			\multicolumn{12}{c}{CIFAR-100} \\
					\hline
		VSGD  & $75.44$&
		$75.46$&
		$\mathbf{75.49}$&
		$75.21$&
		$75.10$&
		$74.81$&
		$74.73$&
		$74.45$&
		$74.18$&
		$73.83$&
		$73.47$&
		$73.12$
		\\
		& $(0.39)$&
		$(0.42)$&
		$(0.30)$&
		$(0.35)$&
		$(0.40)$&
		$(0.67)$&
		$(0.50)$&
		$(0.52)$&
		$(0.45)$&
		$(0.61)$&
		$(0.73)$&
		$(0.80)$
		\\
		MSGD  & $76.95$&
		$77.09$&
		$77.38$&
		$77.53$&
		$77.76$&
		$77.80$&
		$78.02$&
		$78.04$&
		$78.01$&
		$78.15$&
		$\mathbf{78.17}$&
		$78.16$
		\\
		& $(0.21)$&
		$(0.26)$&
		$(0.25)$&
		$(0.25)$&
		$(0.25)$&
		$(0.23)$&
		$(0.17)$&
		$(0.25)$&
		$(0.26)$&
		$(0.28)$&
		$(0.32)$&
		$(0.24)$  
		\\
		\hline
		\hline
	\end{tabular}
	\end{table*}

\section{Discussions}\label{discussion}
\noindent $\bullet$ {\bf Related Literature.}
In the existing literature, we are only aware of \cite{ghadimi2016accelerated} and \cite{jin2017accelerated}  considering stochastic nonconvex optimization using momentum.
\begin{table*}[t]
	\centering
	\caption{Comparison  with Relevant Literature. Notation List: FOOS: First Order Optimal Solution; SOOS: Second Order Optimal Solution; SA: Stochastic Approximation; SEA: Saddle Escaping Analysis; A/N: Asymptotic/Nonasymptotic; LCG: Lipschitz Continuous Gradient; LH: Lipschitz Continuous Hessian.}\label{comparison}
	\small
	\begin{tabular}{|c|c|c|c|c|c|c|}
		\hline
		 &FOOS &SOOS &SA & SEA & Assumptions &A/N\\
		\hline
		Ours& $\surd$ &  $\surd$&  $\surd$ &  $\surd$& Strict Saddle, Isolated Optima & A\\
		\hline
	 \cite{ghadimi2016accelerated} & $\surd$ &  $\times$&  $\surd$ & $\times$& LCG/LH/Unconstrained & N\\
		\hline
	\cite{jin2017accelerated} & $\surd$ &  $\surd$&  $\times$ &  $\surd$& LCG/LH/Unconstrained  & N\\
		\hline

	\end{tabular}
\end{table*}
\cite{ghadimi2016accelerated} only consider convergence to the first order optimal solution, and therefore cannot justify the advantage of the momentum in escaping from saddle points; \cite{jin2017accelerated} only consider a batch algorithm, which cannot explain why the momentum hurts when MSGD converges to optima. Moreover, \cite{jin2017accelerated} need an additional negative curvature exploitation procedure, which is not used in popular Nesterov's accelerated gradient algorithms. We summarize the comparison between our results and related works in Table \ref{comparison}.

Our analysis technique is closely related to several recent works using stochastic differential equations to study stochastic gradient-based methods. \cite{li2017stochastic} adopt a numerical SDE approach to derive the so-called Stochastic Modified Equations for VSGD. However, their analysis requires the drift term in the SDE to be bounded, which is not satisfied by MSGD. Other results consider SDE approximations of several accelerated SGD algorithms for convex smooth problems only \citep{wang2017asymptotic,krichene2017acceleration}. In contrast, our analysis is for nonconvex problems, which are more general and more technically challenging. 

In a broader sense,  our work is also related to \cite{matthews2018gaussian, rotskoff2018neural, mei2018mean, mei2019mean, sirignano2018mean,sirignano2019mean} which use weak convergence to prove the asymptotic approximation of extreme large neural networks.  However, they consider the size of the networks goes to infinity, while we consider the case that step size goes to 0.

\noindent $\bullet$ {\bf Connected Local Optima:} We want to remark that our analysis can be extended to handle connected global optima. As we have mentioned, the major difficulty is the unboundedness of the normalized error $(x_t-x^*)/\sqrt{\eta}$. This can be overcome by choosing a suitable metric to characterize the distance between the iterate and global optima. Take rank-r PCA as an example, where the rotation of any global optimum is also global optimal and thus all the global optima are connected.  In this case, we can use the principal angle between column spans of a given global optimum and the iterate \citep{chen2018dimensionality} to characterize the error.  Since the principal angle is rotational invariant,  the normalized error will be a unique quantity and will not blow up even when the iterate is wandering among different optima. 
Moreover, we can also utilize special landscape structure, such as partial dissipativity \citep{zhou2019toward}, around the connected local optima to facilitate our analysis. However, the analysis will be more involved and is out of the scope of our paper. 

\begin{figure}[t]
	\centering
	 \label{sharp-flat}
	\includegraphics[width=0.6\textwidth]{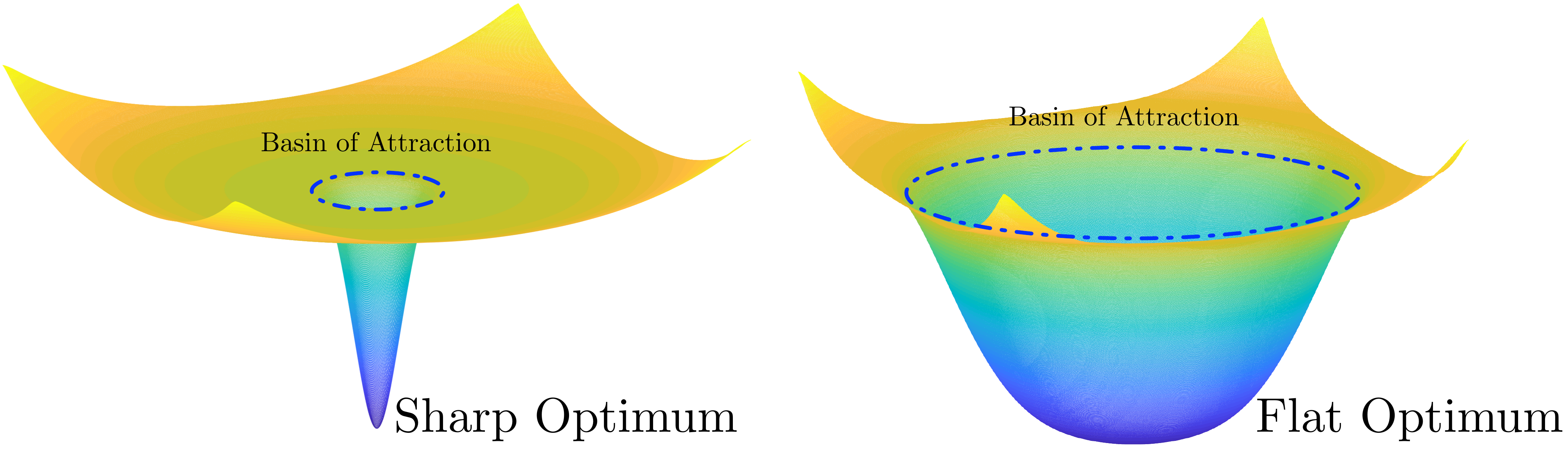}
	\caption{Two Illustrative Examples of the Flat and Sharp Local Optima. MSGD Tends to Avoid the Sharp Local Optimum, since Its High Variance Encourages Exploration.}
\end{figure}
\noindent $\bullet$ {\bf Connection to DNNs:} The results on training DNNs are expectable or partially expectable, given our theoretical analysis for streaming PCA. Our results show that with a good network architecture, the momentum indeed improves the training.

Our  analysis implies that when $\eta$ is sufficiently small, MSGD with step size $\eta$ and momentum $\mu$ performs similarly to VSGD with step size $\frac{\eta}{1-\mu}.$  In practice, however,  people actually use a relative large step size during training and we can still observe the advantage of MSGD over VSGD with the same equivalent step size. As we can observe in Table 2, MSGD always performs better than VSGD. Moreover, MSGD achieves the optimal generalization using $\frac{\eta_{\rm M}}{1-\mu}=5.6$,  but VSGD performs the best using  a smaller equivalent step size $\eta_{\rm V}=2.4<5.6$ under the ResNet over CIFAR-100. This implies MSGD  can afford larger equivalent step size than VSGD.  These phenomena cannot be fully explained by our theory.

\noindent $\bullet$ {\bf Flat/Sharp Local Optima:} \cite{keskar2016large,zhang2017theory,neyshabur2017exploring}  suggest that the landscape of these spurious/bad local optima is usually sharp, i.e., their basin of attractions are small and wiggle. From this aspect,  using a larger equivalent step size can help MSGD  escape from  spurious/bad local optima and stay in ``flat/good local optima", since the higher variance of the noise introduced by the momentum encourages more exploration outside the small basin of attraction of sharp local optima.

\noindent $\bullet$ {\bf Extension:} Our theoretical analysis can be applied to study other problems related to momentum. For example,  \cite{liu2018towards} use the main technique of this paper to study an asynchronous MSGD with the focus on the trade off between momentum and asynchrony. For another example, by analyzing the SDE around different local optima, we can theoretically characterize how momentum helps select flat optima. 

\bibliography{MSGD_ICML}
\bibliographystyle{ims}
\appendix

\section{Summary on Weak Convergence and Main Theorems}\label{summary}
Here, we summarize the theory of weak convergence and theorems used in this paper.
Recall that  the continuous-time interpolation of the solution trajectory $V^{\eta}(\cdot)$ is defined as  $V^\eta(t)=v^{\eta}_k$ on the time interval $[k\eta,k\eta+\eta).$ It has sample paths in the space of  C\`adl\`ag functions ( right continuous and have left-hand limits) defined on $\RR^d$, or \emph{Skorokhod Space}, denoted by $D^d[0,\infty) $. Thus, the weak convergence we consider here is defined in this space $D^d[0,\infty)$ instead of $\RR^d$. The special metric $\sigma$ in $D^d[0,\infty)$ is called  Skorokhod metric, and the topology generated by this metric is Skorokhod topology.  Please refer to \cite{sagitov2013weak, kushner2003stochastic} for detailed explanations.  The weak convergence in $D^d$ is defined as follows:
\begin{definition}[Weak Convergence in $D^d[0,\infty) $]
	Let $\mathcal{B}$  be the minimal $\sigma$-field  induced by Skorokhod topology. Let $\{X_n,\,n<\infty\}$ and $X$ be random variables on $D^d[0,\infty) $  defined on a probability space $(\Omega,P,\mathcal{F}).$ Suppose that $P_n$ and $P_X$ are the probability measures on $(D^d,\mathcal{B})$ generated by $X_n$ and X. We say $P_n$ converges weakly to $P$ ($P_n\Rightarrow P$), if for all  bounded and continuous real-valued functions $F$ on $D^d$, the following condition holds:
	\begin{equation}
	\EE F(X_n)=\int F(x)dP_n(x)\rightarrow \EE F(X)
=\int F(x)dP(x)	
\end{equation}    
	With an abuse of terminology, we say $X_n$ converges weakly to $X$ and write $X_n\Rightarrow X.$ 
\end{definition}
Another important definition we need is \emph{tightness}:
\begin{definition}
	A set of $D^d$-valued random variables $\{X_n\}$ is said to be tight if for each $\delta>0$, there is a compact set $ B_\delta \in D^d$ such that:
	\begin{equation}\label{def_tight}
	\sup_n P\{X_n\notin B_\delta\}\leq \delta.
	\end{equation}
\end{definition}
We care about tightness because it provides us a powerful way to prove weak convergence based on the following two theorems:	
\begin{theorem}[Prokhorov's Theorem]\label{Thm_Prohorov}
	Under Skorokhod topology,  $\{X_n(\cdot)\}$ is tight in  $D^d[0,\infty)$ if and only if it is relative compact which means each subsequence contains a further subsequence  that converges weakly.	
	\end{theorem}
\begin{theorem}[\cite{sagitov2013weak}, Theorem 3.8]\label{Thm_Sag}
	A necessary and sufficient condition for $P_n\Rightarrow P$ is each subsequence
	$P_{n'}$ contains a further subsequence $P_{n''}$ converging weakly to $P.$
\end{theorem}
Thus, if we can prove  $\{X_n(\cdot)\}$ is tight and all the further subsequences share the same weak limit $X$, then we have $X_n$ converges weakly to $X$. 
However, \eqref{def_tight} is hard to verified. We usually check another easier criteria. We first define the c\`adl\`ag modulus to characterize the discontinuity of any  $f\in D^d[0,\infty].$

\begin{definition}[\cite{nowakowski2013multi}, Definition 2.7]\label{modulus}
	For $f\in D^d[0,\infty],$ $T>0$ and $\epsilon>0,$ the modulus of continuity is defined by
	$$
	\varpi'_T (f,\epsilon) := \inf_{\Pi_{T,\epsilon}} \max_{1 \leq i \leq k} w (f,[t_{i - 1}, t_{i})),
	$$
	where  $\Pi_{T,\epsilon}=\{0=t_0\leq t_1\leq \cdots\leq t_k=T,\min_{1 \leq i \leq k} t_{i}-t_{i-1}>\epsilon\}$ and
	$$
	w (f,[t_{i - 1}, t_{i})):= \sup_{s, t\in [t_{i - 1}, t_{i}) } | f(s) - f(t) |.
	$$
\end{definition}

Next theorem provides an sufficient and necessary condition for the tightness of sequence $X_n$ in $D^d[0,\infty) $.
\begin{theorem}[\cite{nowakowski2013multi}, Theorem 2.4]\label{Thm_Tight0}
	Let $\{X_n(\cdot)\}$ be a sequence of processes that have paths in  $D^d[0,\infty) $. Then  $\{X_n(\cdot)\}$ is tight if and only if 
	\begin{itemize}
		\item[(i).] For every $T>0,$ $\delta>0,$ there exists $n_0>0$ and $C>0$ such that
		$$\PP\left(\sup_{t\in[0,T]}X_n(t)>C\right)\leq\delta,~\forall n\geq n_0. $$
		\item[(ii).] For every $T>0,$ $\delta>0,$ $\gamma>0,$ there exists  $n_0>0$ and $\epsilon$ such that
			$$\PP\left(	\varpi'_T (X_n,\epsilon)\geq\gamma\right)\leq\delta,~\forall n\geq n_0.$$
	\end{itemize}
\end{theorem}

Theorem \ref{Thm_Tight} provides one sufficient condition for tightness. Let $\mathcal{F}_t^n$ be the $\sigma$-algebra generated by $\{X_n(s),s\leq t\}$, and $\tau$ denotes a $\mathcal{F}_t^n$-stopping time.
\begin{theorem}[\cite{kushner2003stochastic}, Theorem 3.3, Chapter 7]\label{Thm_Tight}
	Let $\{X_n(\cdot)\}$ be a sequence of processes that have paths in  $D^d[0,\infty) $. Suppose that for each $\delta>0$ and each $t$ in a dense set in $[0,\infty)$, there is a compact set $K_{\delta,t}$ in $\RR$ such that 
	\begin{equation}
	\inf_nP\{X_n(t)\in K_{\delta,t}\}\geq 1-\delta,
	\end{equation}
	and for each positive $T$,
	\begin{equation}
	\lim_\delta \limsup_n \sup_{|\tau|\leq T} \sup_{s\leq\delta} \EE \min[\|X_n(\tau+s)-X_n(\tau)\|_2,1]=0.
	\end{equation}
Then  $\{X_n(\cdot)\}$ is tight in  $D^d[0,\infty). $
\end{theorem}
This theorem is used in Section~\ref{Section_ODE} to prove tightness of the trajectory of Momentum SGD.

At last, we provide the theorem we use to prove the SDE approximation. Let's consider the following algorithm:
\begin{equation}
\theta_{n+1}^\eta=\theta_n^\eta+\eta Y_n^\eta,
\end{equation}
where $Y_n^\eta=g_n^\eta(\theta_{n}^\eta,\xi_n^\eta)+ M_n^\eta$, and $M_n^\eta$ is a martingale difference sequence.
Then the normalized process $U_n^\eta=(\theta_n^\eta-\bar{\theta})/\sqrt{\eta}$ satisfies:
\begin{equation}
U_{n+1}^\eta=U_n^\eta+\sqrt{\eta}(g_n^\eta(\theta_{n}^\eta,\xi_n^\eta)+ M_n^\eta).
\end{equation}
We further assume the fixed-state-chain exists  and use the same notation $\xi_i(\theta)$ to denote the fixed-$\theta$-process. Then we have the following theorem:
\begin{theorem}[\cite{kushner2003stochastic}, Theorem 8.1, Chapter 10]\label{thm8_1}
Assume the following conditions hold:
\begin{enumerate}
\item[C.1] For small $\rho>0$, $\{|Y_n^\eta|^2I_{|\theta_n^\eta-\bar{\theta}|\leq \rho}\}$ is uniformly integrable.\label{a1}
\item[C.2] There is a continuous function $\bar{g}(\cdot)$ such that for any sequence of integers $n_\eta\rightarrow 0$ satisfying $n_\eta \eta\rightarrow 0$ as $\eta\rightarrow 0$ and each compact set $A$,
$$\frac{1}{n_\eta}\sum_{i=jn_\eta}^{jn_\eta+n_\eta-1} E_{jn_\eta}^\eta[g_i^\eta(\theta,\xi_i(\theta))-\bar{g}(\theta)]I_{\{\xi_{jn_\eta}^\eta\}}\rightarrow 0$$ in the mean for each $\theta$, as $j\rightarrow \infty$ and $\eta\rightarrow 0.$\label{a2}
\item[C.3]Define $$\Gamma_{n}^{\eta}(\theta)=\sum_{i=n}^\infty(1-\eta)^{i-n}E_n^\eta[g_i^\eta(\theta,\xi_{i}(\theta))-\bar{g}(\theta)],$$ where when $E_n^\eta$ is used, the initial condition is $\xi_n(\theta)=\xi_n^\eta.$ For the initial conditions $\xi_n^\eta$ confined to any compact set, $$\{|\Gamma_{n}^{\eta}(\theta_n^\eta)|^2I_{|\theta_n^\eta-\bar{\theta}|\leq \rho},|\Gamma_{n}^{\eta}(\bar{\theta})|^2;n,\eta\}$$ 
is uniformly integrable, and 
$$E\left|E_n^\eta \Gamma_{n+1}^{\eta}(\theta_{n+1}^\eta)-\Gamma_{n+1}^{\eta}(\theta_n^\eta) \right|^2I_{|\theta_n^\eta-\bar{\theta}|\leq \rho}=O(\eta^2).$$ 
\item[C.4] There is a Hurwitz matrix $A$ such that
$$\bar{g}(\theta)=A(\theta-\bar{\theta})+o(\theta-\bar{\theta}).$$
\item[C.5] There is a matrix $\Sigma_0=\{\sigma_{0,ij};\,i, j=i,...,r\}$ such that as $n, m\rightarrow \infty,$
$$\frac{1}{m}\sum_{i=n}^{n+m-1} E_n^\eta[M_i^\eta(M_i^\eta)'-\Sigma_0] I_{|\theta_n^\eta-\bar{\theta}|\leq \rho}\rightarrow 0$$
in probability.\label{a5}
\end{enumerate}
Then  $\{U^\eta(\cdot)\}$ is tight. Given tightness,  we further assumes the following assumptions hold.
\begin{enumerate}
\item[C.6] There is a matrix $\bar{\Sigma}_0=\{\bar{\sigma}_{0,ij};\,i, j=i,...,r\}$ such that as $n, m\rightarrow \infty,$
$$\frac{1}{m}\sum_{i=n}^{n+m-1} E_n^\eta[g_i^\eta(\bar{\theta},\xi_i(\bar{\theta}))(g_i^\eta(\bar{\theta},\xi_i(\bar{\theta})))'-\bar{\Sigma}_0] \rightarrow 0$$
in probability.\label{a6}
\item[C.7]Define another function 
	$$G_{n}^{\eta,i}(\theta,\xi_n^\eta)=E_n^\eta\left[\Gamma_{n+1}^{\eta}(\theta_n^\eta)[Y_{n}^\eta]' I_{|\theta_n^\eta-\bar{\theta}|\leq \rho}\big|\theta_n^\eta=\theta \right].$$
	It needs to be a continuous function in $(\theta,\xi_n^\eta)$, uniformly in $n$ and $\eta$.  \label{a7}
\item[C.8]There is a matrix $\Sigma_1=\{\sigma_{1,ij};\,i, j=i,...,r\}$ such that as $n, m\rightarrow \infty,$
$$\frac{1}{m}\sum_{i=n}^{n+m-1} E_n^\eta[G_{n}^{\eta,i}(\bar{\theta},\xi_i(\bar{\theta}))-\Sigma_1] \rightarrow 0$$
in probability.\label{a8}
\end{enumerate}
Then there exists a Wiener process $W(\cdot)$ with covariance matrix $\Sigma=\Sigma_0+\bar{\Sigma}_0+\Sigma_1+\Sigma_1'$ such that $\{U^\eta(\cdot)\}$ converges weakly to a stationary solution of 
$$d U=AU dt+d W.$$

\end{theorem}

\section{Proof of Theorem~\ref{Thm1}}\label{proofODE}
The proof consists of two parts. In the first part, we show that $\{X^\eta(\cdot)\}$ is tight.  Therefore, every sub-sequence has further one sub-sequence that weakly converges to some limit process. In the second part, we find the limit ODE and show that the solution to this ODE exists and is unique. Combining these two parts, we prove the result.

\noindent $\bullet$ {\bf Tightness.} We first rewrite MSGD as follows:
$$x_{k+1}^\eta=x_{1}^\eta-\eta\sum_{j=1}^k\sum_{i=1}^j\mu^{j-i} f(x_i^\eta,\xi_i^\eta).$$  Under Assumption \ref{assumption_general}, we have 
$$\norm{x_k^\eta}_2\leq \norm{x_1^\eta}_2+\eta\sum_{j=1}^k\sum_{i=1}^j\mu^{j-i}C\leq\norm{x_1^\eta}_2+\frac{Ck\eta}{1-\mu}.$$
Then the continuous interpolation $X^\eta(t)$ satisfies: 
$$\norm{X^\eta(t)}_2\leq \norm{x_1^\eta}_2+\frac{t}{\eta}\frac{C\eta}{1-\mu}= \norm{x_1^\eta}_2+\frac{Ct}{1-\mu}.$$
We define $K_{\delta,t}=\left\{x\Big|\norm{x}_2\leq \norm{x_1^\eta}_2+\frac{Ct}{1-\mu} \right\}$. Then for any $\delta>0, t>0$ we have 
\begin{equation*}
	\inf_\eta \PP\{X^\eta(t)\in K_{\delta,t}\}=1\geq 1-\delta.
	\end{equation*}
Moreover, $\forall \tau, s>0,$ we have
$$\norm{X^\eta(\tau+s)-X^\eta(\tau)}_2\leq \frac{Cs}{1-\mu}.$$
Therefore, for each positive $T$,
	\begin{equation*}
	\lim_\delta \limsup_\eta \sup_{|\tau|\leq T} \sup_{s\leq\delta} \EE \min[\|X^\eta(\tau+s)-X^\eta(\tau)\|_2,1]=0.
	\end{equation*}
Then by Theorem \ref{Thm_Tight},  $\{X^\eta(\cdot)\}$ is tight.

\noindent $\bullet$ {\bf Limit Process.} For simplicity, we define
\begin{align*}
\textstyle\beta_k^\eta=\sum_{i=0}^{k-1}\mu^{k-i}(\nabla f(x_i^\eta,\xi_i)-\nabla\cF(x_i^\eta))~~\text{and}~~\epsilon_k=\nabla f(x_k^\eta,\xi_k)-\nabla\cF(x_k^\eta).
\end{align*}

We then rewrite the algorithm as follows:
\begin{align*}
m_{k+1}^\eta=m_k^\eta+(1-\mu)\left[-m_k^\eta+\widetilde{M}(x_k^\eta)\right],~~x_{k+1}^\eta=x_{k}^\eta+\eta (m_{k+1}^\eta+\beta_k^\eta+\epsilon_k^\eta),\end{align*}
where $\widetilde{M}(x_k^\eta) = -\frac{1}{1-\mu}\nabla\cF(x_k^\eta)$ is the rescaled negative gradient and
\begin{align*}
m_{k+1}&=-\sum_{i=0}^{k}\mu^{i}\nabla \cF(x_i^\eta).
\end{align*}
Define the sums
	\begin{align*}
	&\mathcal{E}^{\eta}(t)=\eta\sum_{i=0}^{t/\eta-1}\epsilon_i^\eta,\quad B^{\eta}(t)=\eta\sum_{i=0}^{t/\eta-1}\beta_i^\eta,\\
	&\bar{G}^{\eta}(t)=\eta\sum_{i=0}^{t/\eta-1}\widetilde{M}(x_i^\eta),\quad
	\tilde{G}^{\eta}(t)=\eta\sum_{i=0}^{t/\eta-1}[m_{i+1}^\eta-\widetilde{M}(x_i^\eta)].
	\end{align*}
	Then the continuous-time interpolation of $X^\eta(t)$  can be decomposed as follows.
	$$X^\eta(t)=x^{\eta}_0+\bar{G}^{\eta}(t)+\tilde{G}^{\eta}(t)+B^{\eta}(t)+\mathcal{E}^{\eta}(t).$$
	Define the process $W^\eta(t)$ by 
	$$W^\eta(t)=X^\eta(t)-x^{\eta}_0-\bar{G}^{\eta}(t)=\tilde{G}^{\eta}(t)+B^{\eta}(t)+\mathcal{E}^{\eta}(t).$$
We have already shown that $\{X^\eta(\cdot))\}$ is tight in the first part of the proof. Specifically, there is a subsequence $\eta(k)\rightarrow 0$ and a process $X(\cdot)$ such that $$X^{\eta(k)}(t)\Rightarrow X(t),$$  as $k\rightarrow\infty.$
Under the bounded assumption of  $\nabla f(x,\xi),$ one can show that
$$\norm{x_{k+1}^\eta -x^\eta_k}_2 = \eta\norm{\sum_{i=1}^k\mu^{j-i} f(x^\eta_i,\xi_i^\eta)}_2\leq \frac{\eta}{1-\mu}C,$$
 which further implies the uniform integrability of $\{x_k^\eta\}$. By Lemma 2.1 in \cite{kushner1996stochastic}, we know that any weak sense limit $X(t)$ must have Lipschitz continuous path.   For notational simplicity, we write $\eta(k)$ as $\eta$ in the following proof.

For $t>0$ and  integer $p$, we take $s_i\leq t,$ $i\leq p,$ and $\tau>0$. Let $g(\cdot)$ be a continuous, bounded and real-valued function. Then by definition of $W^{\eta}(t)$, we have
	\begin{align}
	0=E&g(X^\eta(s_i),i\leq p)[W^\eta(t+\tau)-W^\eta(t)]\\
	&-Eg(X^\eta(s_i),i\leq p)[\tilde{G}^\eta(t+\tau)-\tilde{G}^\eta(t)]\label{term1}\\\label{term2}
	&-Eg(X^\eta(s_i),i\leq p)[\mathcal{E}^\eta(t+\tau)-\mathcal{E}^\eta(t)]\\
	&-Eg(X^\eta(s_i),i\leq p)[B^\eta(t+\tau)-B^\eta(t)]. \label{term3}
	\end{align}
	
Let $\mathcal{F}^\eta_n=\sigma\{x^\eta_i,\xi_{i-1}^\eta,i\leq n\},$ then $\mathcal{F}^\eta_{t/\eta}$ measures $\{\mathcal{E}^\eta(s),s\leq t\}$ by definition and the process $\mathcal{E}^\eta(\cdot)$ is actually an $\mathcal{F}^\eta_{t/\eta}$-martingale. By the tower property of the conditional expectation, we know term \eqref{term2} equals to 0.\\
	Next, we eliminate term \eqref{term3}. Note that for any $m,n>0$, we have
	$$\left\|\frac{1}{m}\sum_{i=n}^{n+m-1}\EE[\beta^\eta_i|\mathcal{F}_n]\right\|_2=\left\|\frac{1}{m}\sum_{i=n}^{n+m-1}\mu^{i-n}\beta^\eta_n\right\|_2\leq \frac{1}{(1-\mu)m}\|\beta^\eta_n\|_2.$$
	Since $\beta^\eta_n$ is uniformly bounded in $\eta,m$ and $n$, we have
	$$\lim_{m,n,\eta}\frac{1}{m}\sum_{i=n}^{n+m-1}\EE[\beta^\eta_i|\mathcal{F}_n]=0$$
	in $\mathcal{L}_2$, which also means
	$$\lim_{\eta\rightarrow 0}\EE[B^\eta(t+\tau)-B^\eta(t)|\mathcal{F}^\eta_{t/\eta}]=0.$$
	Together with the boundedness of $f$, by Dominated Convergence Theorem, we know that term \eqref{term3} goes to 0, as $\eta\rightarrow0$.

	For term \eqref{term1}, we first bound $\|\tilde{G}^\eta(t+\tau)-\tilde{G}^\eta(t)\|_2.$ Since $\frac{1}{1-\mu}=\sum_{i=0}^\infty \mu^i$, there exists $N(\eta)=\log_\mu (1-\mu)\eta$ such that $\sum_{i=N(\eta)}^\infty \mu^i<\eta.$ When $k>N(\eta)$, write  $m_k^\eta$ and  $\tilde M (x_k^\eta)$ into summations: 
	\begin{align*}
m^\eta_{k+1}=-\sum_{i=0}^{k}\mu^{i}\nabla \cF(x_i^\eta)=-\sum_{i=0}^{N(\eta)}\mu^{i}\nabla \cF(x_i^\eta)-\sum_{i=N(\eta)+1}^{k}\mu^{i}\nabla \cF(x_i^\eta),
\end{align*}
and
\begin{align*}
\widetilde{M}(x_k^\eta)=-\frac{1}{1-\mu}\nabla \cF(x_k^\eta)=-\sum_{i=0}^{N(\eta)}\mu^{i}\nabla \cF(x_k^\eta)-\sum_{i=N(\eta)+1}^{\infty}\mu^{i}\nabla \cF(x_k^\eta).
\end{align*}
Note that $\|x_{k+1}^\eta-x_k^\eta\|_2\leq \frac{C}{1-\mu}\eta$. Then we have
$$\max_{i=0,1,...,N(\eta)}\|x_{k-i}^\eta-x_k^\eta\|_2\leq  \frac{C}{1-\mu}N(\eta)\eta\rightarrow 0,$$
as $\eta\rightarrow 0$.
By the Lipschitz assumption, for $i=0,1,...,N(\delta),$ we have
$$\|\nabla \cF(x_k^\eta)-\nabla \cF(x_{k-i}^\eta)\|_2\leq L\frac{C}{1-\mu}N(\eta)\eta.$$
Then
$$\left\|\sum_{i=0}^{N(\eta)}\mu^{i}\{\nabla \cF(x_{k-i}^\eta)-\nabla \cF(x_k^\eta)\}\right\|_2\leq\frac{LCN(\eta)\eta}{(1-\mu)^2}.$$
Since $\nabla \cF(x_k^\eta)$ is bounded by $C$, both $\sum_{i=N(\eta)+1}^{k}\mu^{i}\nabla \cF(x_{k-i}^\eta)$ and $\sum_{i=N(\eta)+1}^{\infty}\mu^{i}\nabla \cF(x_k^\eta)$ are bounded by $C\eta.$
Thus, 
$$\|m^\eta_{k+1}-\widetilde{M}(x^\eta_k)\|_2\leq\frac{KCN(\eta)\eta}{1-\mu}+2C\eta=O\left(\eta \log\frac{1}{\eta}\right).$$
For $k<N(\eta)$, following the same approach, we can bound $\|m^\eta_{k+1}-\widetilde{M}(m^\eta_k)\|_2$ by the same bound $O\left(\eta \log\frac{1}{\eta}\right)$.
Therefore, we have the following bound for $\|\tilde{G}^\eta(t+\tau)-\tilde{G}^\eta(t)\|_2.$
	$$\|\tilde{G}^\eta(t+\tau)-\tilde{G}^\eta(t)\|_2\leq\tau O\left(\eta \log\frac{1}{\eta}\right).$$
	Thus, term \eqref{term1} goes to 0 as $\eta\rightarrow0$.
	Then we have
	$$\lim_\eta Eg(X^\eta(s_i),i\leq p)[W^\eta(t+\tau)-W^\eta(t)]=0.$$
	Define 
	$$W(t)=X(t)-X(0)-\int_0^{T}\widetilde{M}(X(s))ds.$$
	Then the weak convergence and the previous analysis together imply that
	$$Eg(X^\eta(s_i),i\leq p)[W(t+\tau)-W(t)]=0.$$
	Here, we need an important result in the martingale theory:
	\begin{theorem}[\cite{kushner2003stochastic}, Theorem 4.1, Chapter 7]\label{thm_mart}
	Let $U(\cdot)$ be a random process with paths in $D^d[0,\infty)$, where $U(t)$ is measurable on the $\sigma$-algebra $\mathcal{F}_t^X$ determined by $\{X(s),s\leq t\}$ for some given process $X(\cdot)$ and let $\EE[U(t)]<\infty$ for each $t$. Suppose that for each real $t\geq0$ and $\tau\geq0$, each integer $p$ and each set of real numbers $s_i\leq t,\, i=1,...,p,$ and each bounded and continuous real-valued function $h(\cdot)$,
	$$Eh(X^\eta(s_i),i\leq p)[U(t+\tau)-U(t)]=0,$$
	then $U(t)$ is a $\mathcal{F}_t^X$-martingale.
	\end{theorem}
By Theorem \ref{thm_mart} , we know that $W(\cdot)$ is a martingale. It has locally Lipschitz continuous  sample paths by the fact $X(\cdot)$ is Lipschitz. Since a Lipschitz continuous martingale must almost surely be a constant, we know $W(t)=W(0)=0$  with probability 1.  In other words,
$X(t)$ is a solution to the following ODE
 \begin{align}
\dot{X}=-\frac{1}{1-\mu}\nabla \cF(x), \quad x(0)=x_0.
\end{align}
Moreover, under Assumption \ref{assumption_general} and by Theorem 12.70.B in \cite{simmons2016differential}, we know that the above initial value problem has only one solution. Therefore, all sub-sequences of $\{X^\eta(\cdot)\}$ weakly converge to the same limit, which implies the weak convergence of the entire sequence. We prove the theorem.

\section{Detailed Proof in Section \ref{Section_SDE}}
\subsection{Proof of Theorem~\ref{Thm_SDE1}}\label{proof_Thm_SDE1}
\begin{proof} The proof  follows  from Theorem 10.8.1 in \cite{kushner2003stochastic} (Theorem \ref{thm8_1}). We need to check the  Assumption C.1 to C.8 (in Appendix \ref{summary})
\begin{itemize}
	\item[1.] The uniform integrability in C.1 directly follows from the uniform boundedness assumption of $\nabla f(x,\xi)$. 
	\item[2.] C.2 can be easily got from the proof of ODE approximation.
	\item[3.] To check condition C.4, we need use our isolated stationary point assumption, i.e, Assumption \ref{ass_isolated}. At the local optimum $x^*,$  the Hessian matrix must be positive definite.  Then C.4  is obviously satisfied with the Hurwitz matrix $-\nabla^2\cF(x^*).$
\end{itemize}  
The main challenge left is to calculate the variance of the Wiener process  and check the other five assumptions. 

For simplicity, $E_k^\eta[\cdot]$ means the conditional expectation for $$\{\zeta_{k+j},j\geq0;\zeta_k(X)=\zeta_k^\eta\}.$$
From Equation (\ref{eq11}), the variance can be decomposed into three parts.
	The first part is from the noise $\gamma_k^{\eta,i}$. Since we have assumed the weak convergence $x_k^\eta\Rightarrow x^*$, we have in distribution,
	\begin{align*}
	\lim_{\eta,k}E_k^\eta(\gamma_{k+j}^{\eta} (\gamma_{k+j}^{\eta})^\top)=\Sigma.
	\end{align*}
	
	Since the limit is a constant, the convergence also holds in probability. Thus, C.5 is satisfied.
	The second part comes from the fixed-state-chain:
	\begin{align*}
	E_k^\eta(g(x^*,\zeta_{k+j}^\eta(x^*))g(x^*,\zeta_{k+j}^\eta(x^*))^\top)&=E_k^\eta(\zeta_{k+j}^\eta(x^*)-\nabla F(x^*))(\zeta_{k+j}^\eta(x^*)-\nabla F(x^*))^\top\\
	&= E_k^\eta \zeta_{k+j}^\eta(x^*) (\zeta_{k+j}^\eta(x^*))^\top\\
	&=\mu^{2j}(\zeta^{\eta}_{k})(\zeta^{\eta}_{k})^\top+\sum_{m=0}^{j-1}\mu^{2(j-m)}E_k^\eta[\nabla f(x^*,\xi_{k+m})\nabla f(x^*,\xi_{k+m})^\top]\\
	&\rightarrow\frac{\mu^2}{1-\mu^2}\Sigma,
	\end{align*}
 in probability, as $k,j\rightarrow 0.$ Thus, C.6 is satisfied.

The last part is from the term $g(\zeta_k^\eta,x_k^\eta)-g(\zeta_k(x_k^\eta),x_k^\eta)$. Define the discounted sequence
	$$\Gamma_{k}^{\eta}(x)=\sum_{j=0}^\infty(1-\eta)^{j}E_k^\eta[g(x,\zeta_{k+j}^{\eta}(x))-\tilde{M}(x)].$$
	Note that 
	\begin{align*}
	E_k^\eta[\zeta_{k+j}^{\eta}(x)]&=E_k^\eta[\mu^j\zeta_{k}^{\eta}-\sum_{m=0}^{j-1}\mu^{j-m}\nabla f(x,\xi_{k+m})]\\
	&=\mu^j\zeta_{k}^{\eta}-\sum_{m=0}^{j-1}\mu^{j-m}\nabla\cF(x).
	\end{align*}
Thus, we have 
	$$E_k^\eta[g(x,\zeta_{k+j}^{\eta}(x))-\tilde{M}(x)]=\mu^j\zeta_{k}^{\eta}+\frac{\mu^{j+1}}{1-\mu}\nabla\cF(x).$$
	Then 
	$$\Gamma_{k}^{\eta}(x)=\sum_{j=0}^\infty(1-\eta)^{j}\left\{\mu^j\zeta_{k}^{\eta}+\frac{\mu^{j+1}}{1-\mu}\nabla\cF(x)\right\}=\frac{1}{1-(1-\eta)\mu}\left(\zeta_{k}^{\eta}-\frac{\mu}{1-\mu}\tilde{M}(x)\right).$$
	Since ${M}$ is locally Lipschitz, and $\|x_{k+1}^\eta-x_k^\eta\|_2=O(\eta)$, the following result holds:
		\begin{align*}
		\|E_k^\eta [\Gamma_{k+1}^{\eta}(x_{k+1}^\eta)-\Gamma_{k+1}^{\eta}(x_k^\eta)]\|_2^2&=\left\|\frac{\mu}{(1-(1-\eta)\mu)(1-\mu)}\left\{E_k^\eta[\tilde{M}(x_{k+1}^\eta)-\tilde{M}(x_k^\eta)]\right\}\right\|_2^2\\
		&=O(\eta^2).
		\end{align*}
Then, Assumption C.3 holds. 
	
	Define another function 
	$$G_{k}^{\eta}(x,\zeta_k^\eta)=E_k^\eta\left[\Gamma_{k+1}^{\eta,i}(x_k^\eta)(Z_{k}^{\eta})^\top|x_k^\eta=x\right].$$
	It is easy to check this is a continuous function in $(x,\zeta_k^\eta)$, uniformly in $k$ and $\eta$ (Assumption C.7). Moreover,
	\begin{align*}
	\Gamma_{k+1}^{\eta}(x_k^\eta)(Z_{k}^{\eta})^\top&=\frac{1}{1-(1-\eta)\mu}\left(\zeta_{k+1}^{\eta}-\frac{\mu}{1-\mu}\tilde{M}(x_{k}^{\eta})\right)\frac{1}{\mu}(\zeta_{k+1}^{\eta})^\top\\
	&=\frac{1}{1-(1-\eta)\mu}\left(\frac{1}{\mu}\zeta_{k+1}^{\eta}(\zeta_{k+1}^{\eta})^\top-\frac{1}{1-\mu}\tilde{M}(x_{k}^{\eta})(\zeta_{k+1}^{\eta})^\top\right).
	\end{align*}
Then we have
	\begin{align*}
	E_k^\eta[\zeta_{k+1}^{\eta}(\zeta_{k+1}^{\eta})^\top|x_k^\eta=x^*]&= E_k^\eta\left[\left(\mu\zeta_{k}^{\eta}-\mu\nabla f(x_k^\eta,\xi_k)\right)\left(\mu\zeta_{k}^{\eta}-\mu\nabla f(x_k^\eta,\xi_k)\right)^\top\Big|x_k^\eta=x^*\right]\\
	&=\mu^2\zeta_{k}^{\eta}(\zeta_{k}^{\eta})^\top+\mu^2\Sigma,
	\end{align*}
	and
	\begin{align*}
	E_k^\eta[\tilde{M}(x_k^\eta)\zeta_{k+1}^{\eta}|x_k^\eta=x^*]=0.
	\end{align*}
	Those imply that
	\begin{align*}
	E_k^\eta G_{k+j}^{\eta}(x^*,\zeta_{k+j}^{\eta}(x^*))&=\frac{\mu}{1-(1-\eta)\mu}(E_k^\eta\zeta_{k+j}^\eta(x^*)(\zeta_{k+j}^\eta(x^*))^\top+\Sigma)\\
	&\rightarrow \frac{1}{1-\mu^2}\frac{\mu}{1-\mu}\Sigma,
		\end{align*}
	in probability. Thus, C.8 is satisfied. 
	We have proved all the assumptions of Theorem \ref{thm8_1} are satisfied. As a result, there exists a Wiener Process $W$, such that any subsequence of $\{U^{\eta,i}\}$ converges weakly  to a stationary solution of 
	\begin{equation*}
	dU=-\frac{1}{1-\mu} \nabla^2 \cF(x^*) Udt+dW,
	\end{equation*}
	where the variance of $W$ is 
	 $[1+\frac{\mu^2}{1-\mu^2}+ 2\frac{1}{1-\mu^2}\frac{\mu}{1-\mu}]\Sigma=\frac{1}{(1-\mu)^2}\Sigma.$
	
	 Lastly, we show that the above SDE has one unique solution given any initial. In fact, one can verify that both the drift term and the diffusion term are Lipschitz continuous. By Theorem 5.2.5 in \cite{karatzas1998brownian}, we know that the solution exists and is unique.
	 
	 Therefore, $\{U^{\eta,i}\}$ converges weakly to the unique stationary solution of 
	\begin{equation*}
	dU=-\frac{1}{1-\mu} \nabla^2 \cF(x^*) Udt+dW,
	\end{equation*}
We finish the proof. 
\end{proof}

\subsection{Proof of  Theorem~\ref{lemma_phase3}}\label{Time-Optimal-proof}

\begin{proof}
Since we restart our record time, we assume here the algorithm is initialized around  one local optimum $x^*$. Thus, we have $\norm{U^{\eta}(0)}_2^2=\eta^{-1} \delta^2<\infty$. Note that $U^{\eta}(t)$ converges to  $U(t)$ in this neighborhood, and the second moment of $U(t)$ is: 
\begin{align*}
\EE \left(\norm{U(t)}_2^2\right)&=\EE\left[\tr(U(t)U(t)^\top)\right] = \tr \left[\EE U(t)U(t)^\top\right]\\
&= \tr\left[\exp\left(-\frac{t}{1-\mu}\nabla^2\cF(x^*)\right)[ U(0)U(0)^\top] \exp\left(-\frac{t}{1-\mu}\nabla^2\cF(x^*)\right) \right]\\
&~~+\tr \left[\int_0^t \exp\left(-\frac{1}{1-\mu}\nabla^2\cF(x^*)s\right) \frac{1}{(1-\mu)^2}\Sigma \exp\left(-\frac{1}{1-\mu}\nabla^2\cF(x^*)s\right) ds\right]\\
&= \tr\left[\exp\left(-\frac{t}{1-\mu}\nabla^2\cF(x^*)\right)[ U(0)U(0)^\top] \exp\left(-\frac{t}{1-\mu}\nabla^2\cF(x^*)\right) \right]\\
& ~~+ \frac{1}{(1-\mu)^2}\int_0^t \tr \left( \exp\left(-\frac{1}{1-\mu}\nabla^2\cF(x^*)s\right) \Sigma \exp\left(-\frac{1}{1-\mu}\nabla^2\cF(x^*)s\right) \right)ds\\
& =  \frac{1}{(1-\mu)^2}\int_0^t \left\|\exp\left(-\frac{1}{1-\mu}\nabla^2\cF(x^*)s\right) \Sigma^{\frac{1}{2}}\right\|_\mathrm{F}^2 ds\\
&~~~+\left\|\exp\left(-\frac{1}{1-\mu}\nabla^2\cF(x^*)t\right)U(0)\right\|_\mathrm{F}^2\\
& = \sum_{i=1}^d \left\{\left\| e_i e_i^\top U(0)\right\|_\mathrm{F}^2  \exp\left(-\frac{2\lambda_i }{1-\mu}t\right)  + \int_0^t \frac{1}{(1-\mu)^2}\left\| e_i e_i^\top\Sigma^{\frac{1}{2}}\right\|_\mathrm{F}^2  \exp\left(-\frac{2\lambda_i }{1-\mu}s\right) ds\right\}\\
& = \sum_{i=1}^d  \left\| e_i e_i^\top U(0)\right\|_\mathrm{F}^2  \exp\left(-\frac{2\lambda_i }{1-\mu}t\right) + \frac{1}{(1-\mu)}\frac{1-\exp(-\frac{2\lambda_i}{1-\mu}t)}{2\lambda_i} \left\| e_i e_i^\top\Sigma^{\frac{1}{2}}\right\|_\mathrm{F}^2 \\
& =  \sum_{i=1}^d (U(0)^\top e_i )^2  \exp\left(-\frac{2\lambda_i }{1-\mu}t\right) + \frac{1-\exp(-\frac{2\lambda_i}{1-\mu}t)}{2(1-\mu)\lambda_i} e_i^\top\Sigma e_i,
\end{align*}


By Markov inequality, we have:
\begin{align*}
\eta^{-1}\epsilon\PP\left(\left\|X^{\eta}(T_3)-x^*\right\|_2^2> \epsilon \right) &  \leq \eta^{-1}\EE \left(\left\|X^{\eta}(T_3)-x^*\right\|_2^2\right) = \EE \left(\norm{U^\eta(T_3)}_2^2 \right)   \notag \\
&\hspace{-0.75in}\rightarrow \sum_{i=1}^d (U(0)^\top e_i )^2 \exp\left(-\frac{2\lambda_i }{1-\mu}T_3\right) + \frac{1-\exp(-\frac{2\lambda_i}{1-\mu}T_3)}{2(1-\mu)\lambda_i} e_i^\top\Sigma e_i,~~\textrm{as}~\eta\rightarrow 0.
\end{align*}
Thus, for a sufficiently small $\eta$, we have
\begin{align*}
\PP\left(\left\|X^{\eta}(T_3)-x^*\right\|_2^2> \epsilon \right)&\leq\frac{2}{\eta^{-1}\epsilon} \sum_{i=1}^d (U(0)^\top e_i )^2 \exp\left(-\frac{2\lambda_i }{1-\mu}T_3\right) + \frac{1-\exp(-\frac{2\lambda_i}{1-\mu}T_3)}{2(1-\mu)\lambda_i} e_i^\top\Sigma e_i\\
& \leq \frac{2}{\eta^{-1}\epsilon} \Big(\eta^{-1}\delta^2\exp\left[-2\frac{\lambda_dT_3}{1-\mu}\right] + \frac{\phi}{2(1-\mu)\lambda_d}\Big(1-\exp\big(-2\frac{\lambda_1 T_3}{1-\mu}\big)\Big)  \Big) \notag\\
& \leq \frac{2}{\eta^{-1}\epsilon} \Big(\eta^{-1}\delta^2\exp\left[-2\frac{\lambda_dT_3}{1-\mu}\right] + \frac{\phi}{2(1-\mu)\lambda_d} \Big),
\end{align*}
where $\phi = \sum_{i=1}^d e_i^\top\Sigma e_i.$
The above inequality actually implies that the desired probability is asymptotically upper bounded by the term on the right hand. Thus, to guarantee $$\PP\left((\left\|X^{\eta}(T_3)-x^*\right\|_2^2> \epsilon  \right) \leq \frac{1}{4}$$ when $\eta$ is sufficiently small, we need 
 $$\frac{2}{\eta^{-1}\epsilon} \Big(\eta^{-1}\delta^2\exp\left[-2\frac{\lambda_dT_3}{1-\mu}\right] + \frac{\phi}{2(1-\mu)\lambda_d} \Big) \leq \frac{1}{4}.$$
The above inequality has a solution only when:
\begin{equation*}
(1-\mu)\lambda_d\epsilon-4\eta\phi>0.
\end{equation*}
Moreover, when the above inequality holds, we have:
\begin{align*}
T_3 = \frac{1-\mu}{2\lambda_d}\log\left(\frac{8  (1-\mu)\lambda_d\delta^2}{(1-\mu)\lambda_d\epsilon-4\eta\phi}\right).
\end{align*} 
We finish the proof. 
 
\end{proof}

\subsection{Proof of Theorem~\ref{Time_Saddle}}\label{Time_Saddle_proof}
\begin{proof}
Recall that Theorem \ref{SDE_Saddle} holds when  $u_{k}^{\eta}=(x_k^\eta-\hat{x})/\sqrt{\eta}$ is bounded. Thus, if $\norm{X^\eta(T_1)}_2^2\geq \delta^2$ holds at some time $T_1$, the algorithm has successfully escaped from the saddle point. We approximate $U^{\eta}(t)$ by the limiting process approximation, which is Gaussian distributed at time $t$. As $\eta \rightarrow 0$, by simple manipulation, we have
	\begin{align*}
		\PP\left(\norm{X^\eta(T_1)}_2^2\geq \delta^2\right)  = \PP\left(\norm{U^{\eta}(T_1)}_2^2 \geq \eta^{-1}\delta^2\right).
	\end{align*} 
	
	We then prove $\PP\left(\norm{U^{\eta}(T_1)}_2^2 \geq \eta^{-1}\delta^2\right)\geq 1-\nu$. At time t, $U^{\eta}(t)$ converges to a Gaussian distribution with mean $0$ and covariance  matrix $$\int_0^{T_1} \exp\left(-\frac{1}{1-\mu}\nabla^2\cF(\hat x)s\right) \frac{1}{(1-\mu)^2}\Sigma \exp\left(-\frac{1}{1-\mu}\nabla^2\cF(\hat x)s\right) ds. $$ Let $\nabla^2\cF(\hat x) = P\Lambda P^\top$ where $\Lambda = \diag(\lambda_1,..., \lambda_d)$ and $\lambda_1\geq \lambda_2\geq...\geq\lambda_d$ and $\lambda_d<0.$ Since $P$ is orthogonal, we have  $\norm{P^\top U^{\eta}(T_1)}_2= \norm{U^{\eta}(T_1)}_2, $ and $P^\top U^{\eta}$ converges to a Gaussian distribution with mean $0$ and covariance  matrix 
	\begin{align*}
		&\int_0^{T_1} P^\top\exp\left(-\frac{1}{1-\mu}\nabla^2\cF(\hat x)s\right) \frac{1}{(1-\mu)^2}\Sigma \exp\left(-\frac{1}{1-\mu}\nabla^2\cF(\hat x)s\right) Pds\\
		= &\int_0^{T_1} P^\top P\exp\left(-\frac{1}{1-\mu}\Lambda s\right) \frac{1}{(1-\mu)^2}P^\top \Sigma P \exp\left(-\frac{1}{1-\mu}\Lambda s\right) P^\top Pds\\
		= &\int_0^{T_1} \exp\left(-\frac{1}{1-\mu}\Lambda s\right) \frac{1}{(1-\mu)^2}P^\top \Sigma P \exp\left(-\frac{1}{1-\mu}\Lambda s\right) ds.
	\end{align*}
	Moreover , $(P^\top U^{\eta})^{(d)}$ converge to normal distribution with mean $0$ and variance  
	$$\int_0^{T_1} \exp\left(-\frac{2\lambda_d}{1-\mu} s\right) \frac{1}{(1-\mu)^2} (P^\top \Sigma P)_{d,d} ds = \frac{(P^\top \Sigma P)_{d,d}}{2\lambda_d(1-\mu)}\left(1-\exp\left(-\frac{2\lambda_d}{1-\mu} s\right)\right).$$
	Therefore, let $\Phi(x)$ be the CDF of $N(0,1)$, we have
	\begin{align*}
		\PP\left(\frac{\big |(P^\top U^{\eta}(T_1))^{(d)}\big |}{\sqrt{ \frac{(P^\top \Sigma P)_{d,d}}{2\lambda_d(1-\mu)}\left(1-\exp\left(-\frac{2\lambda_d}{1-\mu} s\right)\right)}}\geq \Phi^{-1}\left(\frac{1+\nu/2}{2}\right)\right) \rightarrow 1-\nu/2,~~\textrm{as}~\eta\rightarrow 0.
	\end{align*}
	When the following inequality holds,
	$$\eta^{-\frac{1}{2}}\delta\leq \Phi^{-1}\left(\frac{1+\nu/2}{2}\right)\cdot\sqrt{ \frac{(P^\top \Sigma P)_{d,d}}{2\lambda_d(1-\mu)}\left(1-\exp\left(-\frac{2\lambda_d}{1-\mu} s\right)\right)},$$ we get
	\begin{align*}
		T_1 = \frac{(1-\mu)}{2|\lambda_d|}\log\left(\frac{2\eta^{-1}\delta^2(1-\mu)|\lambda_d|}{\Phi^{-1}\left(\frac{1+\nu/2}{2}\right)^2(P^\top \Sigma P)_{d,d}} +1\right).
	\end{align*}
Thus, for a sufficiently small $\epsilon$, we have \begin{align*}
\PP\left(\norm{U^{\eta}(T_1)}_2^2 \geq \eta^{-1}\delta^2\right) &= \PP\left(\norm{P^\top U^{\eta}(T_1)}_2^2 \geq \eta^{-1}\delta^2\right)\\
&\geq  \PP\left(\left|\left(P^\top U^{\eta}(T_1)\right)^{(d)}\right| \geq \eta^{-1/2}\delta\right)\\
&\geq  1-\nu.
 \end{align*}
 Take $\nu=\frac{1}{4},$ and we prove the theorem.

\end{proof}

\section{Detailed Proof in Section \ref{section_pca}}\label{Appendix_PCA}
\subsection{Derivation of Momentum Stochastic Generalized Hebbian Algorithm}
SGHA is essentially a primal-dual algorithm. Specifically, we consider the Lagrangian function of \eqref{PCA}:
$$L(v,\lambda)=v^{\top}\EE_{X\sim\cD}[XX^{\top}]v - \lambda(v^\top v -1),$$ where $\lambda$ is the Lagrangian multiplier. We then check the optimal KKT conditions:
$$\EE_{X\sim\cD}[XX^{\top}]v-\lambda v=0~\text{and}~ v^\top v=1,$$
which implies $\lambda = v^{\top}\EE_{X\sim\cD}[XX^{\top}]v.$  At the k-th iteration, SGHA  takes the following primal-dual update:
\begin{itemize}
	\item Dual Update: $\lambda_k=v_k^{\top}\Sigma_kv_k,$
	\item Primal Update: $v_{k+1}=v_k+\eta(\Sigma_k v_k- \lambda_k v_k),$
\end{itemize}
where   $\Sigma_k=X_kX_k^{\top}$ and $\mu(v_k-v_{k-1})$ is the momentum with a parameter $\mu \in [0,1)$.  Combine the primal and dual updates together, we obtain a dual free update:
\begin{align*}
v_{k+1}=v_k+\eta(\Sigma_k v_k- v_k^{\top}\Sigma_kv_k v_k) =v_k+\eta(I-v_kv_k^\top)\Sigma_kv_k.
\end{align*}
Adding the additional momentum term  $\mu(v_k-v_{k-1}),$ we get  update \eqref{alg0}. 

\subsection{Proof of Lemma \ref{lem_bound} }\label{proof_lem_bound}
\begin{proof}
 First, if we assume $\{v_k\}$ is uniformly bounded by 2, by formulation \eqref{alg0}, we then have
	\begin{align*}
   & v_{k+1}-v_k=\mu (v_k-v_{k-1})+\eta\{\Sigma_k v_{k}-v_{k}^{\top}\Sigma_k v_{k}v_{k}\},\\
	\Longrightarrow &v_{k+1}-v_k=\sum_{i=0}^k\mu^{k-i}\eta\{\Sigma_i v_{i}-v_{i}^{\top}\Sigma_i v_{i}v_{i}\},\\
	\Longrightarrow &\|v_{k+1}-v_k\|_2\leq  C_{\delta}\frac{\eta}{1-\mu},
	\end{align*}
	where $ C_\delta=\sup_{\|v\|\leq2,\|X\|\leq C_d}\|XX^T v-v^{T}XX^T vv\|\leq 2 C_d$.
Next, we show the boundedness assumption on $v$ can be taken off. In fact, with an initialization on $\mathbb{S}$ (the sphere of the unit ball),  the algorithm is  bounded in a much smaller ball of radius $1+O(\eta).$

Recall $\delta_{k+1}=v_{k+1}-v_k$.  Let's consider the difference between the norm  of two iterates,
\begin{align*}
&\Delta_{k}=\|v_{k+1}\|^2-\|v_k\|^2=\|\delta_{k+1}\|^2+2v_k^{\top} \delta_{k+1}\\
&\Delta_{k+1}-\Delta_k=\|\delta_{k+2}\|^2+2v_{k+1}^{\top} \delta_{k+2}-\|\delta_{k+1}\|^2-2v_k^{\top} \delta_{k+1}\\
&=\|\delta_{k+2}\|^2-\|\delta_{k+1}\|^2+2\mu v_{k+1}^{\top} \delta_{k+1}+2\eta v_{k+1}^{\top}\Sigma_{k+1}v_{k+1}(1-v_{k+1}^{\top}v_{k+1})-2v_k^{\top} \delta_{k+1}\\
&=\|\delta_{k+2}\|^2-\|\delta_{k+1}\|^2+2\mu v_{k}^{\top} \delta_{k+1}+2\mu\|\delta_{k+1}\|^2+2\eta v_{k+1}^{\top}\Sigma_{k+1}v_{k+1}(1-v_{k+1}^{\top}v_{k+1})-2v_k^{\top} \delta_{k+1}\\
&=|\delta_{k+2}\|^2+\mu\|\delta_{k+1}\|^2-(1-\mu)(\|\delta_{k+1}\|^2+2v_k^{\top} \delta_{k+1})+2\eta v_{k+1}^{\top}\Sigma_{k+1}v_{k+1}(1-v_{k+1}^{\top}v_{k+1})\\
&=\|\delta_{k+2}\|^2+\mu\|\delta_{k+1}\|^2-(1-\mu)\Delta_{k}+2\eta v_{k+1}^{\top}\Sigma_{k+1}v_{k+1}(1-v_{k+1}^{\top}v_{k+1})\\
&\leq \|\delta_{k+2}\|^2+\mu\|\delta_{k+1}\|^2-(1-\mu)\Delta_{k}.
\end{align*}
The last inequality holds when  $\|v_{k+1}\|\geq1.$ Let $\kappa=\inf\{i:\|v_{i+1}\|>1\},$ then
\begin{align*}
&\Delta_{\kappa+1}\leq(1+\mu)\left(\frac{C_{\delta}}{1-\mu}\right)^2\eta^2+\mu\Delta_\kappa.\\
\end{align*}
Moreover, if $1<\|v_{\kappa+i}\|\leq2$ holds for  $i=1,...,n<\frac{t}{\eta},$ we have
\begin{align*}
\Delta_{\kappa+i} &\leq(1+\mu)\left(\frac{C_{\delta}}{1-\mu}\right)^2\eta^2+\mu\Delta_{\kappa+i-1} \\
&\leq \frac{1+\mu}{1-\mu}\left(\frac{C_{\delta}}{1-\mu}\right)^2\eta^2+\mu^{i}\Delta_{\kappa}.
\end{align*}
Thus,  
\begin{align*}
\|v_{\kappa+n+1}\|^2&= \|v_\kappa\|^2+\sum_{i=0}^{n}\Delta_{\kappa+i}\\
&\leq 1+\frac{1}{1-\mu}\Delta_k+\frac{t}{\eta}\frac{1+\mu}{1-\mu}\left(\frac{C_{\delta}}{1-\mu}\right)^2\eta^2\\
&\leq 1+O\left(\frac{\eta}{(1-\mu)^3}\right).
\end{align*}
In other words, when $\eta$ is very small, we cannot go far from $\mathbb{S}$ and the assumption that $\|v\|\leq 2$ can be removed. 
\end{proof}

\subsection{Proof of Corollary \ref{ODE_solution}}\label{proof_ode_solution}

To apply Theorem \ref{Thm1} to prove the ODE approximation for algorithm \eqref{alg0}, we only need to check whether Assumptions \ref{assumption_general} and \ref{ass_isolated} hold. From our landscape analysis in Section \ref{section_pca}, we know that Assumption  \ref{ass_isolated} holds naturally for streaming PCA. We only need to verify the uniform boundedness and Lipschitz continuity. 

The next lemma shows that the algorithm trajectory of \eqref{alg0} is bounded and thus the boundedness and Lipschitz continuity in Assumption \ref{assumption_general} holds for \eqref{alg0}.
\begin{lemma}\label{lem_bound}
Under Assumption \eqref{Ass1}, given  $v_0\in\mathbb{S}$, for any $k\leq O(1/\eta)$, we have $$\|v_k\|^2 \leq 1+O((1-\mu)^{-3}\eta)~~~~\text{and}~~~~\|v_{k+1}-v_{k}\|\leq \frac{2C_d\eta}{1-\mu}.$$
\end{lemma}
\begin{proof}
First, if we assume $\{v_k\}$ is uniformly bounded by 2, by formulation \eqref{alg0}, we then have
	\begin{align*}
   & v_{k+1}-v_k=\mu (v_k-v_{k-1})+\eta\{\Sigma_k v_{k}-v_{k}^{\top}\Sigma_k v_{k}v_{k}\},\\
	\Longrightarrow &v_{k+1}-v_k=\sum_{i=0}^k\mu^{k-i}\eta\{\Sigma_i v_{i}-v_{i}^{\top}\Sigma_i v_{i}v_{i}\},\\
	\Longrightarrow &\|v_{k+1}-v_k\|_2\leq  C_{\delta}\frac{\eta}{1-\mu},
	\end{align*}
	where $ C_\delta=\sup_{\|v\|_2\leq2,\|X\|_2\leq C_d}\|XX^T v-v^{T}XX^T vv\|_2\leq 2 C_d$.
Next, we show the boundedness assumption on $v$ can be taken off. In fact, with an initialization on $\mathbb{S}$ (the sphere of the unit ball),  the algorithm is  bounded in a much smaller ball of radius $1+O(\eta).$

Recall $\delta_{k+1}=v_{k+1}-v_k$.  Let's consider the difference between the norm  of two iterates,
\begin{align*}
&\Delta_{k}=\|v_{k+1}\|_2^2-\|v_k\|_2^2=\|\delta_{k+1}\|_2^2+2v_k^{\top} \delta_{k+1}\\
&\Delta_{k+1}-\Delta_k=\|\delta_{k+2}\|_2^2+2v_{k+1}^{\top} \delta_{k+2}-\|\delta_{k+1}\|_2^2-2v_k^{\top} \delta_{k+1}\\
&=\|\delta_{k+2}\|_2^2-\|\delta_{k+1}\|_2^2+2\mu v_{k+1}^{\top} \delta_{k+1}+2\eta v_{k+1}^{\top}\Sigma_{k+1}v_{k+1}(1-v_{k+1}^{\top}v_{k+1})-2v_k^{\top} \delta_{k+1}\\
&=\|\delta_{k+2}\|_2^2-\|\delta_{k+1}\|_2^2+2\mu v_{k}^{\top} \delta_{k+1}+2\mu\|\delta_{k+1}\|_2^2+2\eta v_{k+1}^{\top}\Sigma_{k+1}v_{k+1}(1-v_{k+1}^{\top}v_{k+1})-2v_k^{\top} \delta_{k+1}\\
&=|\delta_{k+2}\|_2^2+\mu\|\delta_{k+1}\|_2^2-(1-\mu)(\|\delta_{k+1}\|_2^2+2v_k^{\top} \delta_{k+1})+2\eta v_{k+1}^{\top}\Sigma_{k+1}v_{k+1}(1-v_{k+1}^{\top}v_{k+1})\\
&=\|\delta_{k+2}\|_2^2+\mu\|\delta_{k+1}\|_2^2-(1-\mu)\Delta_{k}+2\eta v_{k+1}^{\top}\Sigma_{k+1}v_{k+1}(1-v_{k+1}^{\top}v_{k+1})\\
&\leq \|\delta_{k+2}\|_2^2+\mu\|\delta_{k+1}\|_2^2-(1-\mu)\Delta_{k}.
\end{align*}
The last inequality holds when  $\|v_{k+1}\|_2\geq1.$ Let $\kappa=\inf\{i:\|v_{i+1}\|_2>1\},$ then
\begin{align*}
&\Delta_{\kappa+1}\leq(1+\mu)\left(\frac{C_{\delta}}{1-\mu}\right)^2\eta^2+\mu\Delta_\kappa.
\end{align*}
Moreover, if $1<\|v_{\kappa+i}\|_2\leq2$ holds for  $i=1,...,n<\frac{t}{\eta},$ we have
\begin{align*}
\Delta_{\kappa+i} &\leq(1+\mu)\left(\frac{C_{\delta}}{1-\mu}\right)^2\eta^2+\mu\Delta_{\kappa+i-1} \\
&\leq \frac{1+\mu}{1-\mu}\left(\frac{C_{\delta}}{1-\mu}\right)^2\eta^2+\mu^{i}\Delta_{\kappa}.
\end{align*}
Thus,  
\begin{align*}
\|v_{\kappa+n+1}\|_2^2&= \|v_\kappa\|_2^2+\sum_{i=0}^{n}\Delta_{\kappa+i}\\
&\leq 1+\frac{1}{1-\mu}\Delta_k+\frac{t}{\eta}\frac{1+\mu}{1-\mu}\left(\frac{C_{\delta}}{1-\mu}\right)^2\eta^2\\
&\leq 1+O\left(\frac{\eta}{(1-\mu)^3}\right).
\end{align*}
In other words, when $\eta$ is very small, we cannot go far from $\mathbb{S}$ and the assumption that $\|v\|_2\leq 2$ can be removed	
\end{proof}

Therefore all the assumptions for Theorem \ref{Thm1} holds and we know that  $V^{\eta}(\cdot)\Rightarrow V(\cdot)$ in the weak sense as $\eta\rightarrow0$ in the space $D^d[0,\infty)$, where $V(\cdot)$ is the unique solution to the following ODE:
\begin{equation*}
\dot{V} = \frac{1}{1-\mu}(\Sigma V-V^{\top}\Sigma VV), ~~V(0)=v_0.
\end{equation*}

To solve ODE \eqref{ODE_PCA}, we rotate the coordinate to decouple each dimension.  Under Assumption \ref{Ass0},  there exists an orthogonal matrix Q such that:
$\Sigma=Q\Lambda Q^{\top},$
where $\Lambda={\rm diag}(\lambda_1, \lambda_2,...,\lambda_d).$ Let $H(t)=Q^\top V(t),$ or equivalently $V(t)= QH(t).$ Substitute $V(t)$ with $QH(t)$ in ODE \eqref{ODE0}, then  we can  obtain the following ODE.
\begin{align}\label{ODE}
\dot{H}=\frac{1}{1-\mu}[\Lambda H-H^{\top}\Lambda HH].
\end{align}
ODE \eqref{ODE} is different from (4.6) in \cite{chen2017online} by a constant $\frac{1}{1-\mu},$ and has an explicit form solution.
Then we have  the initial value problem  \eqref{ODE_PCA} has a solution $V(t)= QH(t),$ where 
\begin{equation}\label{H_sol}
H^{(i)}(t)=\Big(\sum_{i=1}^d [H^{(i)}(0)\exp\Big(\frac{\lambda_it}{1-\mu}\Big)]^2\Big)^{-\frac{1}{2}}H^{(i)}(0)\exp\left(\frac{\lambda_it}{1-\mu}\right),\quad  i=1,...,d.
\end{equation}
where $H(0)=Q^\top v_0, v_0\in\SSS.$ Moreover, suppose $v_0\neq \pm v^{i},~ \forall i=2,...,d,$  as $t\rightarrow\infty,$ one can easily verify that $V(t)$ converges to $v^1,$ which is the global maximum to \eqref{eqn:PCA}.

Last, we show the uniqueness of the above solution. Define $f(t,v)=\frac{1}{1-\mu}[\Sigma v -v^\top\Sigma v v]$ and a domain $\cR=\big\{(t,v)\big| t\geq 0, \norm{v}_2\leq 1\big\}.$ Since $f(t,v)$ is continuously differentiable with respect to $(t,v),$ $f(t,v)$ satisfies Lipschitz continuous condition in $\cR$ with respect to $v$ and uniformly in $t.$ By Theorem  1.2.1 in \cite{Hu2004}, we know the solution is unique. 

\subsection{Proof of Corollary \ref{corollary_rate}}\label{proof_rate}
\begin{proof}[Proof of Corollary \ref{corollary_rate}.]
Phase I and III are a directly application of Theorems \ref{lemma_phase3} \ref{Time_Saddle}. Here we only consider Phase II.

	After Phase I, we restart our record time, i.e., $H^{\eta,1}(0)\geq\delta$  and we obtain
	\begin{align*}
		\PP(\left\|V^{\eta}(T_2) - v^1\right\|_2^2\leq \delta^2) & \rightarrow \PP(\left\|V(T_2) - v^1\right\|_2^2\leq \delta^2) =  \PP(\left\|H(T_2) - e^1\right\|_2^2\leq \delta^2),
	\end{align*}
where $H$ is defined in \eqref{H_sol}. Since $H$ is deterministic and
		\begin{align}\label{ineq1}
 \left(H^{(1)}(T_2)\right)^2
	&=\left( \sum\limits_{j=1}^{d}\left(\left(H^{(j)}(0)\right)^2\exp{\left(2\frac{\lambda_j}{1-\mu} T_2\right)}\right)\right)^{-1}\left(H^{(1)}(0)\right)^2\exp{\left(2\frac{\lambda_1}{1-\mu} T_2\right)}\notag \\
	& \geq  \left( \delta^2 \exp\left(2\frac{\lambda_1}{1-\mu} T_2\right)+(1-\delta^2)\exp\left(2\frac{\lambda_2}{1-\mu} T_2\right) \right)^{-1}\delta^2 \exp\left(2\frac{\lambda_2}{1-\mu} T_2\right),  
	\end{align}
	 Thus, when the term \eqref{ineq1} satisfies\begin{align}\label{ineq2}
	\left( \delta^2 \exp\left(2\frac{\lambda_1}{1-\mu} T_2\right)+(1-\delta^2)\exp\left(2\frac{\lambda_2}{1-\mu} T_2\right) \right)^{-1}\delta^2 \exp\left(2\frac{\lambda_1}{1-\mu}T_2\right) \geq 1-\delta^2/2,
	 \end{align} 
	 we have $$\PP(\left(H^{(1)}(T_2)\right)^2 \geq 1-\delta^2/2)=1.$$ Then for sufficiently small $\eta$, we have $$	\PP(\left(H^{\eta,1}(T_2)\right)^2\geq 1-\delta^2/2)\geq \frac{3}{4}. $$ 
 Note that when $\left(H^{(1)}(T_2)\right)^2 \geq 1-\delta^2/2,$ we have
 $$\|H^\eta(T_2)-e^1\|_2^2\leq 2-2\sqrt{1-\delta^2/2}\leq \delta^2.$$
Therefore, 
 $$	\PP(\|V^\eta(T_2)-v^1\|_2^2\leq \delta^2 ) = \PP(\|H^\eta(T_2)-v^1\|_2^2\leq \delta^2 ) \geq \frac{3}{4}. $$

	 Solving the above inequality \eqref{ineq2}, we get
	\begin{align*}
		T_2=\frac{1-\mu}{2(\lambda_1-\lambda_2)}\log\frac{2-\delta^2}{\delta^2}~.
	\end{align*} 
	We finish the proof.
\end{proof}

\section{Deep Neural Networks Experiments}\label{NNsetting}

	\begin{figure}[htb!]
		\centering
	
		\label{cifar10}
	\subfigure[$\eta_{\textrm{V}}=\eta_{\textrm{M}}/(1-\mu)=1.6$]{
		\includegraphics[width=0.31\textwidth]{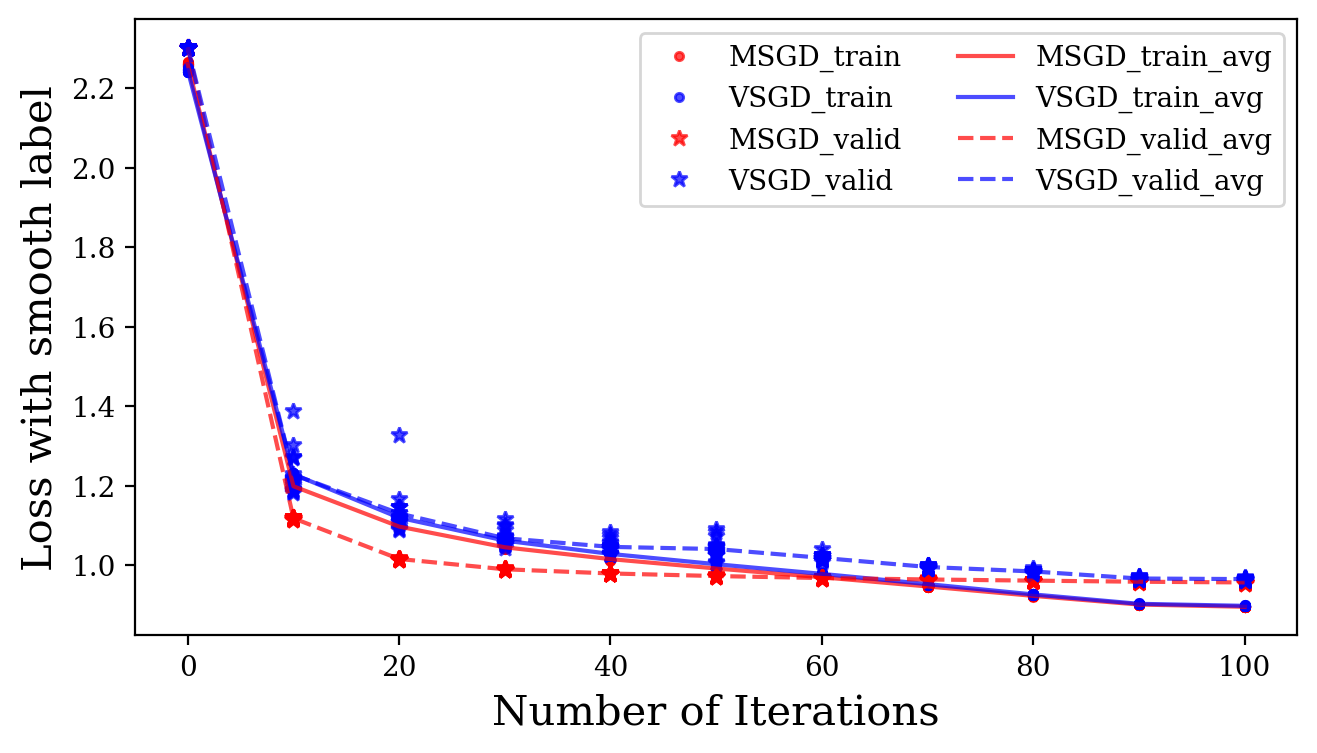}\label{fig:10_0.4}
	}
	\subfigure[$\eta_{\textrm{V}}=\eta_{\textrm{M}}/(1-\mu)=2$]{
		\includegraphics[width=0.31\textwidth]{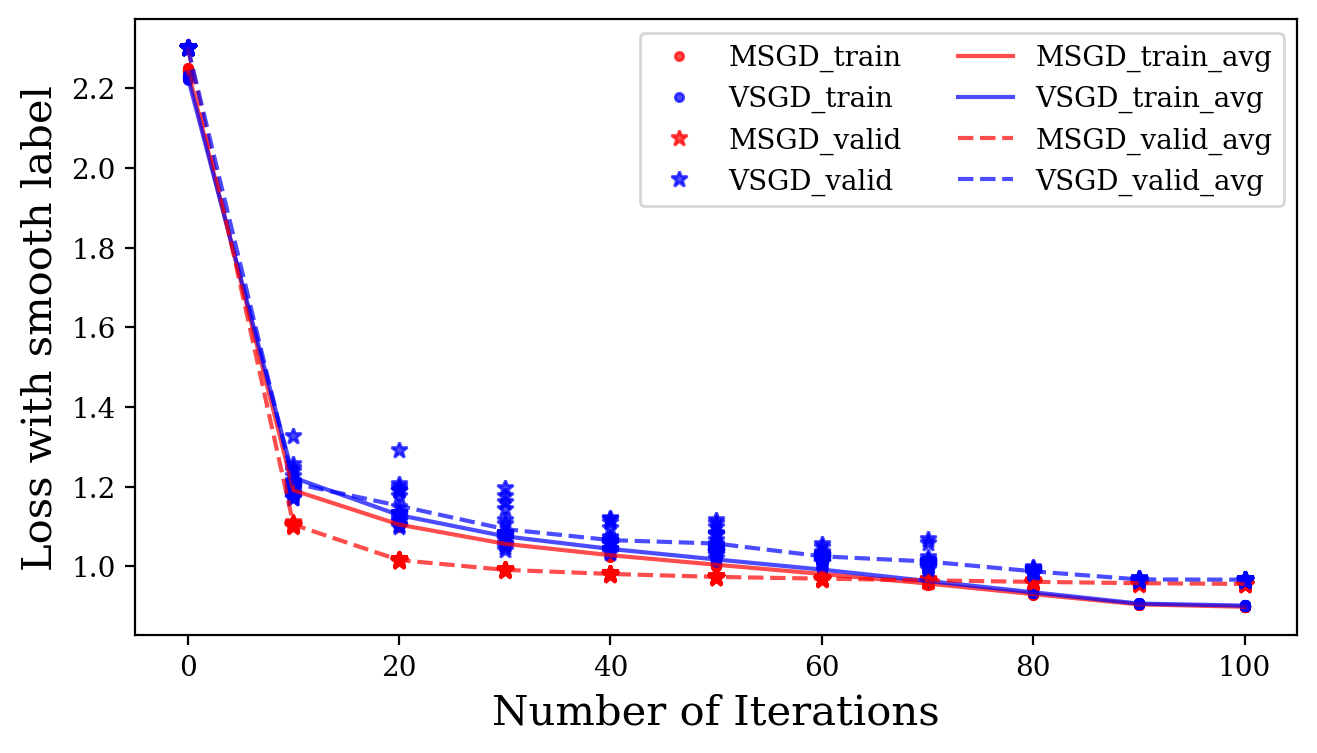}\label{fig:10_0.8}
	}
	\subfigure[$\eta_{\textrm{V}}=\eta_{\textrm{M}}/(1-\mu)=2.4$]{
		\includegraphics[width=0.31\textwidth]{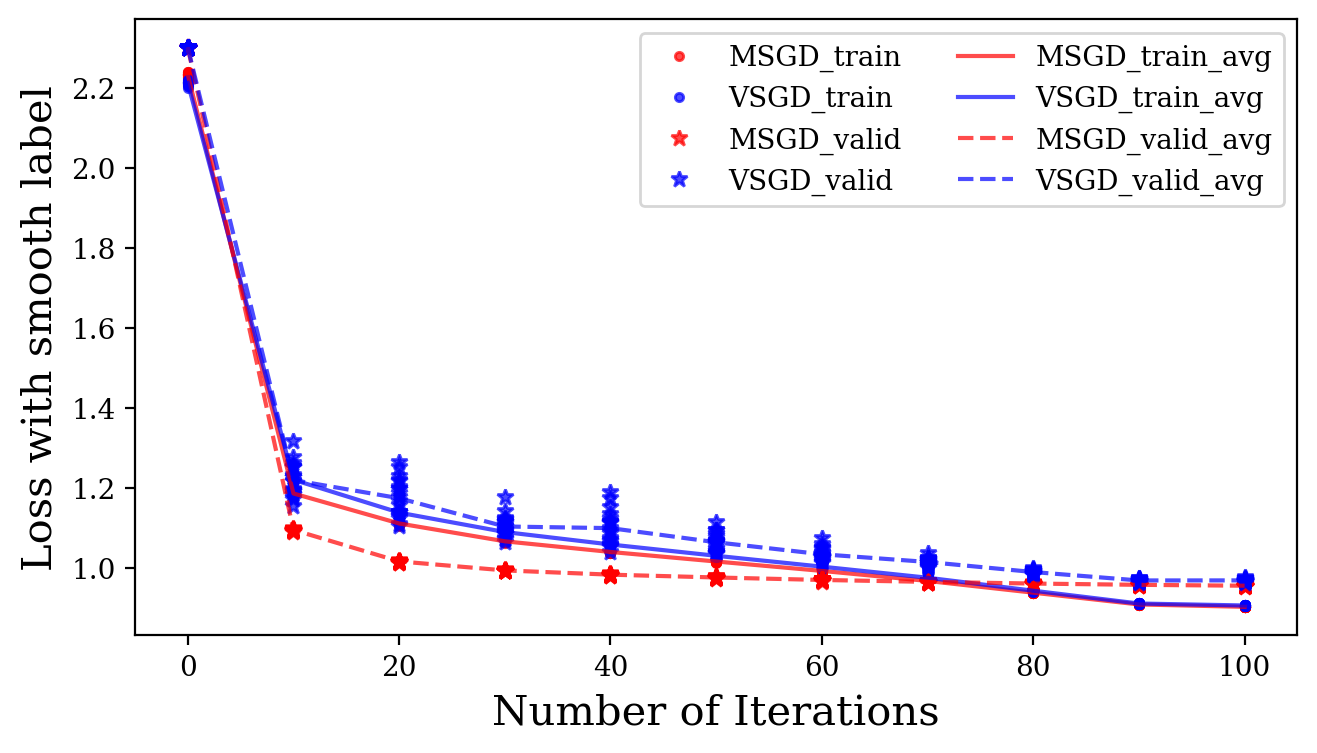}\label{fig:10_1.2}
	}\\
		\subfigure[$\eta_{\textrm{V}}=\eta_{\textrm{M}}/(1-\mu)=2.8$]{
		\includegraphics[width=0.31\textwidth]{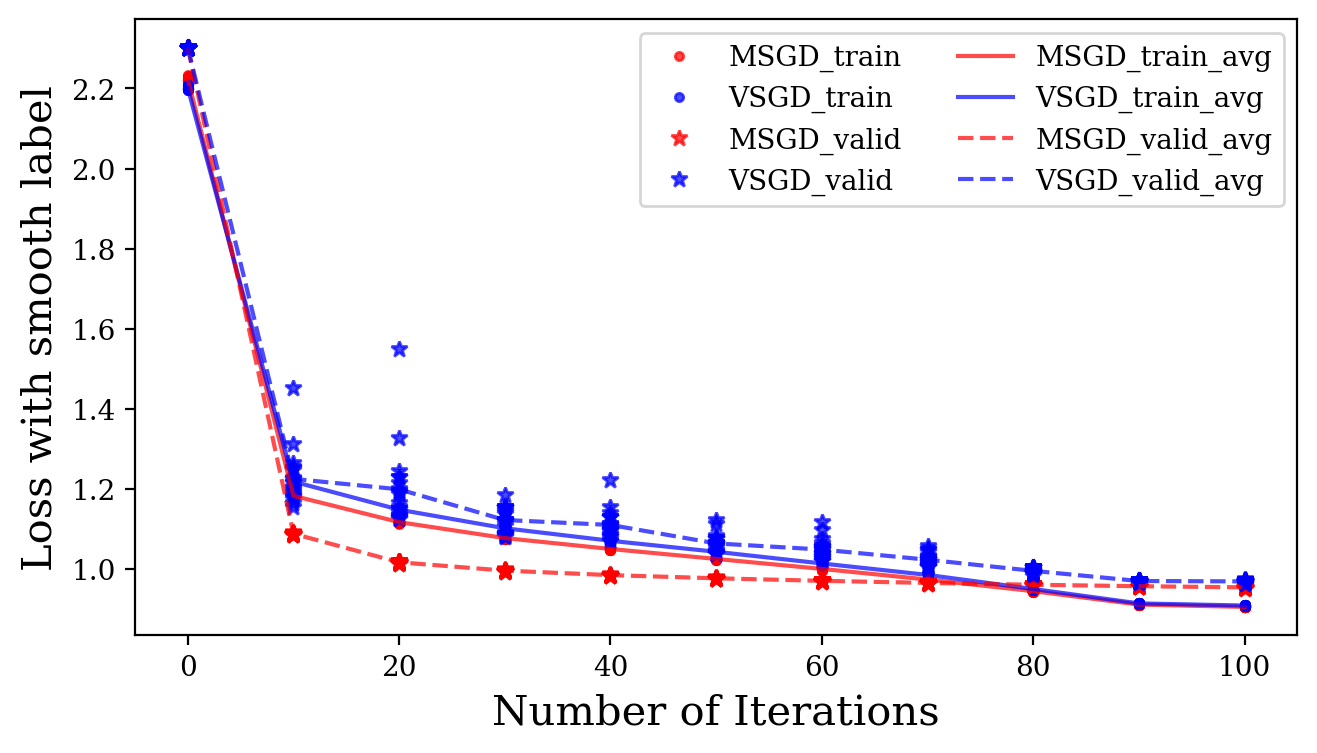}\label{fig:10_1.6}
	}
	\subfigure[$\eta_{\textrm{V}}=\eta_{\textrm{M}}/(1-\mu)=3.2$]{
		\includegraphics[width=0.31\textwidth]{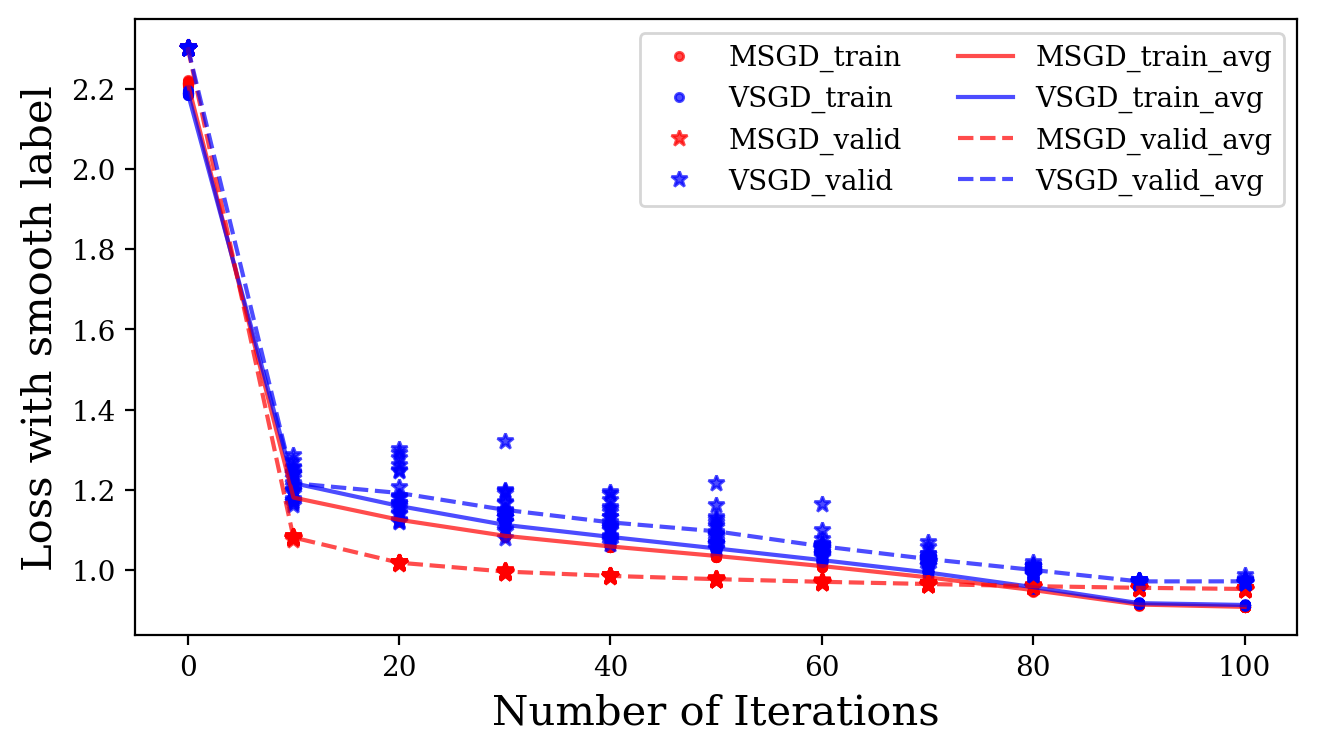}\label{fig:10_2}
	}
	\subfigure[$\eta_{\textrm{V}}=\eta_{\textrm{M}}/(1-\mu)=3.6$]{
		\includegraphics[width=0.31\textwidth]{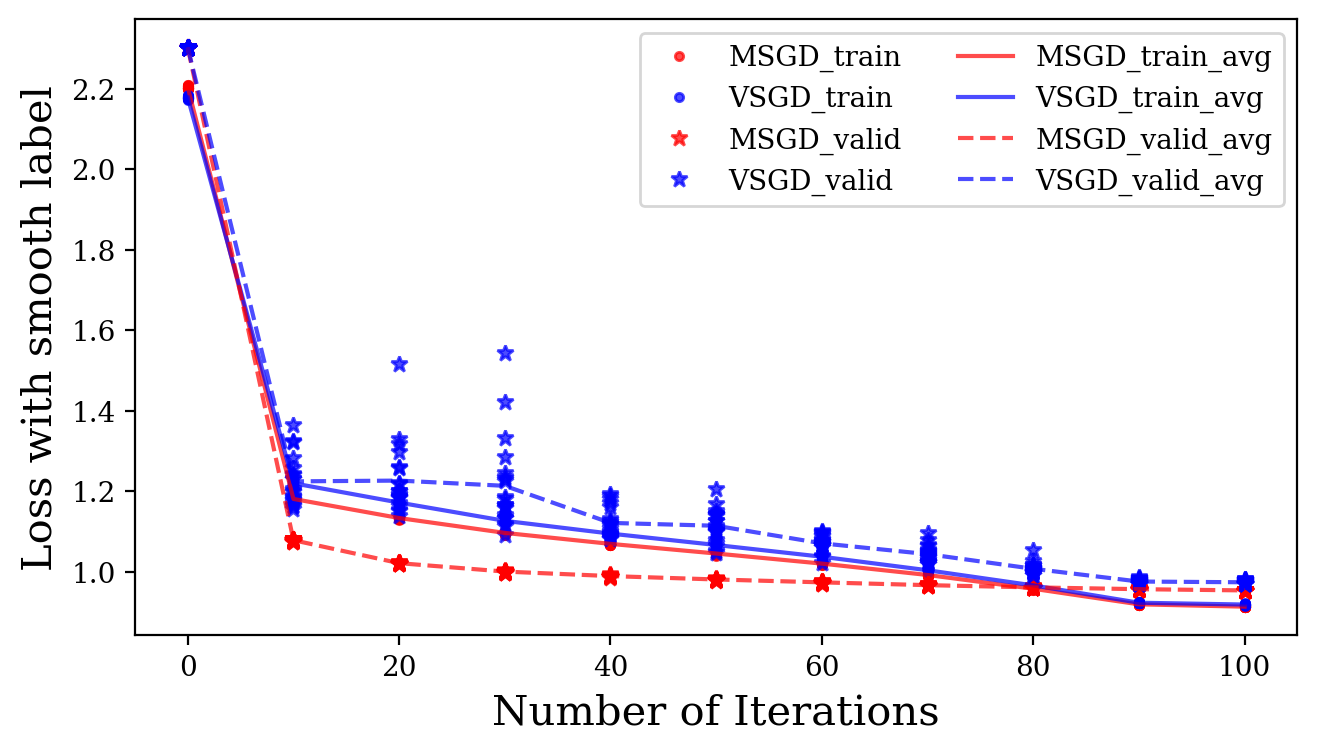}\label{fig:10_2.4}
	}\\
		\subfigure[$\eta_{\textrm{V}}=\eta_{\textrm{M}}/(1-\mu)=4$]{
		\includegraphics[width=0.31\textwidth]{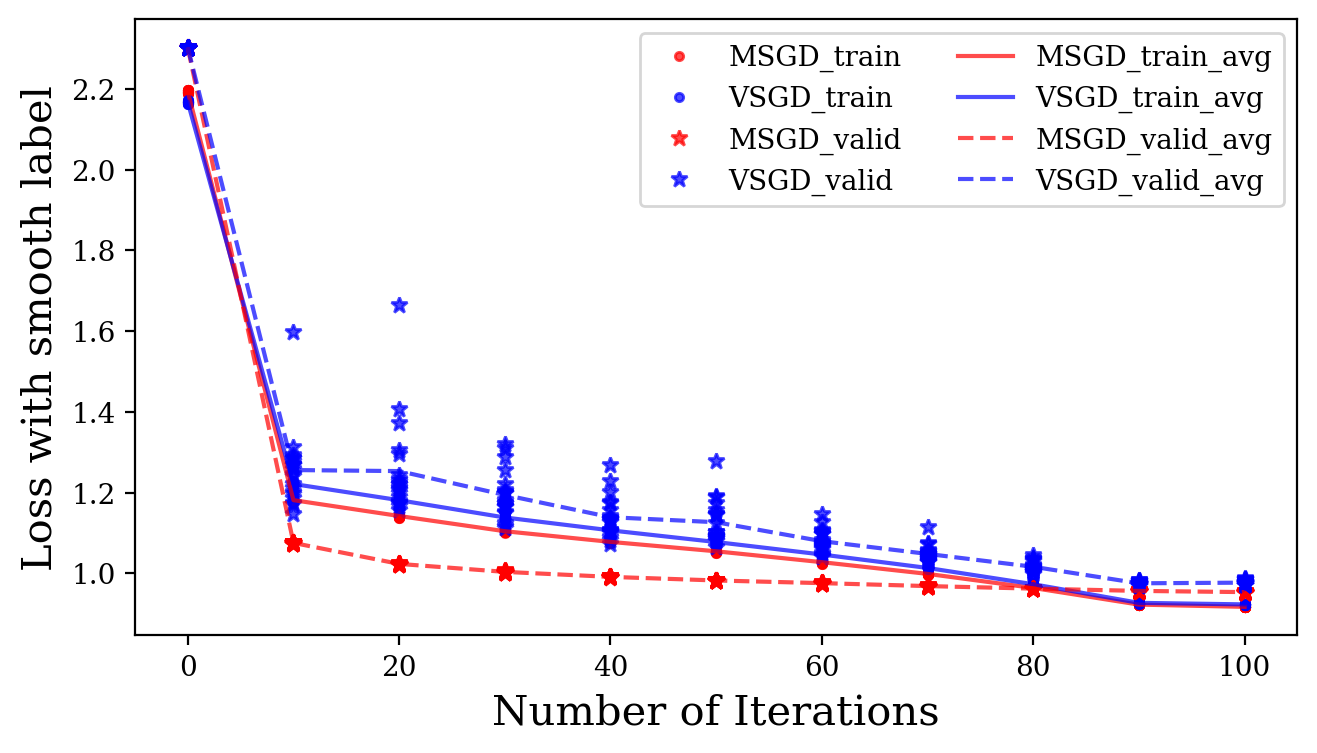}\label{fig:10_2.8}
	}
		\subfigure[$\eta_{\textrm{V}}=\eta_{\textrm{M}}/(1-\mu)=4.4$]{
		\includegraphics[width=0.31\textwidth]{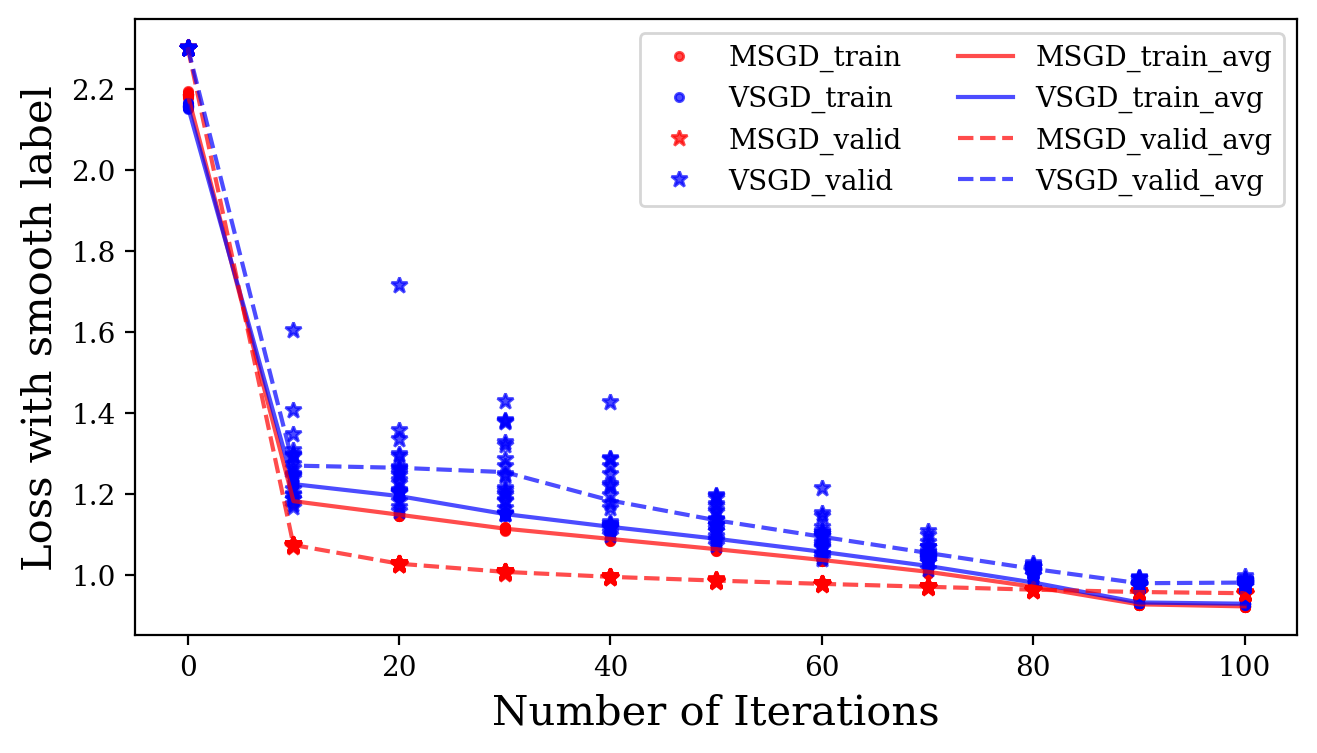}\label{fig:10_3.2}
	}
		\subfigure[$\eta_{\textrm{V}}=\eta_{\textrm{M}}/(1-\mu)=4.8$]{
		\includegraphics[width=0.31\textwidth]{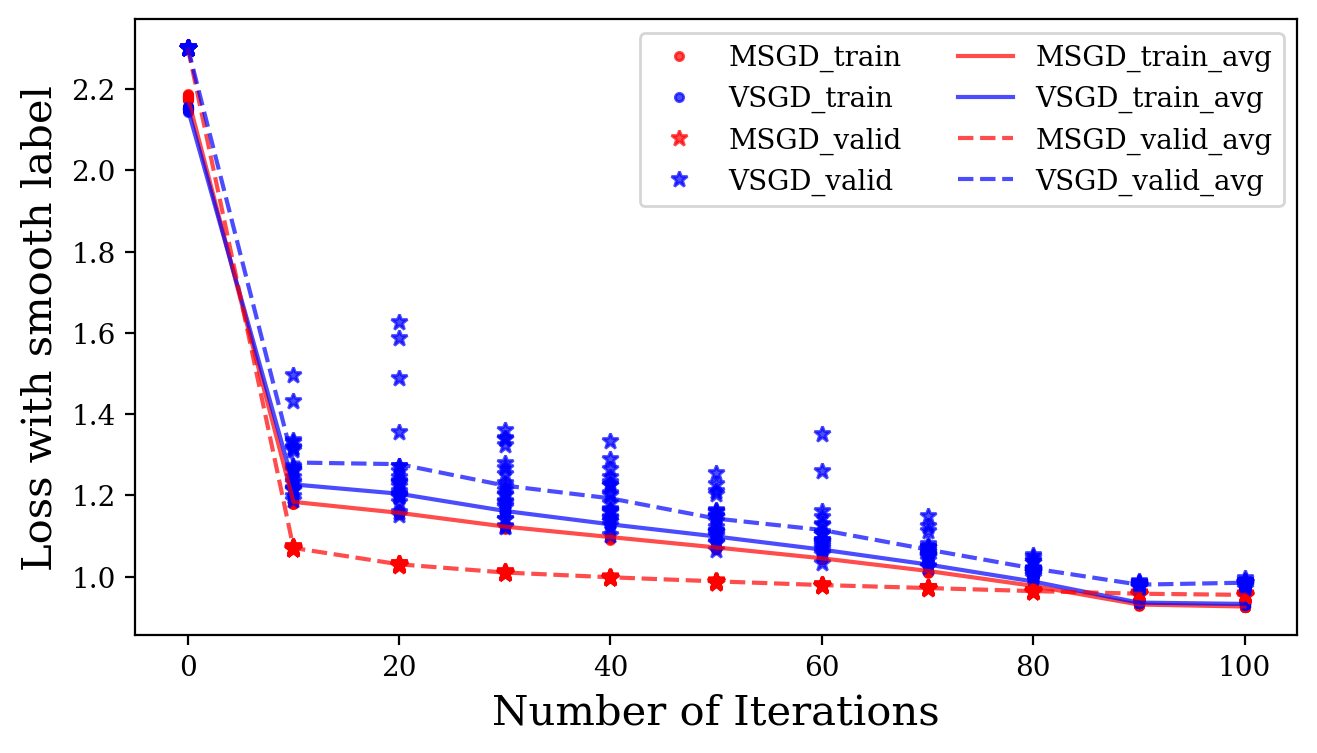}\label{fig:10_3.6}
	}\\
			\subfigure[$\eta_{\textrm{V}}=\eta_{\textrm{M}}/(1-\mu)=5.2$]	{
	\includegraphics[width=0.31\textwidth]{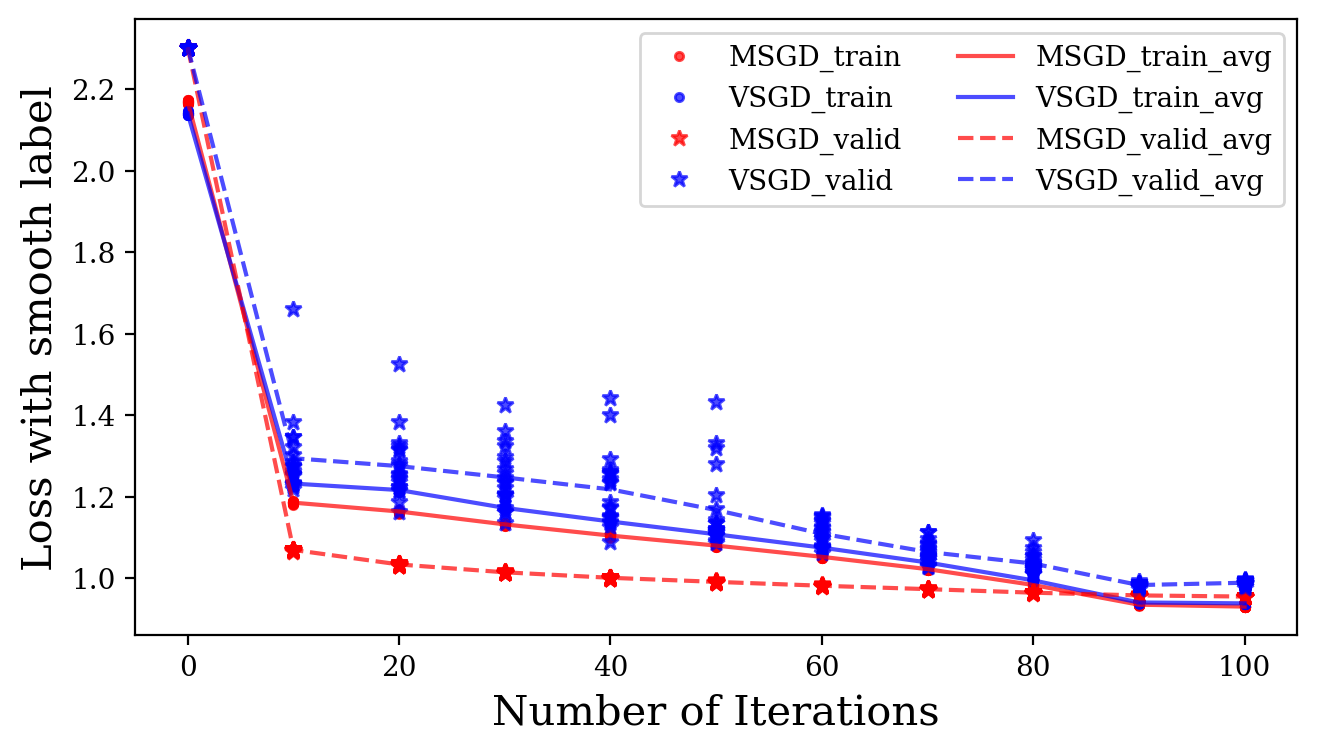}\label{fig:10_3.6}
}
			\subfigure[$\eta_{\textrm{V}}=\eta_{\textrm{M}}/(1-\mu)=5.6$]{
		\includegraphics[width=0.31\textwidth]{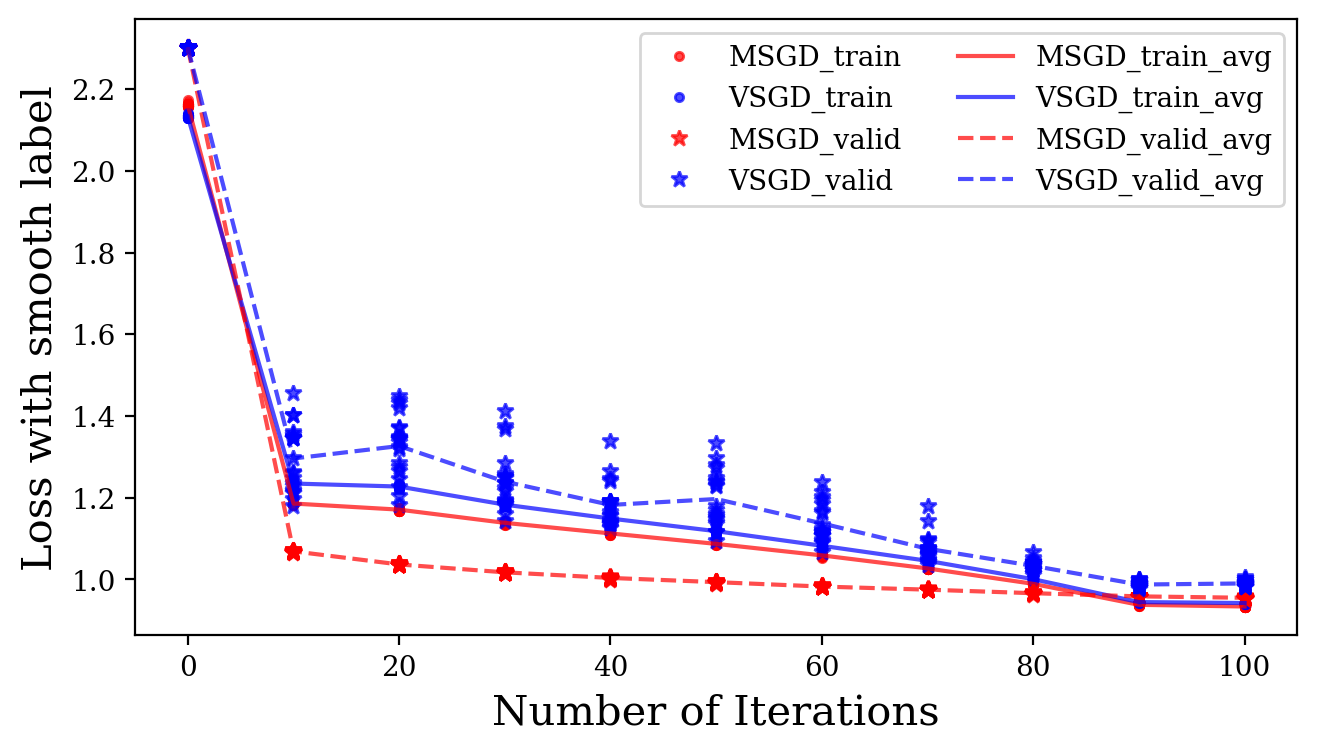}
	}
	\subfigure[$\eta_{\textrm{V}}=\eta_{\textrm{M}}/(1-\mu)=6$]{
		\includegraphics[width=0.31\textwidth]{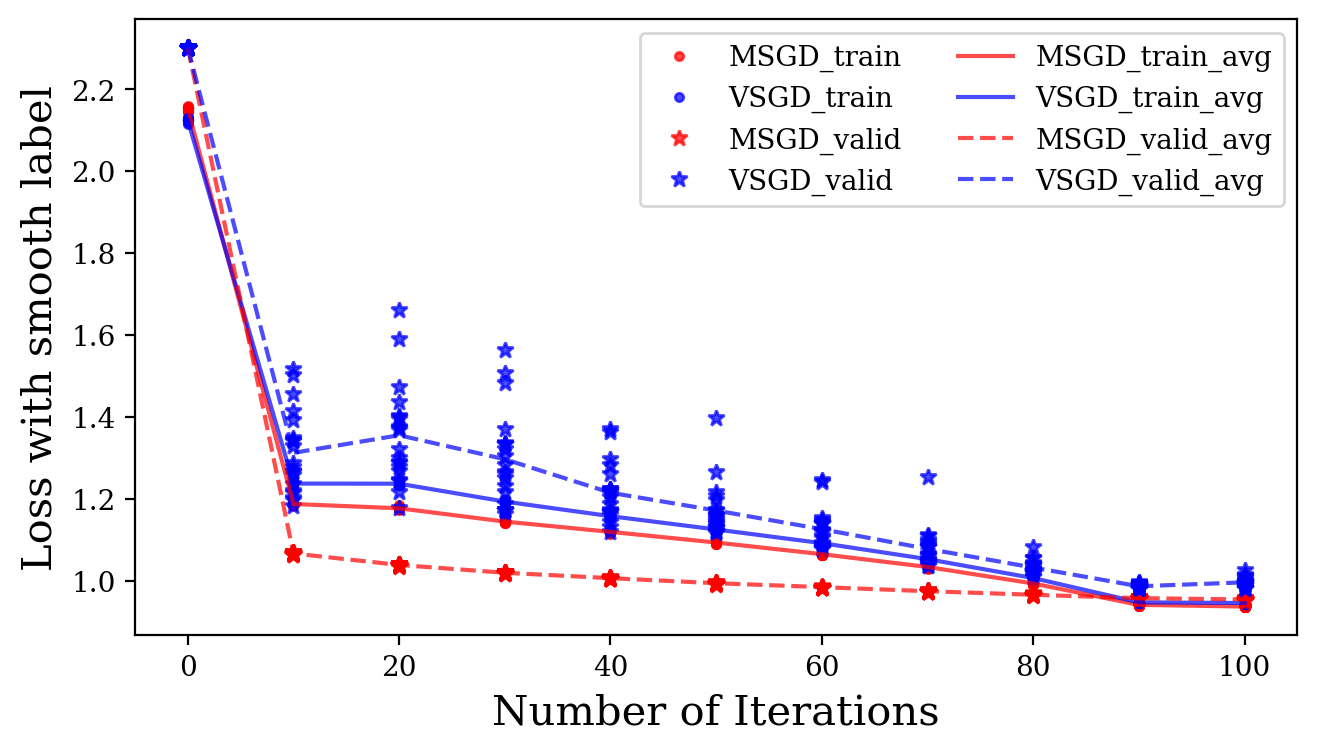}
	}
		\caption{Experimental Results of  ResNet-9 on CIFAR-10. VSGD uses the Equivalent Step Sizes of MSGD.}

				\end{figure}

\begin{figure}[htb!]
		\centering
		
		\label{cifar100}

	\subfigure[$\eta_{\textrm{V}}=\eta_{\textrm{M}}/(1-\mu)=1.6$]{
		\includegraphics[width=0.31\textwidth]{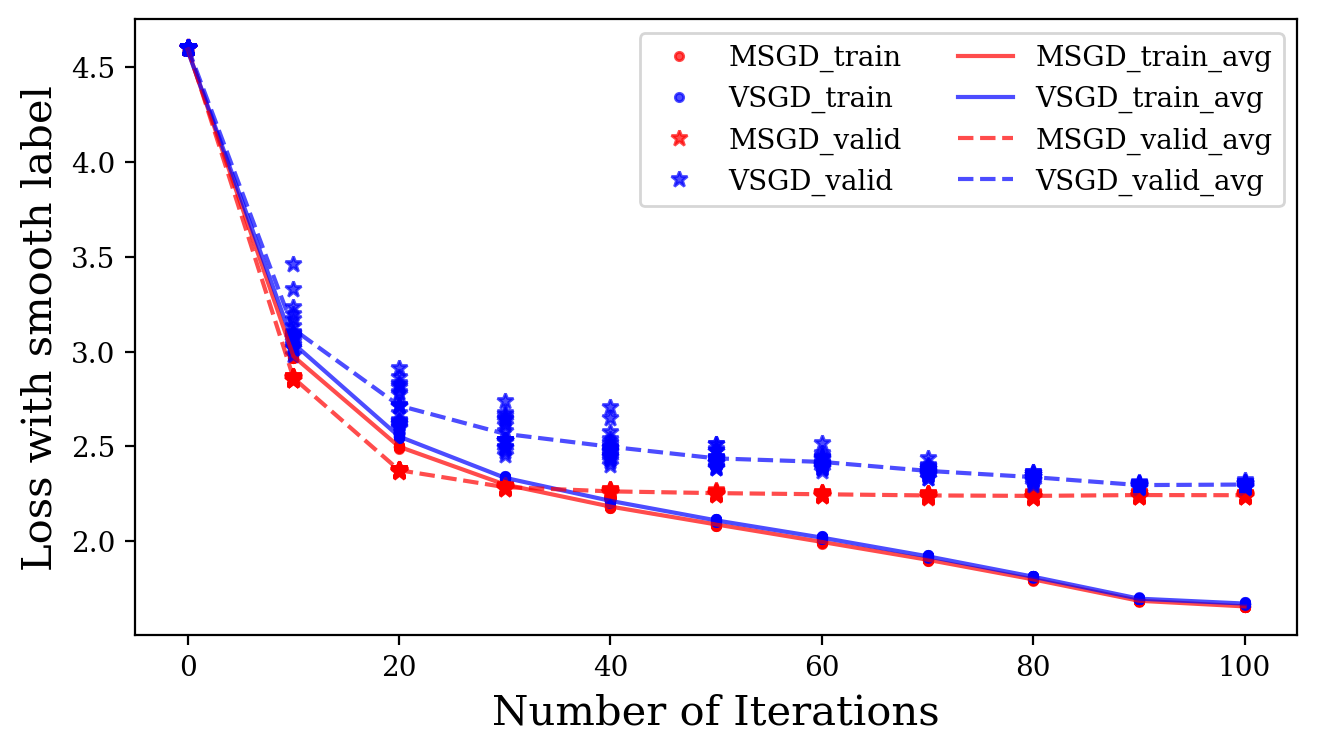}\label{fig:100_0.4}
	}
	\subfigure[$\eta_{\textrm{V}}=\eta_{\textrm{M}}/(1-\mu)=2$]{
		\includegraphics[width=0.31\textwidth]{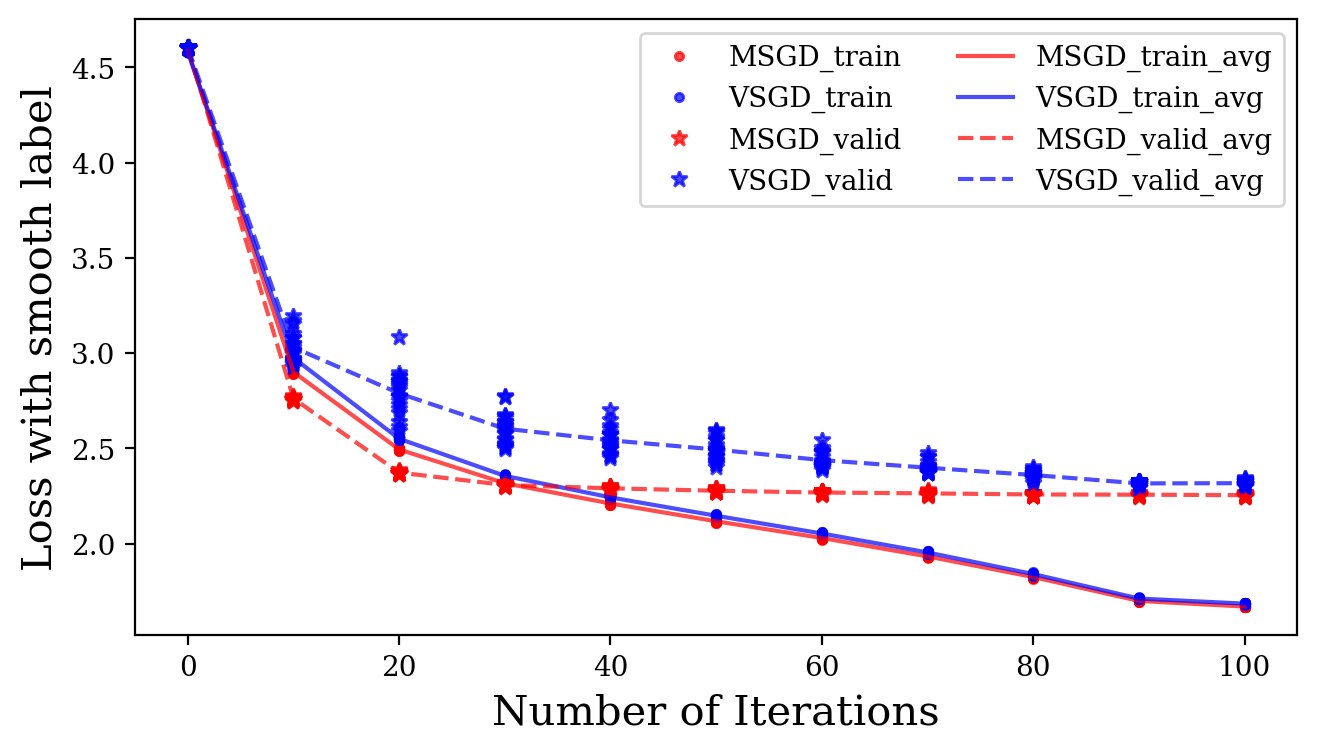}\label{fig:100_0.8}
	}
	\subfigure[$\eta_{\textrm{V}}=\eta_{\textrm{M}}/(1-\mu)=2.4$]{
		\includegraphics[width=0.31\textwidth]{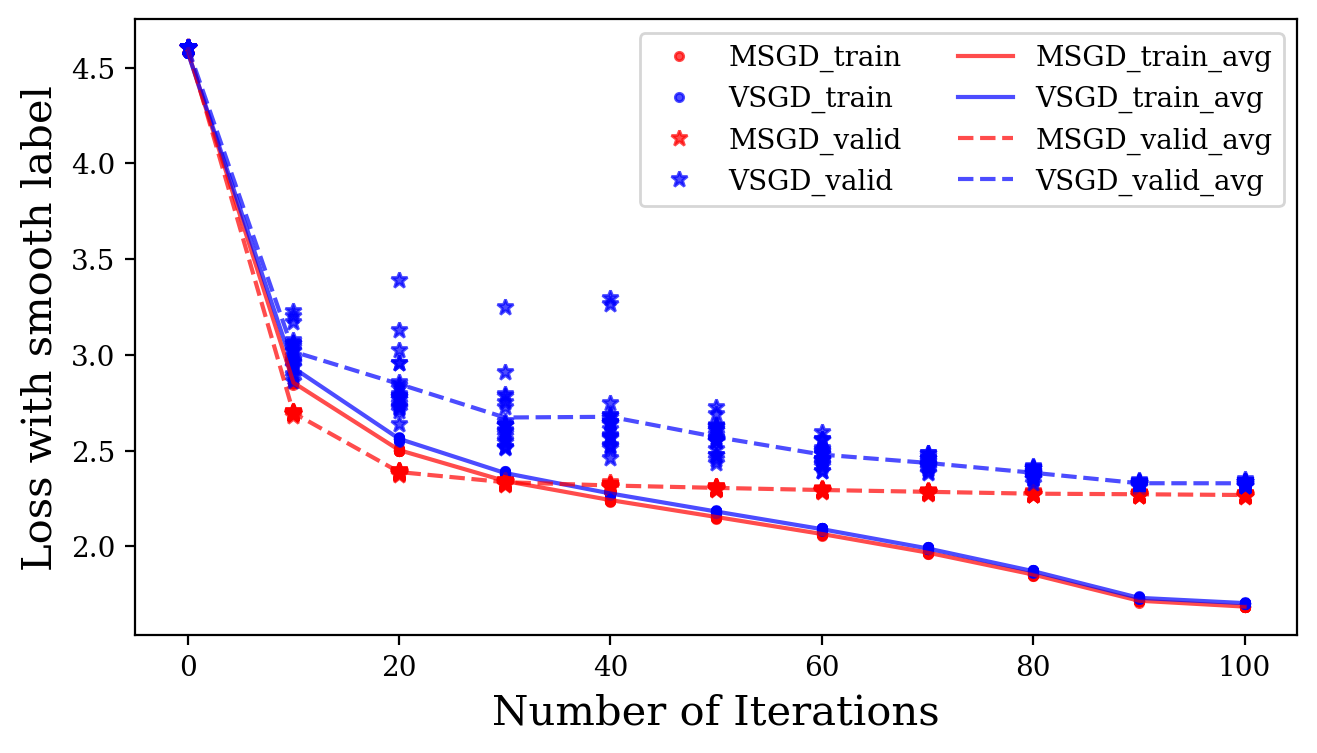}\label{fig:100_1.2}
	}\\
		\subfigure[$\eta_{\textrm{V}}=\eta_{\textrm{M}}/(1-\mu)=2.8$]{
		\includegraphics[width=0.31\textwidth]{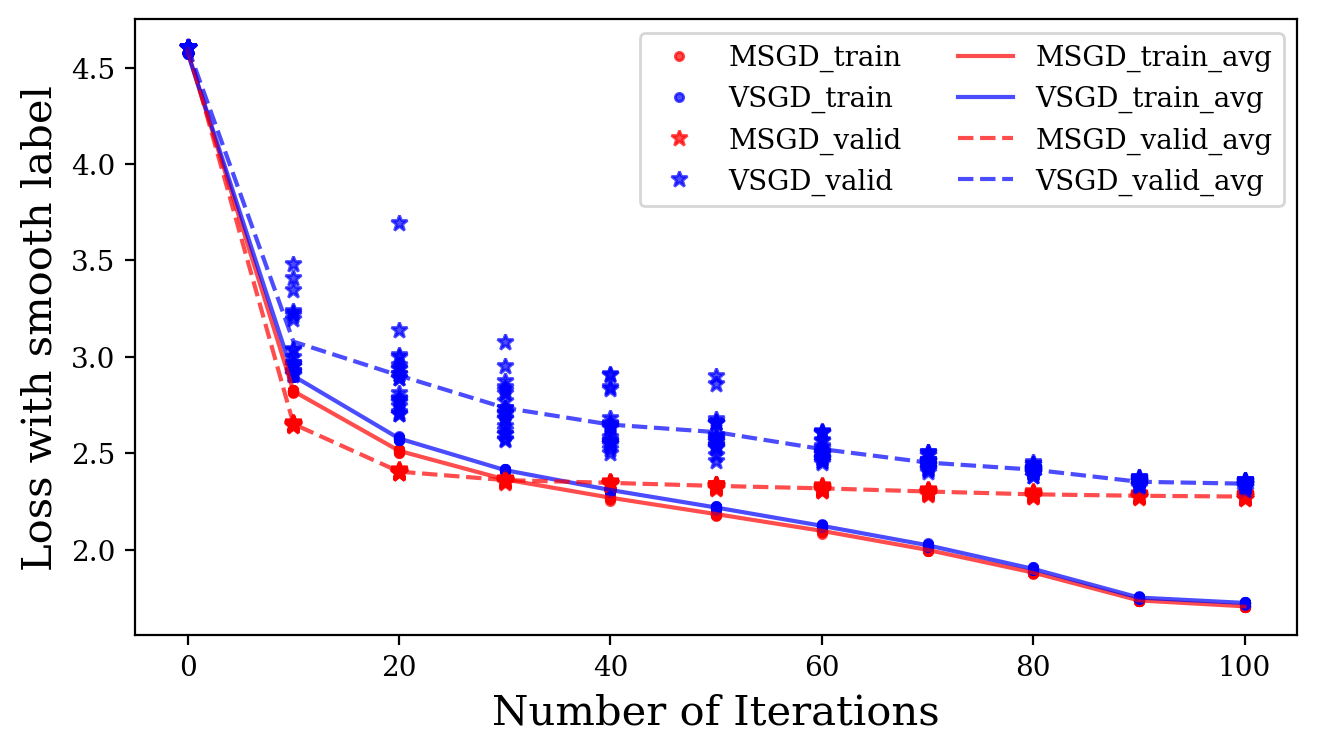}\label{fig:100_1.6}
	}
	\subfigure[$\eta_{\textrm{V}}=\eta_{\textrm{M}}/(1-\mu)=3.2$]{
		\includegraphics[width=0.31\textwidth]{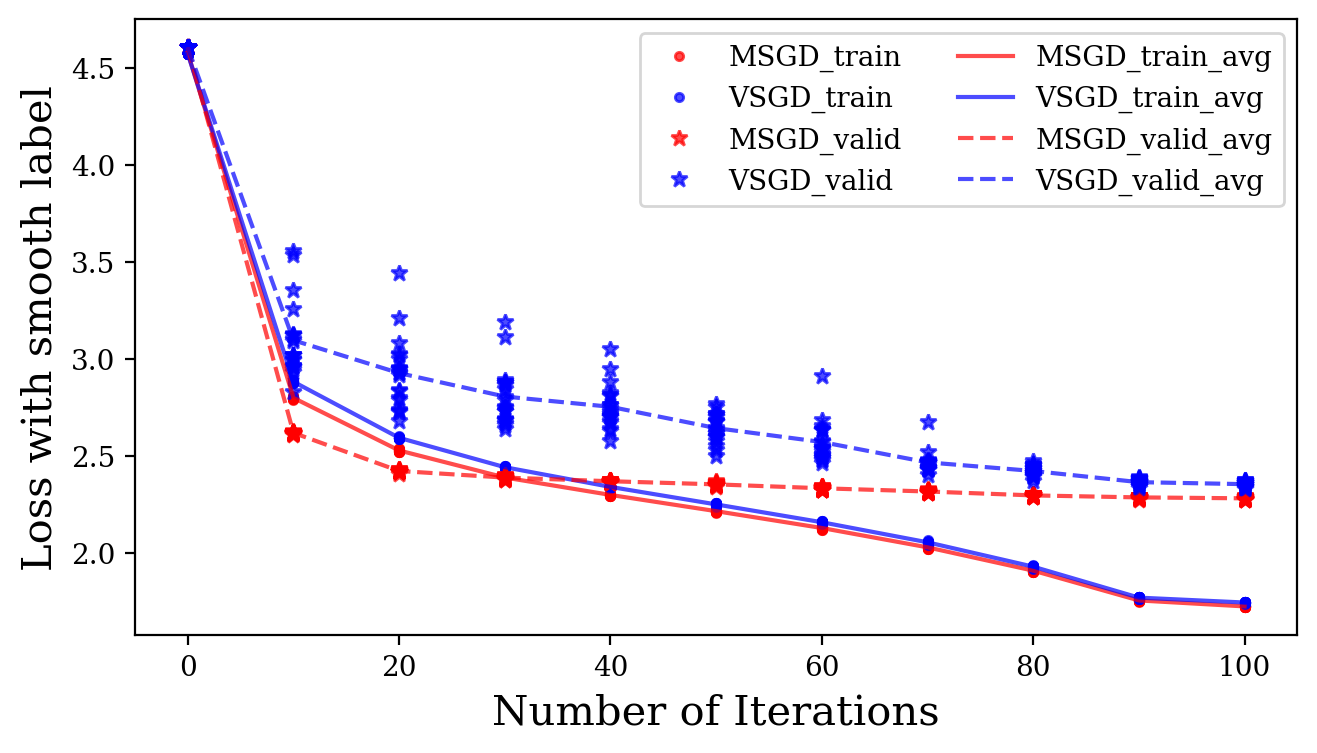}\label{fig:100_2}
	}
	\subfigure[$\eta_{\textrm{V}}=\eta_{\textrm{M}}/(1-\mu)=3.6$]{
		\includegraphics[width=0.31\textwidth]{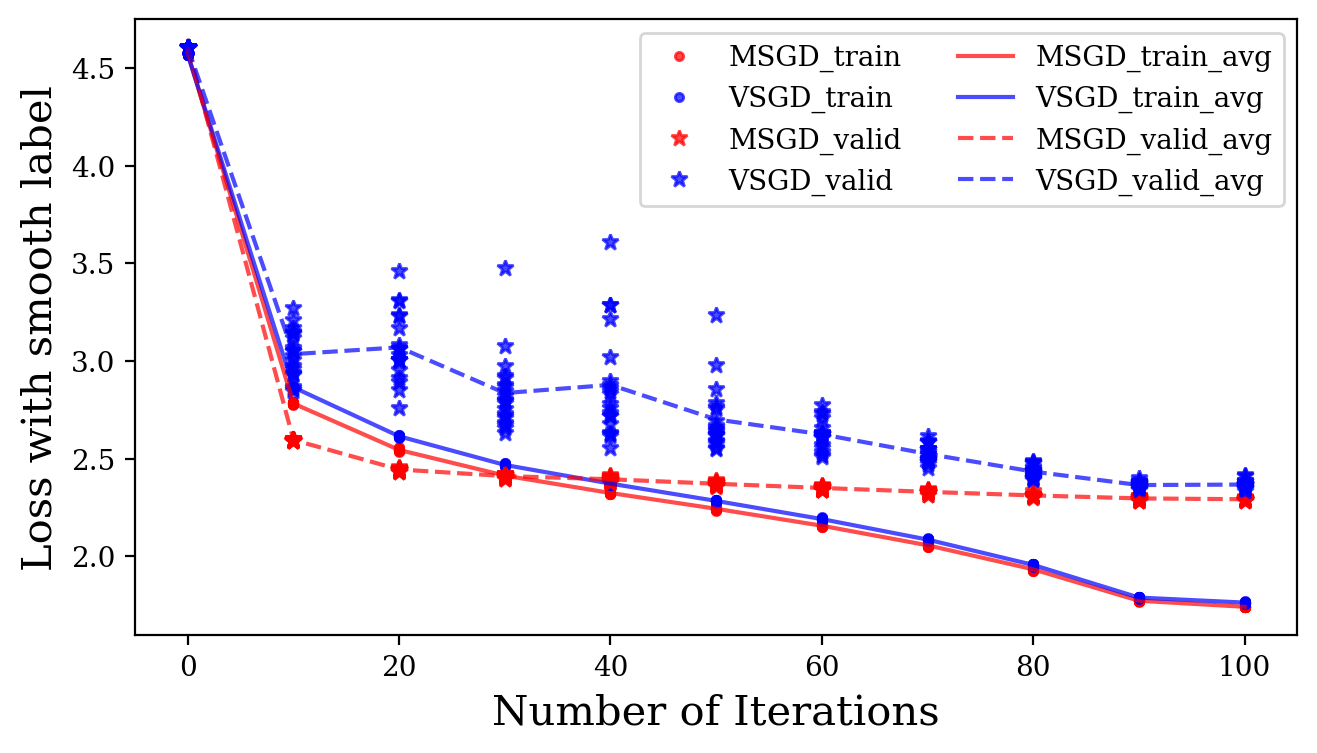}\label{fig:100_2.4}
	}\\
		\subfigure[$\eta_{\textrm{V}}=\eta_{\textrm{M}}/(1-\mu)=4$]{
		\includegraphics[width=0.31\textwidth]{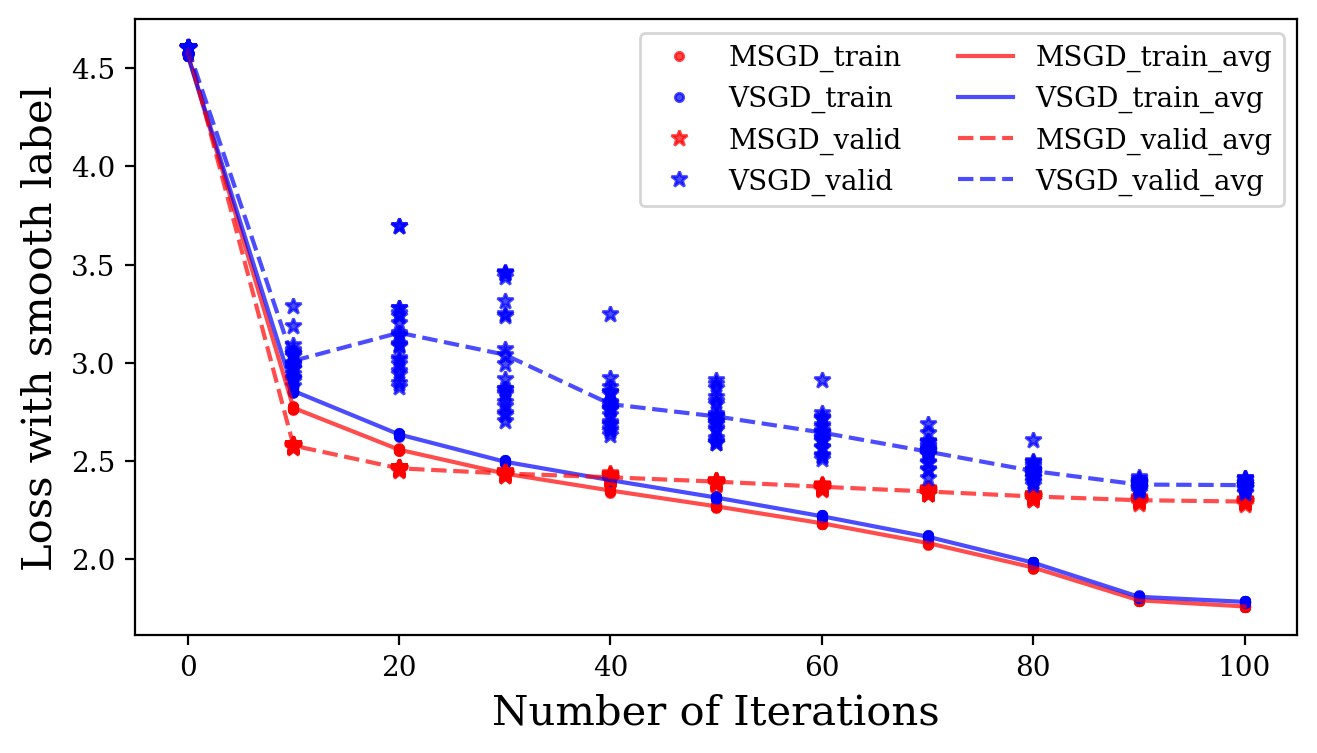}\label{fig:100_2.8}
	}
		\subfigure[$\eta_{\textrm{V}}=\eta_{\textrm{M}}/(1-\mu)=4.4$]{
		\includegraphics[width=0.31\textwidth]{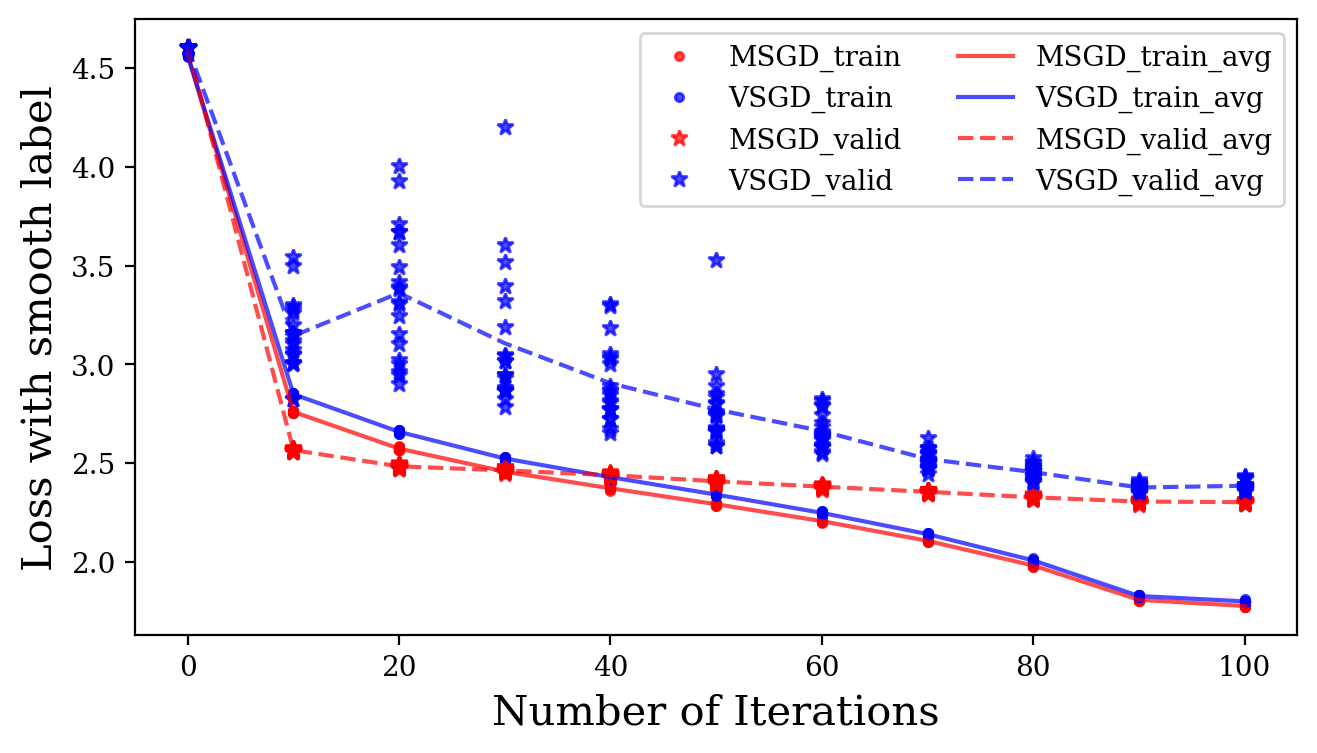}\label{fig:100_3.2}
	}
		\subfigure[$\eta_{\textrm{V}}=\eta_{\textrm{M}}/(1-\mu)=4.8$]{
		\includegraphics[width=0.31\textwidth]{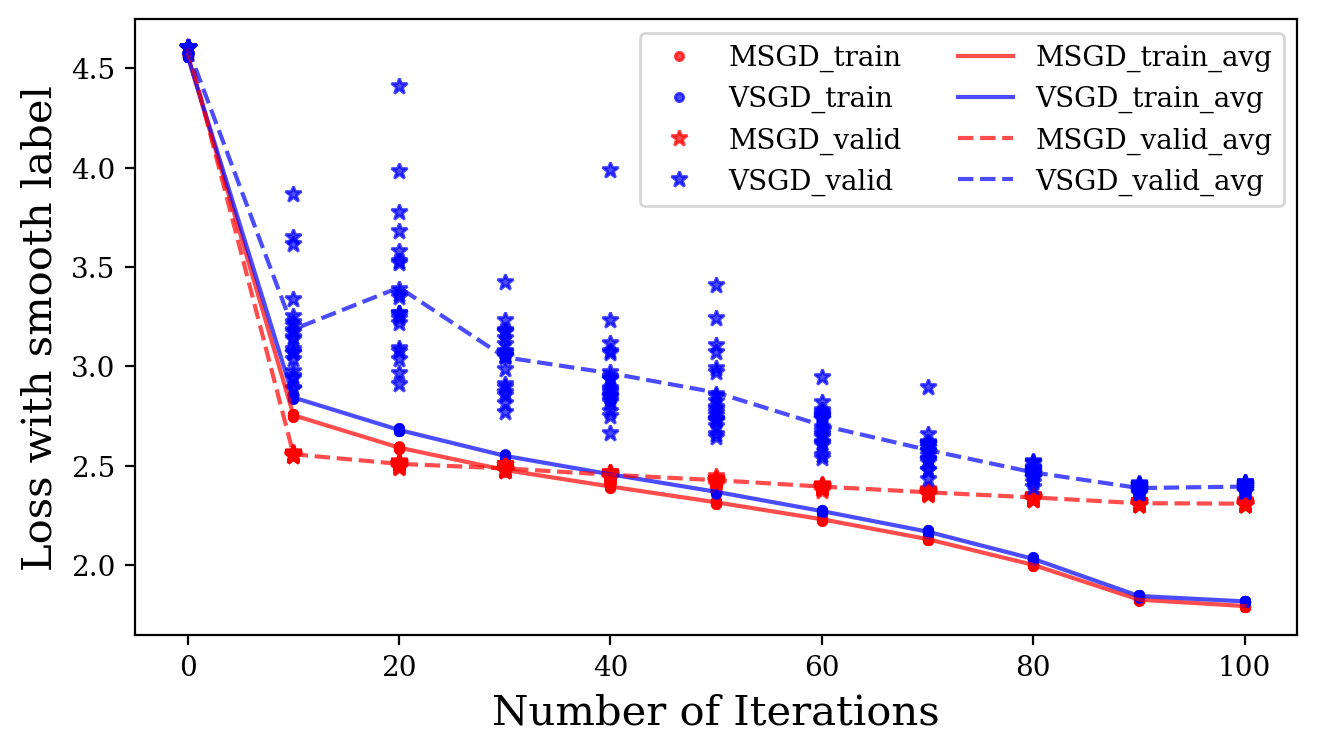}\label{fig:100_3.6}
	}\\
	\subfigure[$\eta_{\textrm{V}}=\eta_{\textrm{M}}/(1-\mu)=5.2$]{
		\includegraphics[width=0.31\textwidth]{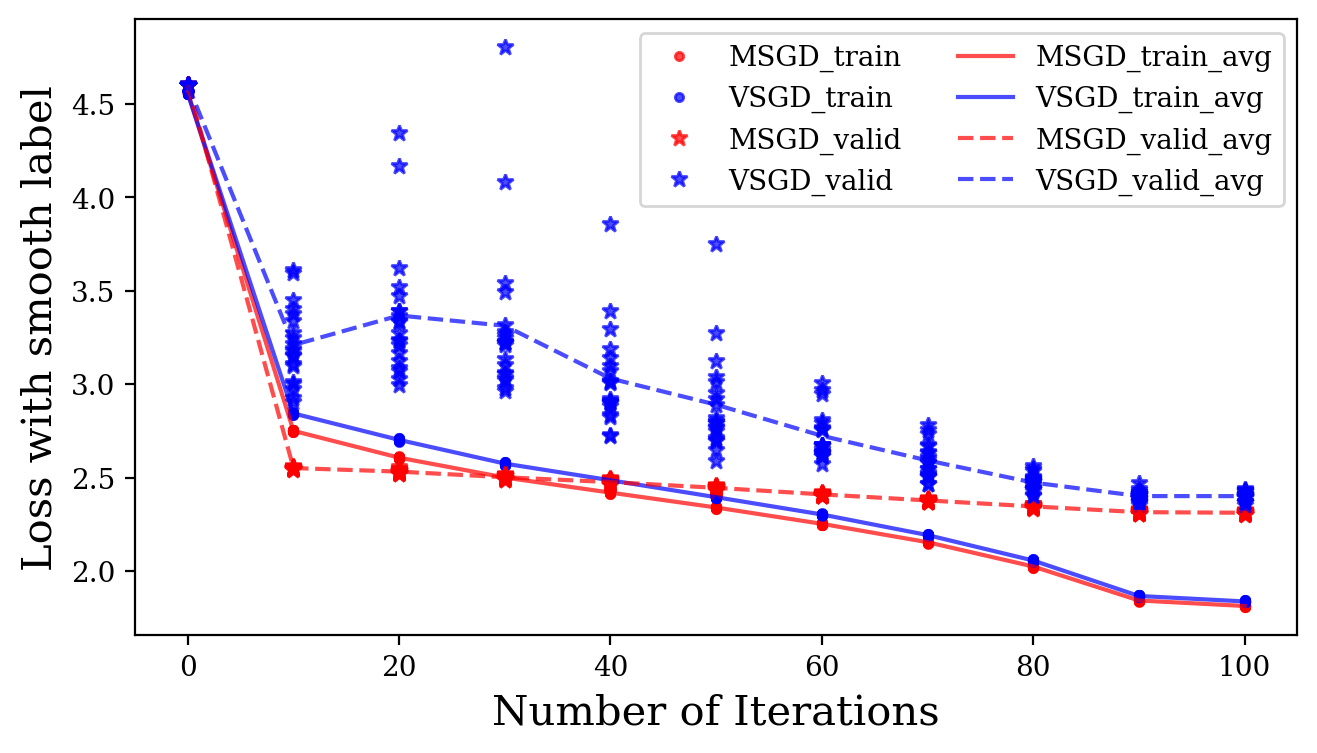}\label{fig:100_3.6}
	}
	\subfigure[$\eta_{\textrm{V}}=\eta_{\textrm{M}}/(1-\mu)=5.6$]{
		\includegraphics[width=0.31\textwidth]{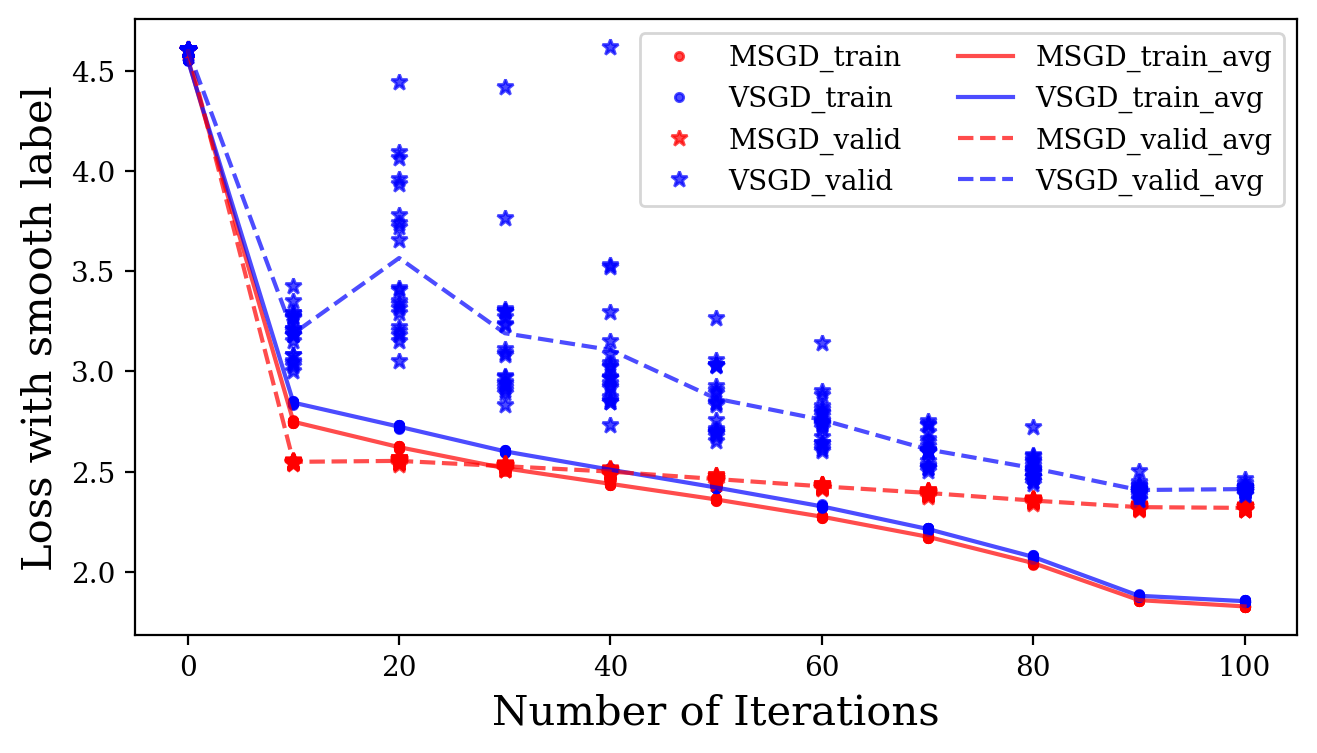}
	}
	\subfigure[$\eta_{\textrm{V}}=\eta_{\textrm{M}}/(1-\mu)=6$]{
		\includegraphics[width=0.31\textwidth]{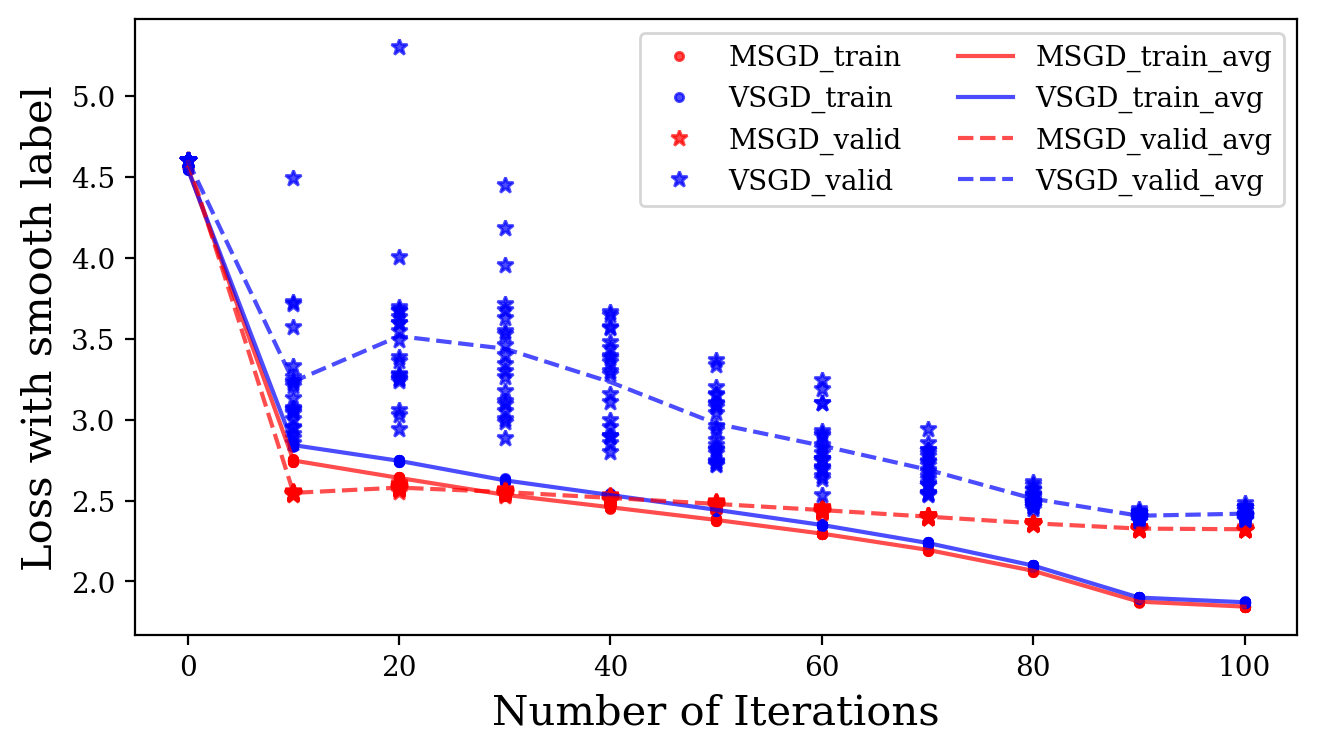}
	}
\caption{Experimental Results of  ResNet-9 on CIFAR-100. VSGD uses the Equivalent Step Sizes of MSGD.}
 
				\end{figure}

\end{document}